\pdfoutput=1
\documentclass[11pt, letterpaper]{article}
\usepackage[margin=1in]{geometry}
\usepackage[utf8]{inputenc}

\usepackage{amsmath,amsfonts,bm}

\def\eqref#1{equation~\ref{#1}}

\def\1{\bm{1}}

\DeclareMathAlphabet{\mathsfit}{\encodingdefault}{\sfdefault}{m}{sl}
\SetMathAlphabet{\mathsfit}{bold}{\encodingdefault}{\sfdefault}{bx}{n}

\newcommand{\E}{\mathbb{E}}

\newcommand{\R}{\mathbb{R}}

\usepackage{url}
\usepackage{amsfonts}
\usepackage{amsmath}
\usepackage{amssymb}
\usepackage{amsthm}
\usepackage{enumitem}
\usepackage{microtype}
\usepackage{xcolor}
\usepackage[numbers,sort]{natbib}
\usepackage[hypertexnames=false]{hyperref}
\usepackage{graphicx}
\usepackage{booktabs}
\usepackage{xspace}
\usepackage{float}
\usepackage{pifont}
\usepackage[capitalise,nameinlink]{cleveref}
\usepackage{authblk}


\crefname{equation}{Eq.}{Eqs.}
\Crefname{equation}{Eq.}{Eqs.}
\creflabelformat{equation}{#2#1#3}

\newtheorem{theorem}{Theorem}
\newtheorem{lemma}{Lemma}
\newtheorem{definition}{Definition}
\newtheorem{corollary}{Corollary}
\newtheorem{proposition}{Proposition}
\crefname{theorem}{Theorem}{Theorems}
\crefname{corollary}{Corollary}{Corollaries}
\crefname{proposition}{Proposition}{Propositions}

\newcommand{\Prob}{\mathbb{P}}
\newcommand{\tr}{\operatorname{tr}}
\newcommand{\op}{\mathrm{op}}

\DeclareMathAlphabet{\mathbfit}{OML}{cmm}{b}{it}
\newcommand{\question}{\mathbfit{q}}
\newcommand{\passk}{\ensuremath{\mathrm{pass}@k}\xspace}
\newcommand{\Verify}{\text{\textsc{Verify}}}

\usepackage[most]{tcolorbox}
\definecolor{amber}{HTML}{B7791F}
\newtcolorbox{objectivebox}{
    colback=amber!8,
    colframe=amber!65,
    boxrule=0.6pt,
    arc=2pt,
    left=6pt, right=6pt, top=4pt, bottom=4pt,
    enhanced,
    width=1\linewidth,
    center,
    coltitle=black,
    colbacktitle=amber!22,
    fonttitle=\bfseries\small,
    title={Pass@$k$ Pruning Objective},
    attach boxed title to top left={xshift=8pt, yshift=-\tcboxedtitleheight/2},
    boxed title style={colframe=amber!65, boxrule=0.5pt, arc=1.5pt},
}

\tcolorboxenvironment{theorem}{
    colback=amber!8,
    colframe=amber!65,
    boxrule=0.6pt,
    arc=2pt,
    left=6pt, right=6pt, top=4pt, bottom=4pt,
    enhanced,
}

\setcitestyle{numbers,square,comma}

\title{Direct Pass@$k$ Optimization for Thought \mbox{Pruning:} Balancing Quality and Diversity in Test-Time Reasoning}

\author{Mahdi Salmani}
\author{Adel Javanmard}
\author{Devansh Gupta}
\author{Meisam Razaviyayn}
\affil{University of Southern California\\
\texttt{\char`\{salmanis, ajavanma, guptadev, razaviya\char`\}@usc.edu}}
\date{}

\begin{document}
\raggedbottom   %

\maketitle

\begin{abstract}
Test-time scaling has emerged as a key inference strategy for reasoning models, where a model generates a large number of candidate solutions in the hope of finding at least one correct solution. Running all candidate solutions to completion can be computationally demanding, so recent work prunes the candidate pool after generating only a short prefix, committing further computation only to the surviving candidates. Existing pruning strategies divide broadly into two categories: Quality-based methods assign a score to each partial solution and discard those predicted to ultimately lead to an incorrect final solution, while diversity-based methods select partial solutions that follow different reasoning trajectories and are therefore likely to cover different final solutions. The tension between these two objectives is an instance of exploration versus exploitation, yet little to no prior work on pruning exposes this trade-off directly, let alone optimizes it. In this paper, we introduce a pruning procedure that addresses this trade-off directly by optimizing \passk over the selected solution pool. In particular, it chooses the pool to maximize the probability that at least one of its partial solutions, when continued to completion, yields a correct answer. Because this objective simultaneously rewards high-quality partial solutions and penalizes redundancy among them, quality and diversity are naturally balanced by the criterion itself rather than through separate ad-hoc design choices. On eight open-weight reasoning models, our method outperforms existing quality-based and diversity-based pruning strategies. The gains are particularly pronounced when the pruning procedure keeps a moderate number of partial solutions for continued generation, highlighting the value of directly balancing quality and diversity during pruning.
\end{abstract}

\section{Introduction}
\begin{figure}[t]
\centering
\includegraphics[width=0.38\textwidth]{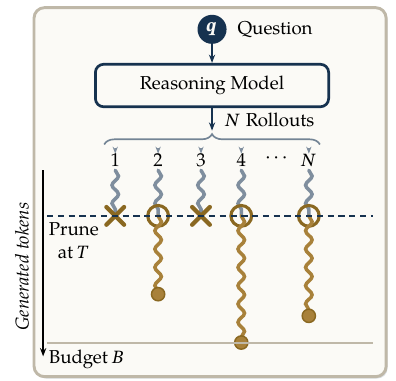}
\caption{\textbf{Test-time pruning.} Only $k$ of the $N$ rollouts (ringed) continue past the cutoff $T$, up to the budget $B$.}
\label{fig:problem}
\end{figure}
Large Language Models (LLMs) have demonstrated remarkable capabilities in complex reasoning, showcasing a strong aptitude for solving complex problems that require multi-step logic \citep{wei2022chain, kojima2022large, yao2023tree, jaech2024openai, guo2025deepseek}. To further enhance these capabilities, test-time scaling has emerged as a prominent strategy, where generating multiple candidate solutions for a single problem allows the model to explore diverse reasoning strategies and achieve substantially better final performance \citep{snell2024scaling, wang2022self, brown2024large, wu2024inference}. However, a significant challenge in scaling test-time compute is the prohibitive cost of generating many full solutions, which we refer to as \emph{rollouts} moving forward \citep{pope2023efficiently, kwon2023efficient}. Beyond generation, there is also the cost of evaluating the finished rollouts in best-of-$N$ methods, which score each rollout with a reward model, and this is computationally demanding on a large pool of long reasoning traces \citep{cobbe2021training, lightman2024lets, sun2024fast}. This computational bottleneck, both in generation and in evaluating all $N$ rollouts in best-of-$N$ methods, necessitates test-time pruning of partial rollouts before they reach a final answer and terminate. In other words, we generate all $N$ rollouts only up to a cutoff of $T$ tokens, keep the $k$ most promising of them and run those to the full token budget $B$, as illustrated in \cref{fig:problem}.

One family of methods examines the confidence dynamics of the model during generation, utilizing metrics such as logit entropy or confidence \citep{manvi2024adaptive, yang2026dynamic, hosseini2026early, wang2025entropy}. The underlying intuition is that if a model is overly confident at an intermediate decoding step, extending the reasoning process is a waste of compute; conversely, if the model is entirely lacking in confidence, it is unlikely to reach a correct solution regardless of the additional compute provided, allowing for early termination. Another line of work leverages the model's internal state, specifically the hidden states of its transformer blocks, to identify which rollouts are worth continuing \citep{zhang2025reasoning, afzal2025knowing}. By ``probing" these internal states, one can predict whether the model will ultimately arrive at a correct solution, effectively framing the pruning problem as a classification task. This allows for the early discarding of partial rollouts that are unlikely to yield a correct final answer. In parallel, a separate line of research prioritizes representation exploration (RepExp) to select a subset of rollouts that cover as many distinct reasoning strategies as possible \citep{tuyls2025repexp}. By selecting partial rollouts whose corresponding embeddings are dissimilar, RepExp aims to explore a diverse set of reasoning trajectories.

\begin{figure}[t]
\setlength{\abovecaptionskip}{0pt}
\centering
\includegraphics[width=\textwidth]{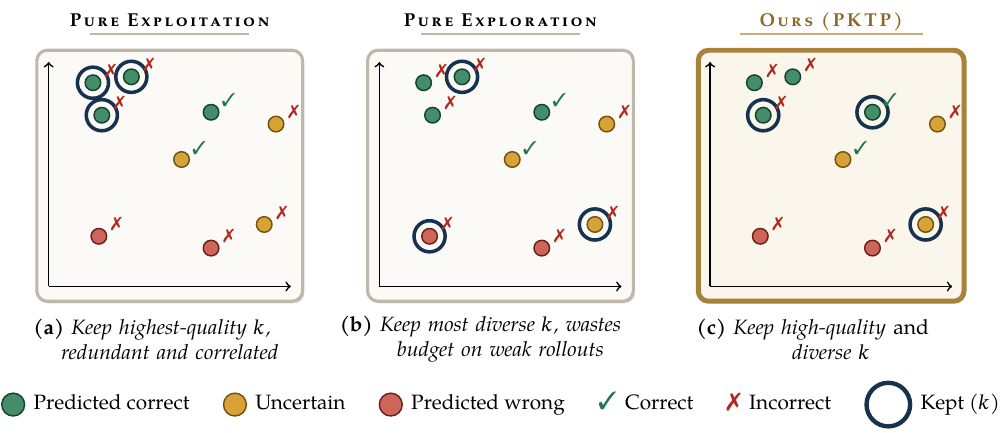}
\caption{\textbf{Conceptual illustration.} The same pool of candidate rollouts, positioned by their partial rollout
representations in embedding space, under three pruning strategies. Each rollout is colored by its
\emph{predicted} correctness and marked \ding{51}/\ding{55} by its \emph{actual} outcome; the $k=3$ kept
rollouts are ringed. \textbf{(a) Pure exploitation} keeps the highest-scoring rollouts, but they are
correlated and cluster together, so they miss the correct rollout. \textbf{(b) Pure exploration} keeps a
maximally diverse subset yet spends its budget on low-quality rollouts, also missing it. \textbf{(c) Our method} optimizes
\passk directly over the kept subset, balancing quality and diversity, and retains the
correct rollout. }
\label{fig:teaser}
\end{figure}

These approaches, however, sit at the two ends of a spectrum, and neither directly optimizes the quantity we actually care about. The problem with probing methods is that the rollouts they rank highest tend to be similar to one another. As a result, the probe may be confidently wrong about all of the kept rollouts at once, and none of them reaches a correct answer. This is the problem of pure exploitation (\cref{fig:teaser}a). RepExp, on the other hand, finds diverse rollouts, but because it ignores the quality of the rollouts, it can spend much of its budget on rollouts that are unlikely to succeed (\cref{fig:teaser}b). These limitations motivate selecting the rollouts by the metric we actually want to optimize, choosing the set $S$ of $k$ rollouts that maximizes \passk, that is, $\max_{|S| = k} \mathrm{pass}@k(S)$, which we formalize in \cref{eq:passAtK}. Optimizing this objective directly keeps the selected rollouts diverse enough to avoid shared failures while still favoring high-quality ones (\cref{fig:teaser}c).

To this end, we propose \textbf{PKTP} (\textbf{P}ass@$\mathbfit{k}$ \textbf{T}hought \textbf{P}runing), a pruning framework that chooses the kept rollouts by maximizing \passk directly. The main difficulty is that, even if the rollouts are sampled independently, their outcomes are not independent (conditioned on the partial observation of the rollout). Hence,  \passk depends on their joint distribution (i.e., rollouts with similar partial rollouts tend to succeed or fail together). We model this joint distribution with a latent Gaussian model, in which the correctness of each rollout is obtained by thresholding a Gaussian score. Evaluating \passk under this model requires an expectation over the latent variables, so we instead derive a closed-form second-order approximation of it. We show that this approximation is submodular, enabling us to approximately solve it using a greedy procedure to find a subset of size~$k$. An important aspect of the proposed framework is that the latent model is trained only once and does not depend on $k$, and a single greedy pass yields a nested sequence of subsets, so the budget $k$ can be chosen online at inference time without retraining.

\paragraph{Contributions.}
First, we cast pruning as maximizing the \passk of the kept set, and we model how similar rollouts fail together.
Second, we derive a closed-form, submodular approximation of \passk that links the probe and RepExp.
Third, through experiments on eight reasoning models, we show that PKTP improves over existing pruning methods, especially for moderate values of $k$.

\section{Problem Formulation and Related Work}
\label{sec:problem}

\paragraph{Setup.}
Given a question $\question$, we draw $N$ candidate reasoning rollouts $y^{(i)} \sim \pi(\,\cdot \mid \question\,)$ from a model~$\pi$. Each rollout is observed only through its first $T$ tokens, the prefix $y_{1:T}^{(i)}$. We encode this prefix as a representation vector $h_i = h\!\left(y_{1:T}^{(i)}\right)$, a hidden-state embedding of the reasoning up to token T. Pruning then keeps $k \leq N$ rollouts, runs them to the full token budget $B$, and discards the rest. A completed rollout $y_{1:B}$ is either correct or not, recorded by the sign $C(y_{1:B}) \in \{-1,+1\}$,
\begin{equation}
    C(y_{1:B}) \;=\; \begin{cases} +1, & \text{if } \Verify\!\left(y_{1:B}\right), \\[2pt] -1, & \text{otherwise}. \end{cases} \label{eq:corrIndicator}
\end{equation}
Here $\Verify$ is a \emph{task-specific checker} that returns true exactly when a completed rollout reaches a correct final answer, for instance by matching a reference solution on a math problem or passing the unit tests on a coding task. We write $C(y^{(i)})$ for the outcome of the $i$-th rollout once it is run to the budget $B$.

\paragraph{Representations.}
To prune partial rollouts, we need to map each prefix to a vector that captures its essential information. One line of work relies on hand-designed features of the rollout and the model's internal states, such as token confidence, the geometry of the hidden-state trajectory, or the convergence of intermediate answers. These features are usually very low dimensional \citep{fu2025deep,vilas2025tracing,sun2026trajectories,jiang2026beyond,liu2025answer}. A second line of work instead keeps the hidden states of the transformer and pools them into a single vector \citep{wang2026truth,tuyls2025repexp,lu2026streaming}, and this is the class we use. In particular, we take $h_i$ to be an exponential moving average (EMA) of the hidden states, which has proven an effective pooling operation \citep{lu2026streaming}. Our improvements are thus shown for this class of embeddings, although our framework can be applied to other embeddings as well. \Cref{app:representation} gives the details of the representation we use.

\paragraph{Objective.}
At pruning time the outcomes $C(y^{(i)})$ are still unknown, since each depends on a (random) continuation we have not yet generated. How, then, should we choose the set $S \subseteq \{1, \ldots, N\}$ of $|S| = k$ rollouts to keep? Two extreme choices are as follows.
\begin{enumerate}[label=(\roman*), leftmargin=*]
    \item \textbf{(Pure exploitation)} Keep in $S$ the rollouts $\{i\}$ with the highest predicted success $\Prob\!\bigl(C(y^{(i)}) = +1 \mid \question, y_{1:T}^{(i)}\bigr)$, estimated by a linear probe trained on the representations $h_i$ \citep{zhang2025reasoning}, that is, the $k$ rollouts with the largest scores $w^\top h_i$ for a learned weight vector $w$. We refer to this pruning policy as \textbf{Linear Probe}.
    \item \textbf{(Pure exploration)} Build a diverse $S$, starting from a random rollout and greedily adding to $S$ the rollout $i^\star$ by solving 
    \begin{equation}
        i^\star = \arg\max_{i \notin S}\; h_i^\top \Bigl(\tau I + \sum_{j \in S} h_j h_j^\top\Bigr)^{-1} h_i ,
    \end{equation}
    that is, we add, the rollout whose representation is least explained by those already chosen, where $\tau > 0$ is a regularizer. This is the \textbf{RepExp} pruning policy \citep{tuyls2025repexp}.
\end{enumerate}

Neither extreme  balances quality against diversity. What we actually want is the set of $k$ rollouts with the best \passk. Mathematically, we seek the set $S$  most likely to contain a correct rollout,
\begin{objectivebox}
\begin{equation}
    \max_{S \,\subseteq\, \{1, \ldots, N\}} \Prob\!\left(\,\sum_{i \in S} \mathbf{1}\{C(y^{(i)}) = +1\} \;\geq\; 1 \;\middle|\; \question,\, y_{1:T}^{(1)}, \ldots, y_{1:T}^{(N)}\,\right) \quad \text{s.t.} \quad |S| = k. \label{eq:passAtK}
\end{equation}
\end{objectivebox}
Here, $\mathbf{1}\{\cdot\}$ is the indicator function, i.e., $\mathbf{1}\{True\} = 1$ and $\mathbf{1}\{False\} = 0$.
The objective in \cref{eq:passAtK} is exactly  \passk of the kept pool, the probability that at least one surviving rollout, once run to the budget $B$, yields a correct answer. We also assume that the final outcomes are binary, although the same approach generalizes to continuous outcomes, where each completed rollout is scored by a reward model and the goal is to maximize the best of the kept rollouts \citep{chow2024inference}, in which case \cref{eq:passAtK} becomes a max@$k$ objective \citep{walder2025pass}.

\paragraph{Related work.}
Existing pruning methods score each rollout on its own, using the model's confidence \citep{manvi2024adaptive, fu2025deep}, a reward model \citep{sun2024fast}, or a probe on its hidden states \citep{zhang2025reasoning, step2026}, or they select a diverse set, as RepExp does \citep{tuyls2025repexp}. Pass@$k$ has also been used as a training objective for reasoning models \citep{walder2025pass, chen2025passk}, but not to decide which partial rollouts to keep. Outside of language models, choosing a set so that at least one item succeeds has been studied in information retrieval and Bayesian optimization \citep{chen2006less, sanner2011diverse, wilson2018maximizing}. \Cref{app:related} gives an extended discussion. In what follows, we discuss how such an objective can be solved.

\section{Method}
\label{sec:method}

There are some challenges in solving \cref{eq:passAtK}: First, we need to obtain the \passk of a given set of rollouts. We address this by learning the joint distribution of the rollouts' outcomes under a probabilistic model, from which the \passk of any set can be estimated. Second, even with such an estimate, finding the subset of $k$ rollouts with the highest estimated \passk is a mixed-integer program. We can therefore view our framework as two components, learning the joint distribution of the outcomes (\cref{sec:prob-model}), and solving this mixed-integer program at selection time (\cref{sec:subset-selection}).

\subsection{Probabilistic model}
\label{sec:prob-model}

Let $H = [h_1, \ldots, h_N] \in \R^{d \times N}$ collect the representations of the $N$ generated partial rollouts, $Q_\theta:\R^{d} \to \R$ be a learnable mapping predicting the rollout quality, and define $Q_\theta(H) = [q_1, \ldots, q_N]^\top$ with $q_i := Q_\theta(h_i)$. We model the latent scores of the rollouts as
\begin{equation}
    \mathbf{Z} = \underbrace{Q_\theta(H)\vphantom{H^\top}}_{\text{Rollout Quality}} + \underbrace{H^\top \mathbf{f}}_{\text{Rollouts Correlation}} + \underbrace{\mathbf{m}\vphantom{H^\top}}_{\text{Outcome Noise}}, \qquad \mathbf{m} \sim \mathcal{N}(0, I_N), \quad \mathbf{f} \sim \mathcal{N}(0, I_d), \label{eq:model}
\end{equation}
where $\mathbf{m}$ and $\mathbf{f}$ are independent. Intuitively, $\mathbf{Z}_i$ is a score of how promising rollout $i$ is, so the higher its quality $q_i$, the higher its score. On top of $q_i$, the term $h_i^\top \mathbf{f}$ is a random effect shared by rollouts with similar representations, which makes them succeed or fail together. Indeed, for distinct rollouts, $\operatorname{Cov}(\mathbf{Z}_i,\mathbf{Z}_j\mid H)=h_i^\top h_j$, so rollouts with aligned $h_i$ have positively correlated outcomes. Finally, $\mathbf{m}_i$ is noise specific to  rollout~$i$.

The correctness $C(y^{(i)})$ of rollout $i$ is then modeled by thresholding its score at zero. More specifically, we define the random variable  $\mathbf{C}_i = 2 \times \mathbf{1}\{\mathbf{Z}_i \geq 0\} - 1$.
In other words, a positive score implies that the rollout is correct, while a negative score implies that it is incorrect.
Since the scores are independent conditionally on $\mathbf{f}$, the likelihood of an observed correctness pattern $c \in \{-1,1\}^N$ is
\begin{equation}
    \Pr\bigl[\mathbf{C} = c \mid H\bigr] = \E_{\mathbf{f}}\Bigl[\prod_{i=1}^N \Phi\bigl(c_i(q_i + h_i^\top \mathbf{f})\bigr)\Bigr], \label{eq:likelihood}
\end{equation}
where $\Phi$ is the standard Gaussian CDF, and the derivation can be found in \cref{app:likelihood}.

Suppose we have $M$ training questions, where question $m$ consists of the representations $H^{(m)}$ of a pool and its observed outcomes $c^{(m)} \in \{-1,1\}^N$. The mapping $Q_\theta$ is trained by maximizing their likelihood under \cref{eq:likelihood}, that is, by solving
\begin{equation}
    \max_{\theta}\ \frac{1}{M}\sum_{m=1}^{M}\log\Pr\nolimits_{\theta}\bigl[\mathbf{C} = c^{(m)} \mid H^{(m)}\bigr], \label{eq:mle}
\end{equation}
where $\Pr_\theta$ is the likelihood of \cref{eq:likelihood} with $H = H^{(m)} = [h_1^{(m)}, \ldots, h_N^{(m)}]$ and $q_i = Q_\theta(h_i^{(m)})$. Solving \cref{eq:mle} directly is hard, since the expectation over $\mathbf{f}$ in \cref{eq:likelihood} has no closed form. We therefore expand each log-factor to second order in $h_i^\top \mathbf{f}$, which turns the expectation into a Gaussian integral that can be evaluated exactly (\cref{thm:surrogate}).

\begin{theorem}[Second-order surrogate]
\label{thm:surrogate}
For a pattern $c \in \{-1,1\}^N$, let $\phi$ be the standard Gaussian density, $\lambda_i(c_i) = \phi(q_i)/\Phi(c_i q_i)$, $\kappa_i(c_i) = \lambda_i(c_i)\bigl(\lambda_i(c_i) + c_i q_i\bigr)$,
\begin{equation}
    \Sigma(c) = \sum_{i=1}^{N} \kappa_i(c_i)\, h_i h_i^\top, \qquad b(c) = \sum_{i=1}^{N} c_i \lambda_i(c_i)\, h_i .
\end{equation}
Then, as $\varepsilon \to 0$,
\begin{equation}
    \log\Phi\bigl(c_i(q_i + \varepsilon)\bigr) = \log\Phi(c_i q_i) + c_i\lambda_i(c_i)\, \varepsilon - \frac{\kappa_i(c_i)}{2}\varepsilon^2 + O(|\varepsilon|^3),
\end{equation}
and replacing each log-factor of $\Pr[\mathbf{C} = c \mid H]$ in \cref{eq:likelihood} by this expansion at $\varepsilon = h_i^\top \mathbf{f}$ gives
\begin{equation}
    \widetilde{\Pr}[\mathbf{C} = c \mid H] = \Bigl(\prod_{i=1}^{N}\Phi(c_i q_i)\Bigr)\det\bigl(I_d + \Sigma(c)\bigr)^{-1/2}\exp\Bigl(\frac12\, b(c)^\top\bigl(I_d + \Sigma(c)\bigr)^{-1}b(c)\Bigr). \label{eq:surrogate-prob}
\end{equation}
\end{theorem}

The proof is in \cref{app:proof-surrogate}. The expansion is local, so we use \cref{eq:surrogate-prob} as a tractable surrogate rather than as an exact identity or a bound. We therefore replace $\Pr_\theta$ in \cref{eq:mle} by $\widetilde{\Pr}_\theta$.

Instead of using $h_i$ directly in the correlation term, one could learn an encoder $\mathcal{E}_\psi$ of the representations and use $\mathcal{E}_\psi(h_i)$ in its place, but its helpfulness depends on the amount of data. Since our data is limited, we keep $h_i$ fixed, so that $Q_\theta$ is the only learned component, and for the same reason we use a simple linear model $Q_\theta(h) = \theta^\top h + \theta_0$ for $Q_\theta$, as in the Linear Probe \citep{zhang2025reasoning}. Unlike the Linear Probe, which is trained on the correctness of each rollout separately, $Q_\theta$ is trained on the joint outcomes of all rollouts of a question, so the quality it learns already accounts for the correlation between them. Since \cref{eq:surrogate-prob} is in closed form and differentiable in $\theta$, we maximize the surrogate likelihood by stochastic gradient ascent over the training questions.

\subsection{Subset selection}
\label{sec:subset-selection}

Maximizing \passk in \cref{eq:passAtK} is the same as minimizing the probability that every kept rollout fails. Under the latent Gaussian model, and suppressing the conditioning on $H$, the failure probability of a set $S$ of kept rollouts is
\begin{equation}
    L(S) = \Pr\bigl[\mathbf{C}_i = -1 \ \forall i \in S\bigr] = \E_{\mathbf{f} \sim \mathcal{N}(0, I_d)}\Bigl[\prod_{i \in S}\Phi\bigl(-(q_i + h_i^\top \mathbf{f})\bigr)\Bigr], \label{eq:selection}
\end{equation}
and we seek the set of $k$ rollouts that minimizes it. Since $L(S)$ has no closed form, we approximate it with \cref{thm:surrogate}.

\begin{corollary}[Selection surrogate]
\label{cor:selection}
For a set $S \subseteq [N] := \{1, \ldots, N\}$, let $\lambda_i = \phi(q_i)/\Phi(-q_i)$, $\kappa_i = \lambda_i(\lambda_i - q_i)$, $\Sigma_S = \sum_{i \in S}\kappa_i\, h_i h_i^\top$ and $b_S = \sum_{i \in S}\lambda_i\, h_i$. Applying \cref{thm:surrogate} to the rollouts in $S$ with $c_i = -1$ for every $i \in S$ gives a surrogate $\widetilde L(S)$ of $L(S)$ with
\begin{equation}
    \log\widetilde L(S) = \underbrace{\sum_{i \in S}\log\Phi(-q_i)}_{\textnormal{Rollout Quality}} \;+\; \underbrace{\vphantom{\sum_{i \in S}}\Bigl(-\frac12\log\det(I_d + \Sigma_S) + \frac12\, b_S^\top(I_d + \Sigma_S)^{-1}b_S\Bigr)}_{\textnormal{Rollouts Correlation}} . \label{eq:surrogate}
\end{equation}
\end{corollary}

The proof is in \cref{app:proof-corollary}. We therefore minimize $\log\widetilde L(S)$ over $|S| = k$. Each term of \cref{eq:surrogate} has a direct interpretation. The first term favors large $q_i$, because the baseline failure factor $\Phi(-q_i)$ decreases with $q_i$. The negative log-determinant rewards a collection of $h_i$'s that covers several directions instead of repeatedly selecting the same direction. The last term is nonnegative and penalizes sets whose weighted vectors $\lambda_i h_i$ add up in the same direction, which the determinant alone cannot detect since it treats $h_i$ and $-h_i$ identically. For example, two rollouts of the same quality fail together more often with $h_2 = h_1$ than with $h_2 = -h_1$, yet both cases give the same $\Sigma_S$ and hence the same determinant, while $b_S$ is $2\lambda_1 h_1$ in the first case and $0$ in the second.

However, minimizing $\log\widetilde L(S)$ exactly is hard.

\begin{proposition}[Hardness of exact selection]
\label{prop:nphard}
Minimizing $\log\widetilde L(S)$ over all $S \subseteq [N]$ with $|S| = k$ is NP-hard.
\end{proposition}

The proof is in \cref{app:nphard}, by a reduction from the maximum independent set problem. When the objective is monotone and submodular, however, a greedy optimizer that adds one rollout at a time is guaranteed to achieve at least a $(1 - e^{-1})$ fraction of the optimal value \citep{nemhauser1978analysis}. The exact pass probability $P(S) := 1 - L(S)$ is known to be monotone and submodular \citep{wilson2018maximizing}. An approximation of a submodular function, though, is not necessarily submodular, so we prove that its surrogate $\widetilde P(S) := 1 - \widetilde L(S)$ is also monotone and submodular, and hence the greedy optimizer and its approximation guarantee apply to our problem.

\begin{theorem}[Greedy guarantee]
\label{thm:greedy-main}
$\widetilde P$ is normalized ($\widetilde P(\varnothing) = 0$), monotone ($\widetilde P(A) \leq \widetilde P(B)$ whenever $A \subseteq B$), and submodular, that is, for all $A \subseteq B \subseteq [N]$ and $j \notin B$,
\begin{equation}
    \widetilde P(A \cup \{j\}) - \widetilde P(A) \geq \widetilde P(B \cup \{j\}) - \widetilde P(B).
\end{equation}
Consequently, the set $S_k^{\mathrm{greedy}}$ obtained by adding, $k$ times, the candidate that most increases $\widetilde P$ satisfies
\begin{equation}
    \widetilde P(S_k^{\mathrm{greedy}}) \geq (1-e^{-1})\max_{|S|\leq k}\widetilde P(S).
\end{equation}
\end{theorem}

The proof is in \cref{app:proof-greedy}.

\paragraph{Greedy optimization.}
The greedy optimizer starts from the empty set and, at each step, adds the candidate $j \notin S$ with the smallest change $\Delta_j(S) = \log\widetilde L(S \cup \{j\}) - \log\widetilde L(S)$. Since adding a candidate changes $\Sigma_S$ by a rank-one matrix, this change has a closed form that can be updated cheaply without rebuilding the determinant and the inverse (\cref{app:greedy}). With these rank-one updates, the greedy optimizer costs $O\bigl(k(Nd + d^2)\bigr)$ time and $O(Nd + d^2)$ memory. By comparison, a naive implementation that rebuilds and factorizes $I_d + \Sigma_{S \cup \{j\}}$ for every candidate costs $O(d^3)$ per candidate, that is, $O(kNd^3)$ time, with the same $O(Nd + d^2)$ memory.

\subsection{Relation to Linear Probe and RepExp}
\label{sec:relation}

The surrogate makes the quality--diversity trade-off explicit and connects to the two baseline criteria in limiting regimes. If the correlation term in \cref{eq:surrogate} is switched off, then $\log\widetilde L(S)=\sum_{i\in S}\log\Phi(-q_i)$, so minimizing it keeps the $k$ largest quality scores and recovers a quality-only, probe-style ranking; it coincides with Linear Probe when $Q_\theta$ is the corresponding linear probe. At the other extreme, suppose that the quality scores are equal and that the penalty term is negligible throughout selection. Then every $\kappa_i$ equals a common value $\kappa$, so $\Sigma_S = \kappa\sum_{i \in S}h_i h_i^\top$ and the relevant set-dependent term is $-\tfrac12\log\det(I_d+\Sigma_S)$. By the matrix determinant lemma, its greedy marginal ranks a candidate $j$ by $h_j^\top(I_d+\Sigma_S)^{-1}h_j$, the residual-variance criterion used by RepExp on the same representations (see \cref{app:repexp-limit} for the derivation). Between these limits, the modular quality, log-determinant diversity, and penalty terms jointly determine the selected set.

The above two extremes do not show how PKTP moves from Linear Probe to RepExp as $k$ grows. To see this, we expand $\log L(S)$ and its surrogate when the correlations are weak.
\begin{proposition}[Weak correlations]
\label{prop:weak}
Fix $S$, $q_i$ and $h_i$, and replace $\mathbf f$ by $\rho\,\mathbf f$. As $\rho\to0$,
\begin{equation}
 \log L(S)=\sum_{i\in S}\Bigl[\log\Phi(-q_i)+\tfrac{\rho^2}{2}\lambda_iq_i\|h_i\|^2\Bigr]+\rho^2\sum_{i<j\in S}\lambda_i\lambda_j\,h_i^\top h_j+O(\rho^4),
\end{equation}

and $\log\widetilde L(S)$ has the same expansion.
\end{proposition}
\Cref{prop:weak} shows how the correlation strength $\rho$ controls where PKTP sits between the two baselines. At $\rho = 0$ the rollouts are independent, the expansion reduces to $\sum_{i\in S}\log\Phi(-q_i)$, and PKTP ranks the rollouts by quality, as Linear Probe does. For $\rho > 0$, every pair of kept rollouts adds a penalty $\rho^2\lambda_i\lambda_j h_i^\top h_j$, which is large when their representations are similar. A rollout added to a set of $t$ rollouts pays $t$ such penalties, so the first pick is decided by quality alone, and as the set grows the penalties add up until they outweigh the differences in quality, after which PKTP spreads its picks out, as RepExp does. The larger $\rho$ is, the sooner this happens, so $\rho$ sets how fast PKTP moves from Linear Probe to RepExp as $k$ grows.

\Cref{s2:section} makes the two ends of this behavior precise for a fixed $\rho$. At small $k$, if the gap between the $k$th and $(k+1)$th largest qualities exceeds a constant multiple of $k\rho^2$, PKTP keeps exactly the $k$ candidates with the largest $q_i$, as Linear Probe does, and in the same order when this holds at every budget up to $k$. At large $k$, when a few candidates of high quality stand out from a bulk of lower quality, the objective of PKTP, averaged over the bulk qualities, reduces to the RepExp objective plus a term for these candidates. Its minimizers therefore keep them and nearly maximize the RepExp objective. Finally, the conditions of these two results do not conflict. We construct a single model, with one quality distribution and one value of $\rho$, in which PKTP follows Linear Probe at small $k$, while its averaged objective reduces to the RepExp objective as $k$ grows.

\section{Experiments}
\label{sec:experiments}

\subsection{Setup}
\label{sec:setup}

\paragraph{Models.}
We evaluate on eight open-weight reasoning models in two size tiers. The small tier is Qwen3-0.6B, Qwen3-1.7B \citep{yang2025qwen3}, OpenMath-Nemotron-1.5B \citep{moshkov2025aimo} and DeepSeek-R1-Distill-Qwen-1.5B \citep{guo2025deepseek}, and the large tier is Qwen3-8B, OpenMath-Nemotron-7B, DeepSeek-R1-Distill-Qwen-7B and Skywork-OR1-Math-7B \citep{he2025skywork}. Every model generates its rollouts with temperature $0.8$, top-$p$ of $0.95$, and a full token budget of $B = 32{,}768$ tokens.

\paragraph{Datasets.}
We pool problems from several benchmark sources rather than a single one, since a pruning policy has to be trained on part of the problems and evaluated on the rest, and no single competition benchmark is large enough for both. The two model tiers are evaluated on different benchmark substrates, chosen so that neither tier is saturated. Both tiers use AIME 2024, 2025 and 2026 and HMMT 2024, 2025 and 2026 \citep{balunovic2025matharena}, together with RIMO-N \citep{chen2025rimo}. The small tier adds BeyondAIME \citep{seed2025thinking}, OlymMATH (Easy) \citep{sun2025challenging} and Omni-MATH levels 7 and 8 \citep{gao2024omni}, while the large tier replaces these by AMO-Bench \citep{an2025amo}, OlymMATH (Hard) and Omni-MATH levels 8 and 9. This gives around $600$ problems per model. For each model we then remove the problems on which it generates all correct or all wrong solutions in $128$ trials.

\paragraph{Verification.}
A completed rollout counts as correct when its final answer matches the reference answer of the problem. We take that answer from the \verb|\boxed{}| expression of the rollout and ask a Qwen3-32B judge \citep{yang2025qwen3} whether it is equivalent to the reference answer. We use a judge rather than a symbolic checker such as \texttt{math-verify} \citep{kydlicek2025mathverify}, because benchmarks like AMO-Bench have descriptive answers that a symbolic checker cannot handle.

\paragraph{Evaluation.}
Only a few hundred problems are available for each model, so a single split into training and test problems would leave too few on either side, and we evaluate with cross-validation instead. For each model we partition its problems into eight folds, with each benchmark spread evenly across the folds. Each fold serves once as the test set while the learned components are trained on the other seven, so every problem is scored exactly once by a model that was not trained on it. One such pass over all eight folds forms a repetition. Because the outcome of a repetition depends on how the problems happen to be divided, we run $20$ repetitions, each with a new random partition obtained by shuffling the problems of every benchmark before assigning them to folds, and with fresh random seeds for training. Every reported number is the average over the $20$ repetitions, and its standard error is computed across them. For RepExp we set the regularizer to $\tau = 1$. For all methods, the representation $h_i$ of a partial rollout is the EMA of its hidden states with rate $10^{-3}$.

\begin{figure}[t]
\setlength{\abovecaptionskip}{2pt}
\centering
\includegraphics[width=\textwidth]{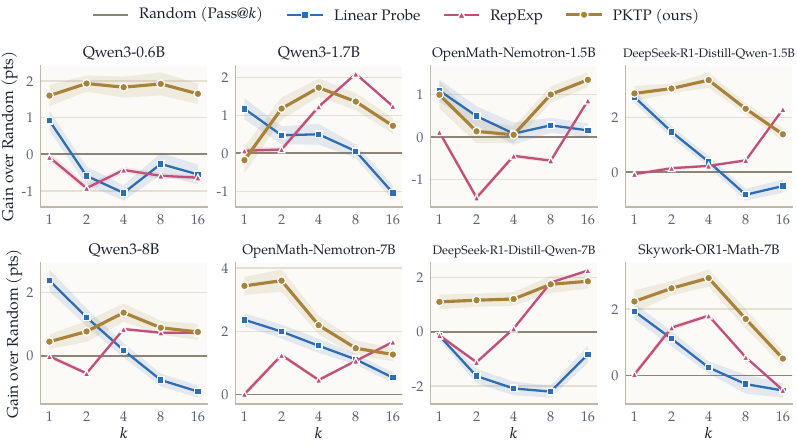}
\caption{\textbf{Gain in \passk over random pruning.} For every model, the pooled \passk of each pruning policy minus that of Random (Pass@$k$), in percentage points, as a function of the number $k$ of rollouts kept, with a band of one standard error of the repetition-by-repetition difference. The absolute numbers are in \cref{tab:main} and the per-benchmark breakdown in \cref{app:per-benchmark}.}
\label{fig:gains}
\end{figure}

\subsection{Experimental Evaluation and Comparison}
\label{sec:results}

\Cref{fig:gains} reports our main setting, in which each method selects $k \in \{1, 2, 4, 8, 16\}$ of the $N = 128$ rollouts of a problem from their prefixes, and we set the cutoff threshold to $T = 4000$ tokens. For every model it shows the \passk of each pruning policy minus that of Random (Pass@$k$), pooled over all test problems. Random keeps $k$ rollouts chosen uniformly at random, which is equivalent to sampling $k$ rollouts without any pruning, so it is the natural reference point. The absolute numbers are given in \cref{tab:main} and the breakdown by benchmark in \cref{app:per-benchmark}.

\begin{figure}[t]
\centering
\includegraphics[width=0.48\textwidth]{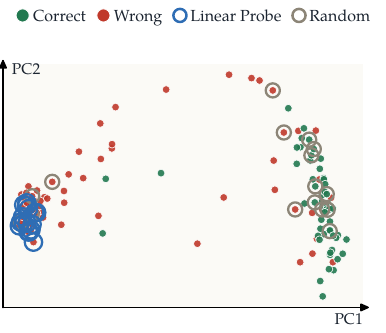}
\caption{\textbf{Linear Probe keeps a cluster of rollouts that fail together.} The $128$ rollouts of one problem (Qwen3-8B, AMO-Bench), projected on the first two principal components of their representations, with correct rollouts in green and wrong rollouts in red. The $16$ rollouts ranked highest by the Linear Probe (blue rings) lie close together and are all wrong, while $16$ random picks (gray rings) include $7$ correct rollouts.}
\label{fig:probe-cluster}
\end{figure}

As evident in \cref{fig:gains}, the Linear Probe performs strongly at $k = 1$, but as $k$ increases it gets closer and closer to Random, as in the Qwen models, and in some cases it even becomes worse than Random at large $k$. This can be attributed to a bias towards a cluster of rollouts that the model predicts to be correct but that may not be, in the case where the classifier is wrong. In other words, ranking by the probe sharpens the distribution of the kept rollouts around what the probe considers correct. This mirrors the effect observed in reinforcement learning, where models trained with GRPO outperform their base models at small $k$ but are outperformed by them at large $k$ \citep{yue2025does}. As an example, \cref{fig:probe-cluster} shows the $128$ rollouts of one problem (Qwen3-8B, AMO-Bench), projected on the first two principal components of their representations. The $16$ rollouts ranked highest by the probe are all close to each other and all wrong, while $16$ random picks include $7$ correct ones.

\begin{figure}[t]
\centering
\includegraphics[width=0.48\textwidth]{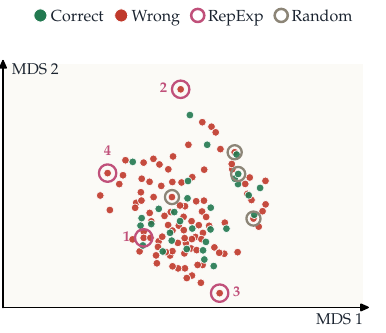}
\caption{\textbf{RepExp keeps outliers that are unlikely to be correct.} The $128$ rollouts of one problem (Qwen3-8B, Omni-MATH), embedded in two dimensions by multidimensional scaling of their representations, with correct rollouts in green and wrong rollouts in red. After a random first pick, RepExp (pink rings, numbered in pick order) selects rollouts far from the rest, and all $4$ of its picks are wrong, while $4$ random picks (gray rings) include $1$ correct rollout.}
\label{fig:repexp-example}
\end{figure}
\Cref{fig:gains} also shows that RepExp behaves in the opposite way to the Linear Probe. At $k = 1$ it shows, as expected, a performance similar to Random, since its first pick is random. As $k$ increases, the general trend is that RepExp gets better and better than Random, especially at $k = 16$ (e.g., the DeepSeek-R1-Distill models). Interestingly, however, for moderate $k$ it is sometimes outperformed by Random, for instance on OpenMath-Nemotron-1.5B and Qwen3-8B. We attribute this to a tendency of RepExp to favor outliers at small and moderate $k$, since the rollouts farthest from those already selected are often unusual rollouts that are unlikely to be correct. A similar effect is reported by \citet{tuyls2025repexp}, where RepExp stops helping under high-temperature sampling, since the less coherent responses look novel in representation space yet do not necessarily contain correct answers. As an example, \cref{fig:repexp-example} shows the $128$ rollouts of one problem (Qwen3-8B, Omni-MATH), embedded in two dimensions with multidimensional scaling on their representations. After its first, random pick, RepExp selects outliers far from the other rollouts, and all $4$ of its picks are wrong, while $4$ random picks include $1$ correct one.

Finally, \cref{fig:gains} shows that PKTP outperforms both the Linear Probe and RepExp, and its advantage is clearest at moderate $k$, where each baseline suffers from its own bias. On DeepSeek-R1-Distill-Qwen-1.5B at $k = 4$, for example, PKTP is $3.0$ points above the Linear Probe and $3.2$ points above RepExp. On Qwen3-0.6B it leads both baselines by more than $2$ points at every $k \geq 2$, and on DeepSeek-R1-Distill-Qwen-7B it is $3.3$ points above the Linear Probe at $k = 4$. Unlike the two baselines, PKTP also stays above Random on every model for $k \geq 2$, since it keeps the correctness signal of the probe while spreading the kept rollouts over outcomes that are unlikely to fail together. At large $k$, however, the performance of PKTP and RepExp gets closer. We believe this is due to the weak correctness signal in the embeddings of the partial rollouts, which raises the question of how a quality embedding can be obtained from a partial rollout, and whether this is possible at all.

Results for the other pool sizes, $N \in \{32, 48, 64, 80, 96, 112\}$ with the same cutoff $T = 4000$, are given in \cref{app:pool-sizes}.

\subsection{Effect of the cutoff threshold}
\label{sec:cutoff-ablation}

\begin{figure}[t]
\centering
\includegraphics[width=0.57\textwidth]{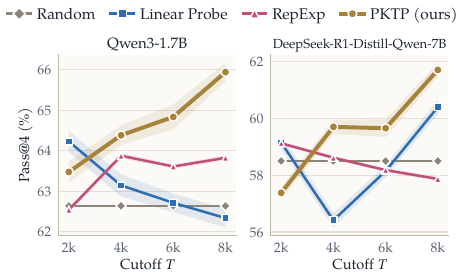}
\caption{\textbf{A later cutoff improves PKTP.} Pass@$4$ on Qwen3-1.7B and DeepSeek-R1-Distill-Qwen-7B as the cutoff $T$ grows from $2000$ to $8000$ tokens, with $N = 128$. PKTP is the best method for every $T \geq 4000$ and keeps improving with longer prefixes, while the Linear Probe and RepExp gain much less, and the Linear Probe on Qwen3-1.7B even drifts below Random. Bands are one standard error.}
\label{fig:cutoff}
\end{figure}
An important aspect of pruning is the cutoff threshold $T$, and how the algorithm behaves as the cutoff grows. \Cref{fig:cutoff} varies $T$ on Qwen3-1.7B and DeepSeek-R1-Distill-Qwen-7B. For every $T \geq 4000$, PKTP outperforms all other methods on both models. An interesting observation is that a later cutoff in general improves the performance of PKTP. On Qwen3-1.7B its pass@$4$ rises from $63.5\%$ at $T = 2000$ to $65.9\%$ at $T = 8000$, and on DeepSeek-R1-Distill-Qwen-7B from $57.4\%$ to $61.7\%$. Intuitively, the more tokens are generated, the more is revealed about the quality of a rollout and where it is heading, so there is more signal to exploit. Interestingly, however, the same does not necessarily happen for the Linear Probe and RepExp. On Qwen3-1.7B, for example, the Linear Probe drifts down from $64.2\%$ to $62.3\%$ and ends below Random. One reason can be the representation we use. Since the EMA with rate $10^{-3}$ effectively averages over the last $1000$ tokens, it may lose the signal in the earlier tokens.

\subsection{Effect of the pool size}
\label{sec:pool-ablation}

\begin{figure}[t]
\centering
\includegraphics[width=0.57\textwidth]{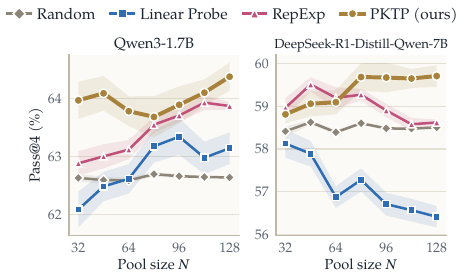}
\caption{\textbf{PKTP leads once the pool is large enough.} Pass@$4$ on Qwen3-1.7B and DeepSeek-R1-Distill-Qwen-7B as the pool size $N$ grows from $32$ to $128$, with $T = 4000$. PKTP is the best method for every $N$ on Qwen3-1.7B and from $N = 80$ on DeepSeek-R1-Distill-Qwen-7B, where the Linear Probe falls as the pool grows and stays below Random. Bands are one standard error.}
\label{fig:pool}
\end{figure}
Another important aspect of pruning is the size $N$ of the pool the $k$ rollouts are chosen from. \Cref{fig:pool} varies $N$ on Qwen3-1.7B and DeepSeek-R1-Distill-Qwen-7B. PKTP outperforms all other methods for every $N$ on Qwen3-1.7B, and on DeepSeek-R1-Distill-Qwen-7B once the pool reaches $N = 80$. A larger pool offers more candidates to choose from, and on Qwen3-1.7B both baselines benefit from it, with RepExp rising from $62.9\%$ at $N = 32$ to $63.9\%$ at $N = 128$ as the pool becomes more diverse and the Linear Probe from $62.1\%$ to $63.1\%$. Interestingly, however, the same does not hold on DeepSeek-R1-Distill-Qwen-7B, where the Linear Probe falls from $58.1\%$ to $56.4\%$ and stays below Random for every $N$, and RepExp falls from $59.5\%$ at $N = 48$ to $58.6\%$, while PKTP rises from $58.8\%$ to $59.7\%$. One reason can be that a larger pool also contains more outliers, which raises the chance that extreme cases like RepExp and the Linear Probe select them at a moderate $k$ such as $k = 4$.

\section{Conclusion}
\label{sec:conclusion}

We introduced PKTP, a pruning framework that chooses which partial rollouts to keep by directly optimizing \passk. PKTP learns the joint distribution of the rollouts' outcomes with a latent Gaussian model, approximates the resulting likelihood in closed form, and selects the kept set greedily, with an approximation guarantee on the surrogate objective. Across eight reasoning models, it outperforms both quality-based and diversity-based pruning, and the gains are largest at moderate $k$, where each of these baselines suffers from its own bias. Since PKTP can be combined with any representation of partial rollouts, representations that better capture the correctness of a rollout would directly strengthen it, and developing them is a natural direction for future work.

\section*{Acknowledgments}
This work was supported in part by a Generative AI Research Award from the USC Office of Research and Innovation, the NSF NRT program, the USC-Meta Center for Research and Education in AI and Learning, and a gift from Amazon.  AJ was supported in part by the NSF Award DMS-2311024, an Amazon Faculty Research Award, an Adobe Faculty Research Award, and an iORB grant from USC Marshall School of Business.

\section*{AI use statement}
In this work, we used generative AI tools to assist with several parts of the project. For the write-up, we used them to edit and polish the text and to format the \LaTeX{} source. For the figures, we used them to write the scripts that plot our experimental results. For the code, we used them to help implement parts of our implementation. For the theory, we used them to help derive and write parts of the proofs in the appendix. We have not used generative AI tools to generate data or experimental results. All AI-assisted content, including the code and the proofs, was reviewed and verified by the authors, who take full responsibility for the content of this paper.

\clearpage
\bibliographystyle{plainnat}
\bibliography{references}

\clearpage
\appendix
\crefalias{section}{appendix}
\crefalias{subsection}{appendix}
\crefalias{subsubsection}{appendix}
\section{Derivations}
\label{app:derivations}

\subsection{The likelihood}
\label{app:likelihood}

Write $\mathbf{Z}_i = q_i + \mathbf{m}_i + h_i^\top \mathbf{f}$. Conditioned on $\mathbf{f}$, the $\mathbf{Z}_i$ are independent of each other, so
\begin{align}
    \Pr\left[\mathbf{C}_1 = c_1, \ldots, \mathbf{C}_N = c_N \mid \mathbf{f}\right] &= \prod_{i=1}^N\Pr[c_i \mathbf{Z}_i \geq 0 \mid \mathbf{f}] \\
    &= \prod_{i=1}^N\Pr[c_i (q_i + \mathbf{m}_i + h_i^\top \mathbf{f}) \geq 0 \mid \mathbf{f}] \\
    &= \prod_{i=1}^N\Pr[c_i \mathbf{m}_i \geq - c_i(q_i + h_i^\top \mathbf{f}) \mid \mathbf{f}].
\end{align}
Since the standard Gaussian is symmetric about $0$, $-c_i\mathbf{m}_i$ follows the same distribution as $\mathbf{m}_i$, and with $\Phi(x) := \Pr_{Z \sim \mathcal{N}(0, 1)}[Z \leq x]$ each factor equals $\Phi\bigl(c_i(q_i + h_i^\top \mathbf{f})\bigr)$. Taking the expectation with respect to $\mathbf{f}$ on both sides gives \cref{eq:likelihood}.

\subsection{Rank-one greedy updates}
\label{app:greedy}

As in \cref{cor:selection}, let $\Sigma_S = \sum_{i \in S}\kappa_i h_i h_i^\top$ and $b_S = \sum_{i \in S}\lambda_i h_i$. Write $\Omega_S = (I_d + \Sigma_S)^{-1}$, $r_j = h_j^\top \Omega_S h_j$ and $u_j = h_j^\top \Omega_S b_S$. Adding a candidate $j \notin S$ changes $\Sigma_S$ by the rank-one matrix $\kappa_j h_j h_j^\top$. The matrix determinant lemma gives
\begin{equation}
\det(I_d + \Sigma_S + \kappa_j h_j h_j^\top) = \det(I_d + \Sigma_S)\bigl(1 + \kappa_j r_j\bigr), \label{eq:det-lemma}
\end{equation}
while Sherman--Morrison gives
\begin{equation}
(I_d + \Sigma_S + \kappa_j h_j h_j^\top)^{-1} = \Omega_S - \frac{\kappa_j \Omega_S h_j h_j^\top \Omega_S}{1 + \kappa_j r_j}. \label{eq:sherman-morrison}
\end{equation}
Substituting \cref{eq:det-lemma} and \cref{eq:sherman-morrison} into \cref{eq:surrogate}, and using $b_{S\cup\{j\}} = b_S + \lambda_j h_j$, the corresponding change in the unhalved quadratic form is
\begin{equation}
2\lambda_j u_j + \lambda_j^2 r_j - \frac{\kappa_j(u_j + \lambda_j r_j)^2}{1 + \kappa_j r_j} = \frac{2\lambda_j u_j + \lambda_j^2 r_j - \kappa_j u_j^2}{1 + \kappa_j r_j} .
\end{equation}
Half of this, together with the change $-\tfrac12\log(1 + \kappa_j r_j)$ in the determinant term and the new quality term $\log\Phi(-q_j)$, yields
\begin{equation}
    \Delta_j(S) = \log\widetilde L(S \cup \{j\}) - \log\widetilde L(S) = \log\Phi(-q_j) - \frac12\log(1 + \kappa_j r_j) + \frac{2\lambda_j u_j + \lambda_j^2 r_j - \kappa_j u_j^2}{2(1 + \kappa_j r_j)} . \label{eq:greedy}
\end{equation}
The greedy step adds the candidate with the smallest $\Delta_j(S)$. After selecting $j$, the maintained state updates are
\begin{align}
\Omega_{S\cup\{j\}} = \Omega_S - \frac{\kappa_j \Omega_S h_j h_j^\top \Omega_S}{1 + \kappa_j r_j}, \qquad b_{S\cup\{j\}} = b_S + \lambda_j h_j .
\end{align}
The remaining scalars $r_\ell$ and $u_\ell$ can be updated recursively as well. If $g_\ell = h_\ell^\top \Omega_S h_j$, then
\begin{align}
r_\ell \leftarrow r_\ell - \frac{\kappa_j g_\ell^2}{1 + \kappa_j r_j}, \qquad
u_\ell \leftarrow u_\ell + \frac{\lambda_j - \kappa_j u_j}{1 + \kappa_j r_j}\, g_\ell .
\end{align}
With all $h_i$ retained in memory, each greedy step costs $O(Nd + d^2)$, so one forward pass costs $O\bigl(k(Nd + d^2)\bigr)$ time and $O(Nd + d^2)$ storage. A naive implementation that instead rebuilds and factorizes $I_d + \Sigma_S + \kappa_j h_j h_j^\top$ for every candidate costs $O(d^3)$ per candidate, hence $O(kNd^3)$ time, with the same $O(Nd + d^2)$ storage.

\subsection{Greedy marginal in the equal-quality limit}
\label{app:repexp-limit}

Suppose that every quality score equals a common value $q$, so that $\lambda_i = \lambda$ and $\kappa_i = \kappa$ for every $i$, and that the penalty term is dropped. Then \cref{eq:surrogate} reduces to
\begin{equation}
\log\widetilde L(S) = |S|\log\Phi(-q) - \frac12\log\det(I_d + \Sigma_S), \qquad \Sigma_S = \kappa\sum_{i \in S}h_i h_i^\top . \label{eq:equal-objective}
\end{equation}
Adding a candidate $j \notin S$ changes $\Sigma_S$ by the rank-one matrix $\kappa h_j h_j^\top$. The matrix determinant lemma, $\det(A + uv^\top) = \det(A)\bigl(1 + v^\top A^{-1}u\bigr)$ for an invertible $A$, applied with $A = I_d + \Sigma_S$ and $u = v = \sqrt{\kappa}\, h_j$, gives
\begin{equation}
\det(I_d + \Sigma_{S \cup \{j\}}) = \det(I_d + \Sigma_S)\bigl(1 + \kappa\, h_j^\top(I_d + \Sigma_S)^{-1}h_j\bigr). \label{eq:equal-det}
\end{equation}
Subtracting \cref{eq:equal-objective} at $S$ from \cref{eq:equal-objective} at $S \cup \{j\}$ and using \cref{eq:equal-det}, the factor $\det(I_d + \Sigma_S)$ cancels and
\begin{align}
\Delta_j(S) &= \log\Phi(-q) - \frac12\Bigl(\log\det(I_d + \Sigma_{S \cup \{j\}}) - \log\det(I_d + \Sigma_S)\Bigr) \\
&= \log\Phi(-q) - \frac12\log\bigl(1 + \kappa\, h_j^\top(I_d + \Sigma_S)^{-1}h_j\bigr). \label{eq:equal-marginal}
\end{align}
The first term of \cref{eq:equal-marginal} is the same for every candidate, and $x \mapsto \log(1 + \kappa x)$ is increasing since $\kappa > 0$, so the candidate with the smallest $\Delta_j(S)$ is the one with the largest $h_j^\top(I_d + \Sigma_S)^{-1}h_j$. Moreover, $(I_d + \Sigma_S)^{-1} = \kappa^{-1}\bigl(\kappa^{-1}I_d + \sum_{i \in S}h_i h_i^\top\bigr)^{-1}$, so this ranking is the RepExp rule of \cref{sec:problem} with regularizer $\tau = 1/\kappa$.
\clearpage
\section{Proofs}
\label{app:proofs}

\subsection{Proofs of \texorpdfstring{\cref{thm:surrogate} and \cref{cor:selection}}{Theorem 1 and Corollary 1}}
\label{app:proof-surrogate-corollary}

\subsubsection{Proof of \texorpdfstring{\cref{thm:surrogate}}{Theorem 1}}
\label{app:proof-surrogate}

Here $\Pr[\mathbf{C} = c \mid H]$ is given by \cref{eq:likelihood}. Fix outcomes $c \in \{-1,1\}^N$, and write $\lambda_i$, $\kappa_i$, $\Sigma$ and $b$ for $\lambda_i(c_i)$, $\kappa_i(c_i)$, $\Sigma(c)$ and $b(c)$. For each $i \in [N]$ and a scalar perturbation $x$, let $\ell_i(x) := \log\Phi\bigl(c_i(q_i + x)\bigr)$, and let $r(u) := \phi(u)/\Phi(u)$, which satisfies $r'(u) = -r(u)\bigl(u + r(u)\bigr)$. Since $c_i^2 = 1$ and $\phi$ is even, $r(c_i q_i) = \lambda_i$, and
\begin{align}
\ell_i'(x) = c_i\, r\bigl(c_i(q_i+x)\bigr) \implies \ell_i'(0) = c_i\lambda_i, \qquad
\ell_i''(0) = r'(c_i q_i) = -\lambda_i(\lambda_i + c_i q_i) = -\kappa_i .
\end{align}
To see that $\kappa_i > 0$, let $u = c_i q_i$. For $u \geq 0$ this follows immediately from $\lambda_i > 0$. For $u < 0$,
\begin{equation}
\Phi(u) = \int_{|u|}^{\infty}\phi(t)\,dt < \frac{1}{|u|}\int_{|u|}^{\infty}t\phi(t)\,dt = \frac{\phi(u)}{|u|},
\end{equation}
and hence $\lambda_i > |u| = -c_i q_i$. Taylor expansion around $x = 0$ now gives
\begin{equation}
\log\Phi\bigl(c_i(q_i+x)\bigr) = \log\Phi(c_i q_i) + c_i\lambda_i x - \frac{\kappa_i}{2}x^2 + R_i(x), \qquad R_i(x) = O(|x|^3). \label{eq:taylor}
\end{equation}

We now substitute $x = h_i^\top \mathbf{f}$ and sum \cref{eq:taylor} over $i \in [N]$. Since $\mathbf{f}$ is common for all $i \in [N]$, the only aggregate quantities that arise are $\Sigma$ and $b$, and the term inside the expectation can be written as
\begin{align}
\prod_{i=1}^{N}\Phi\bigl(c_i(q_i + h_i^\top \mathbf{f})\bigr)
= \left(\prod_{i=1}^{N}\Phi(c_i q_i)\right)
\exp\!\left(b^\top \mathbf{f} - \frac12 \mathbf{f}^\top \Sigma \mathbf{f} + \sum_{i=1}^{N}R_i(h_i^\top \mathbf{f})\right).
\end{align}
Dropping the accumulated remainder gives the quadratic approximation
\begin{equation}
\widetilde{\Pr}[\mathbf{C} = c \mid H] := \left(\prod_{i=1}^{N}\Phi(c_i q_i)\right) \E_{\mathbf{f}}\exp\!\left(b^\top \mathbf{f} - \frac12 \mathbf{f}^\top \Sigma \mathbf{f}\right). \label{eq:quadratic-approx}
\end{equation}

The expectation in \cref{eq:quadratic-approx} has a closed form. Combining its exponent with the standard Gaussian density gives
\begin{align}
\E_{\mathbf{f}}\exp\!\left(b^\top \mathbf{f} - \frac12 \mathbf{f}^\top \Sigma \mathbf{f}\right)
= (2\pi)^{-d/2} \int_{\R^d} \exp\!\left(-\frac12 f^\top(I_d + \Sigma)f + b^\top f\right)df.
\end{align}
Because $\kappa_i > 0$, $\Sigma \succeq 0$ and the matrix $I_d + \Sigma$ is positive definite. Completing the square,
\begin{align}
-\frac12 f^\top(I_d+\Sigma)f + b^\top f
={}& -\frac12 \bigl(f - (I_d+\Sigma)^{-1}b\bigr)^\top (I_d+\Sigma) \bigl(f - (I_d+\Sigma)^{-1}b\bigr) \\
&\quad + \frac12 b^\top(I_d+\Sigma)^{-1}b.
\end{align}
After translating the integration variable, the remaining Gaussian integral contributes $\det(I_d + \Sigma)^{-1/2}$, which gives \cref{eq:surrogate-prob}. \hfill$\square$

\subsubsection{Proof of \texorpdfstring{\cref{cor:selection}}{Corollary 1}}
\label{app:proof-corollary}

Integrating out the outcomes of the rollouts outside $S$ in \cref{eq:likelihood} leaves the same expression with the product over $i \in S$, so \cref{thm:surrogate} applies to the rollouts in $S$, with $\Sigma_S(c)$ and $b_S(c)$ the sums over $i \in S$. Take $c_i = -1$ for every $i \in S$. Then $\lambda_i(-1) = \phi(q_i)/\Phi(-q_i)$, $\kappa_i(-1) = \lambda_i(\lambda_i - q_i)$, and $b_S(c) = -\sum_{i \in S}\lambda_i h_i$, and throughout the appendix we write $b_S = \sum_{i \in S}\lambda_i h_i$. Since $b_S$ enters \cref{eq:surrogate-prob} only through $b_S^\top(I_d+\Sigma_S)^{-1}b_S$, the surrogate of $L(S)$ is $\widetilde L(S)$, and since the logarithm is increasing, minimizing $\widetilde L(S)$ is equivalent to minimizing $\log\widetilde L(S)$. \hfill$\square$

\subsection{Proof of \texorpdfstring{\cref{thm:greedy-main}}{Theorem 2}}
\label{app:proof-greedy}

The submodularity of the exact pass probability is a standard property of the probability of a union of events \citep{wilson2018maximizing}, and we prove it here for completeness. The new result is that the surrogate keeps this property.

\begin{definition}[Submodular function]
A set function $G:2^{[N]}\to\R$ is submodular if, for every
$A\subseteq B\subseteq[N]$ and every $j\notin B$,
\begin{equation}
    G(A\cup\{j\})-G(A)
    \geq G(B\cup\{j\})-G(B).
    \label{eq:dr-definition}
\end{equation}
\end{definition}

\begin{definition}[Monotone function]
A set function $G:2^{[N]}\to\R$ is monotone if $G(A)\leq G(B)$ whenever $A\subseteq B\subseteq[N]$.
\end{definition}

\begin{definition}[Normalized function]
A set function $G:2^{[N]}\to\R$ is normalized if $G(\varnothing)=0$.
\end{definition}

\subsubsection{Submodularity of Exact \passk Probability}
\label{app:submodularity-exact}

First, we show that the exact \passk probability is already a normalized monotone submodular set function. To see this, let $\mathcal{C}_i$ denote the event that response $i$ is correct.  The exact pass probability can be written as
\begin{equation}
    P(S)
    :=
    \Pr\left(\bigcup_{i\in S}\mathcal{C}_i\right).
\end{equation}
For any $A\subseteq B\subseteq[N]$ and $j\notin B$, the marginal improvement from adding response $j$ is
\begin{align}
    P(A\cup\{j\})-P(A)
    &=
    \Pr\left(
        \mathcal{C}_j
        \setminus
        \bigcup_{i\in A}\mathcal{C}_i
    \right),                                                   \label{eq:exact-marginal-A}\\
    P(B\cup\{j\})-P(B)
    &=
    \Pr\left(
        \mathcal{C}_j
        \setminus
        \bigcup_{i\in B}\mathcal{C}_i
    \right).                                                   \label{eq:exact-marginal-B}
\end{align}
Since $A\subseteq B$,
\begin{equation}
    \Bigl(\mathcal{C}_j\setminus\bigcup_{i\in B}\mathcal{C}_i\Bigr)
    \subseteq
    \Bigl(\mathcal{C}_j\setminus\bigcup_{i\in A}\mathcal{C}_i\Bigr).
\end{equation}
Taking probabilities therefore gives
\begin{equation}
    P(A\cup\{j\})-P(A)
    \geq
    P(B\cup\{j\})-P(B),
\end{equation}
which is precisely the submodular function property.  Moreover,
$P(\varnothing)=0$, and $P(A)\leq P(B)$ whenever $A\subseteq B$.  Hence the exact pass probability is normalized, monotone, and submodular.  This conclusion holds for an arbitrary joint distribution of the correctness events and does not require any latent-variable model or conditional independence.

\subsubsection{Submodularity of the Approximate Pass Probability}
\label{app:submodularity}

Now, we show that the approximation under the Gaussian latent model itself is a monotone submodular function. Note that, every approximation of a submodular function may not necessarily be submodular. Before, we proceed, let's try to get an intuition on the fact that $\widetilde L(S)$ indeed represents some notion of a probability. Conditional on $\mathbf f$, each term of the product in $\widetilde L(S)$ can be written as
\begin{align}
    g_i(\mathbf f) := \Phi(-q_i) \exp\left(-\lambda_i h_i^\top \mathbf f - \frac{\kappa_i}{2} (h_i^\top \mathbf f)^2\right).
\end{align}
 Consequently,
 \begin{align}
     \widetilde L(S) = \E_{\mathbf f \sim \mathcal{N}(0, I_d)}\prod_{i\in S} g_i(\mathbf f).
     \label{eq:factorized-surrogate}
 \end{align}
Hence, for $\widetilde L(S)$ to be a definition of some kind of failure probability of selection of correct elements conditioned on $\mathbf f$, it would suffice to prove that $g_i(\mathbf f) < 1$. Towards this, for a scalar $q$, write
\begin{equation}
    \lambda(q):=\frac{\phi(q)}{\Phi(-q)},
    \qquad
    \kappa(q):=\lambda(q)(\lambda(q)-q),
\end{equation}
and define
\begin{equation}
    \beta(q):=\Phi(-q)
       \exp\!\left(\frac{\lambda(q)^2}{2\kappa(q)}\right)
     =\Phi(-q)
       \exp\!\left(\frac{\lambda(q)}{2(\lambda(q)-q)}\right).
\end{equation}

\begin{lemma}
For every finite $q\in\R$, one has $0<\beta(q)<1$. Moreover, $\beta(q)$ is strictly decreasing in $q$.
\label{lem:beta}
\end{lemma}
\begin{proof}
Let $\overline\Phi(q):= \Phi(-q) = 1 - \Phi(q)$ and put
\begin{equation}
    v(q):=\lambda(q)-q>0,
    \qquad
    r(q):=\log\beta(q)
      =\log\overline\Phi(q)+\frac{\lambda(q)}{2v(q)}.
\end{equation}
Since $\overline\Phi'(q)=-\phi(q)$, we have
\begin{equation}
    \frac{d}{dq}\log\overline\Phi(q)=-\lambda(q).
\end{equation}
Also, direct differentiation gives
\begin{equation}
    \lambda'(q)=\lambda(q)(\lambda(q)-q)=\lambda(q)v(q),
    \qquad
    v'(q)=\lambda(q)v(q)-1.
\end{equation}
Consequently,
\begin{align}
    r'(q)
      &=-\lambda(q)
        +\frac{\lambda'(q)v(q)-\lambda(q)v'(q)}{2v(q)^2} =\frac{\lambda(q)}{2v(q)^2}
        \left(1-qv(q)-2v(q)^2\right).
    \label{eq:r-prime}
\end{align}

It remains to show that the expression in parentheses is negative. For this, let $X\sim N(0,1)$. Using the fact that, $\E[X\mid X>q]=\lambda(q)$, and $\E[X^2\mid X>q]=1+q\lambda(q)$,
we get
\begin{equation}
    \E\left[X-q\ \big|\ X>q\right] = v(q),
    \qquad
    \E\left[(X-q)^2\ \big|\ X>q\right] =1-qv(q).
    \label{eq:residual-moments}
\end{equation}
Hence, it would suffice to prove that \begin{equation}
\frac{1}{2}\E\left[(X-q)^2\ \big|\ X>q\right] \leq \left(\E\left[X-q\ \big|\ X>q\right]\right)^2.
\end{equation}
To prove the above expression, consider
\begin{equation}
    T(x):=\Pr(X -q >x | X > q)=\frac{\overline\Phi(q+x)}{\overline\Phi(q)},
    \qquad x\geq0.
\end{equation}
Under the conditional law given $X>q$, the random variable $X-q$
is nonnegative. Hence, by the tail-integral formula for second
moments \cite[Lemma~2.2.13]{durrett2019probability},
\begin{align}
    \frac12 \mathbb{E}\bigl[(X-q)^2 \mid X>q\bigr]
    &=
    \mathbb{E}\left[
        \left.
        \int_0^\infty t\,\mathbf{1}_{\{X-q>t\}}\,dt
        \,\right|\, X>q
    \right] \\
    &=
    \int_0^\infty t\,\Pr(X-q>t \mid X>q)\,dt =
    \int_0^\infty tT(t)\,dt.
\end{align}
The interchange of expectation and integration is justified by
Tonelli's theorem \cite[Theorem~1.7.2]{durrett2019probability}, since the integrand is nonnegative. Moreover, substituting $t=x+y$ in the inner integral and then exchanging the order of integration by Tonelli's theorem again, we obtain
\begin{align}
    \int_0^\infty \int_0^\infty T(x+y)\,dx\,dy
    &=
    \int_0^\infty \int_y^\infty T(t)\,dt\,dy =
    \int_0^\infty \int_0^t T(t)\,dy\,dt = \int_0^\infty tT(t)\,dt.
\end{align}
Here the integration region is
$\{(y,t):0\leq y\leq t<\infty\}$, and nonnegativity of $T$
justifies the exchange.
We know that $\lambda$ is strictly increasing because
$\lambda'(q)=\kappa(q)>0$.  Hence, for $x,y>0$,
\begin{align}
    \frac{T(x+y)}{T(x)}
       &=\exp\!\left(-\int_x^{x+y}\lambda(q+t)\,dt\right) <\exp\!\left(-\int_0^y\lambda(q+t)\,dt\right)
        =T(y),
\end{align}
and therefore $T(x+y)<T(x)T(y)$.  Thus, we get that
\begin{align}
    \frac12\E\left[(X-q)^2\ \big|\ X>q\right]
      &=\int_0^\infty tT(t)\,dt
        =\int_0^\infty\!\int_0^\infty T(x+y)\,dx\,dy \\
        &< \left(\int_0^\infty T(x)\,dx\right)^2
        =\left(\E\left[X-q\ \big|\ X>q\right]\right)^2.
    \label{eq:residual-inequality}
\end{align}
Combining \cref{eq:residual-moments} and \cref{eq:residual-inequality} yields
\begin{equation}
    1-qv(q)-2v(q)^2<0.
\end{equation}
By \cref{eq:r-prime}, $r'(q)<0$.  Finally, as $q\to-\infty$,
$\overline\Phi(q)\to1$, $\lambda(q)\to0$, and
$\lambda(q)/(2v(q))\to0$.  Thus $r(q)\to0$.  Since $r$ is strictly decreasing,
$r(q)<0$ for every finite $q$.  Equivalently, $0<\beta(q)<1$, and $\beta$ is strictly decreasing.
\end{proof}

Thus, coming back to $g_i(\mathbf f)$, completing the square in the scalar $x=h_i^\top \mathbf f$ gives
\begin{align}
    g_i(\mathbf f)
      &=\Phi(-q_i)
        \exp\!\left(
          -\frac{\kappa_i}{2}
             \left(x+\frac{\lambda_i}{\kappa_i}\right)^2
          +\frac{\lambda_i^2}{2\kappa_i}
        \right) \leq \beta(q_i)<1
\end{align}
by Lemma~\ref{lem:beta}.  Thus $0<g_i(\mathbf f)<1$ for every $i$ and every $\mathbf f$. Hence, one can define the
approximate \passk probability by
\begin{equation}
    \widetilde P(S):=1-\widetilde L(S).
    \label{eq:approx-pass}
\end{equation}

It follows immediately from \cref{eq:factorized-surrogate} that $\widetilde L$ decreases under set
inclusion, and hence $\widetilde P=1-\widetilde L$ is monotone.  Also,
$\widetilde P(\varnothing)=1-1=0$, so $\widetilde P$ is normalized.

Now let $A\subseteq B$ and $j\notin B$.  The marginal improvement obtained by adding $j$ is
\begin{align}
    \widetilde P(A\cup\{j\})-\widetilde P(A)
       &=\widetilde L(A)-\widetilde L(A\cup\{j\}) = \E_{\mathbf f}\left[
          (1-g_j(\mathbf f))\prod_{i\in A}g_i(\mathbf f)
       \right].
\end{align}
Because every $g_i(\mathbf f)$ lies in $(0,1)$,
\begin{equation}
    (1-g_j(\mathbf f))\prod_{i\in A}g_i(\mathbf f)
    \geq
    (1-g_j(\mathbf f))\prod_{i\in B}g_i(\mathbf f)
\end{equation}
pointwise in $\mathbf f$. Taking expectations proves that $\widetilde P$ is submodular.

\subsubsection{Greedy Guarantee}
\label{app:greedy-guarantee}

Since $\widetilde P$ is normalized, monotone, and submodular (\cref{app:submodularity}), by \citet{nemhauser1978analysis}, the greedy solution satisfies the guarantee of \cref{thm:greedy-main}. \hfill$\square$

\clearpage
\subsection{Proof of \texorpdfstring{\cref{prop:nphard}}{Proposition 1}}
\label{app:nphard}

\begin{proof}
We show that the problem is NP-hard even when all $q_i$ equal the same integer and all $h_i$ have entries in $\{0, 1\}$. We reduce from the independent set problem, which asks whether a graph contains $k$ pairwise non-adjacent vertices and is NP-complete \citep{karp1972reducibility}. Let $G$ be a graph on the vertex set $[N]$ with edge set $E$, adjacency matrix $A$ and maximum degree $\Delta \geq 1$ (the case without edges is trivial). For $S \subseteq [N]$, let $A_S$ be the adjacency matrix of the subgraph induced by $S$ and $e(S)$ its number of edges. We construct an instance whose minimizers are independent sets of size $k$ whenever such a set exists.

Let $d = |E| + N\Delta$. The first $|E|$ coordinates of $\R^d$ are indexed by the edges of $G$, and the remaining ones are split into $N$ blocks of $\Delta$ coordinates, one block per vertex. For each vertex $i$, let $h_i \in \{0,1\}^d$ have a one in the coordinate of every edge incident to $i$ and a one in $\Delta - \deg(i)$ coordinates of its own block. Two distinct vertices share a nonzero coordinate only if they are adjacent, and then they share exactly one, so $\|h_i\|_2^2 = \Delta$ and $h_i^\top h_j = A_{ij}$ for $i \neq j$, and the Gram matrix of $\{h_i\}_{i \in S}$ is $\Delta I_k + A_S$. We further set $q_i = q$ for all $i$, for an integer $q \leq -1$ fixed below, so that $\lambda_i = \lambda$ and $\kappa_i = \kappa$. Then $\Sigma_S = \kappa\sum_{i \in S} h_i h_i^\top$ and $b_S = \lambda\sum_{i \in S} h_i$, and \cref{eq:surrogate} reads
\begin{equation}
\log\widetilde L(S) = k\log\Phi(-q) - \frac12 D(S) + R(S),
\end{equation}
where $D(S) = \log\det(I_d + \Sigma_S)$ and $R(S) = \frac12 b_S^\top(I_d + \Sigma_S)^{-1} b_S \geq 0$. Since $q \leq -1$, we have $\Phi(-q) \geq \frac12$, hence $\lambda \leq 2\phi(q) < \frac12$ and $\lambda|q| \leq 2|q|\phi(q) \leq 2\phi(1) < \frac12$, where we used that $t\phi(t)$ is decreasing for $t \geq 1$. Consequently,
\begin{equation}
\lambda|q| \leq \kappa = \lambda^2 + \lambda|q| < 1. \label{eq:kappa-bounds}
\end{equation}

We first show that $D$ is maximized exactly by the independent sets, with an explicit margin. By Sylvester's determinant identity and the form of the Gram matrix,
\begin{equation}
D(S) = \log\det\bigl((1 + \kappa\Delta) I_k + \kappa A_S\bigr) = k\log(1 + \kappa\Delta) + \sum_{j=1}^k \log(1 + x_j), \qquad x_j = \frac{\kappa\mu_j}{1 + \kappa\Delta},
\end{equation}
where $\mu_1, \ldots, \mu_k$ are the eigenvalues of $A_S$. They satisfy $\sum_j \mu_j = \tr A_S = 0$, $\sum_j \mu_j^2 = \tr A_S^2 = 2e(S)$ and $|\mu_j| \leq \Delta$, hence $\sum_j x_j = 0$ and $|x_j| < 1$. If $S$ is independent, then $A_S = 0$ and $D(S) = k\log(1 + \kappa\Delta)$. Otherwise $e(S) \geq 1$, and the inequality $\log(1 + x) \leq x - x^2/8$ for $|x| < 1$, together with \cref{eq:kappa-bounds}, gives
\begin{equation}
D(S) \leq k\log(1 + \kappa\Delta) - \frac18\sum_{j=1}^k x_j^2 = k\log(1 + \kappa\Delta) - \frac{\kappa^2 e(S)}{4(1 + \kappa\Delta)^2} \leq k\log(1 + \kappa\Delta) - \frac{\kappa^2}{4(1 + \Delta)^2}.
\end{equation}

We next bound the penalty term uniformly over $S$. Since $(I_d + \Sigma_S)^{-1} \preceq I_d$ and $\bigl\|\sum_{i \in S} h_i\bigr\|_2^2 = k\Delta + 2e(S) \leq k(\Delta + k)$, we have $0 \leq R(S) \leq \frac12\|b_S\|_2^2 \leq \frac12\lambda^2 k(\Delta + k)$.

Now suppose that $G$ contains an independent set $S^\circ$ of size $k$, and let $S$ be a set of size $k$ with $e(S) \geq 1$. Combining the two bounds with \cref{eq:kappa-bounds},
\begin{align}
\log\widetilde L(S) - \log\widetilde L(S^\circ) &\geq \frac12\bigl(D(S^\circ) - D(S)\bigr) - R(S^\circ) \geq \frac{\kappa^2}{8(1 + \Delta)^2} - \frac{\lambda^2}{2}\, k(\Delta + k) \\
&\geq \lambda^2\Bigl(\frac{q^2}{8(1 + \Delta)^2} - \frac{k(\Delta + k)}{2}\Bigr),
\end{align}
and the right-hand side is positive for $q = -2(1 + \Delta)(\Delta + k)$. With this choice, every minimizer of $\log\widetilde L$ is an independent set whenever $G$ contains one of size $k$, so $G$ contains an independent set of size $k$ if and only if a minimizer of $\log\widetilde L$ is independent. The instance consists of $N$ binary vectors of dimension $d = |E| + N\Delta$ and a single integer $q$ of polynomial magnitude, so it is constructed in polynomial time, which proves the claim.
\end{proof}

\clearpage
\subsection{Proof of \texorpdfstring{\cref{prop:weak}}{Proposition 2}}
\label{app:weak}

This section proves \cref{prop:weak}, which describes the failure probability and its surrogate when the correlations between rollouts are weak. To weaken the correlations while keeping the quality scores fixed, we replace $\mathbf f$ by $\rho\,\mathbf f$ for $\rho > 0$, so that
\begin{equation}
 \mathbf Z_i = q_i + \rho\,h_i^\top\mathbf f + \mathbf m_i, \qquad \operatorname{Cov}(\mathbf Z_i,\mathbf Z_j\mid H) = \rho^2 h_i^\top h_j \quad (i \neq j),
\end{equation}
and $\rho = 1$ recovers the model of \cref{sec:prob-model}. The quality scores $q_i = Q_\theta(h_i)$ and the representations $h_i$ are held fixed. Since only the products $\rho\,h_i$ enter the model, this is equivalent to scaling the representations by $\rho$ with the $q_i$ fixed, as in assumption (A2) of \cref{s2:section}. The failure probability of a set $S$ becomes
\begin{equation}
 L(S) = \E_{\mathbf f}\prod_{i\in S}\Phi\bigl(-q_i - \rho\,h_i^\top\mathbf f\bigr),
\end{equation}
and its surrogate is \cref{eq:surrogate} with every $h_i$ replaced by $\rho\,h_i$.

\begin{proof}[Proof of \cref{prop:weak}]
Let $\varepsilon_i = \rho\,h_i^\top\mathbf f$ for $i \in S$, so that $\varepsilon_S \sim \mathcal N(0, K)$ with $K_{ij} = \rho^2 h_i^\top h_j$, and let $u_i(x) = \log\Phi(-q_i - x)$. Then $u_i'(0) = -\lambda_i$ and $u_i''(0) = -\kappa_i$, and every derivative of $\log\Phi$ of order two or higher is bounded on $\R$. By Taylor's theorem, for all $x \in \R$,
\begin{equation}
 u_i(x) - u_i(0) = -\lambda_i x - \frac12\kappa_i x^2 + \frac16 u_i'''(0)\, x^3 + r_i(x), \qquad |r_i(x)| \leq C x^4,
\end{equation}
where $C$ does not depend on $\rho$. Let $X = \sum_{i\in S}\bigl(u_i(\varepsilon_i) - u_i(0)\bigr)$, so that $L(S) = \prod_{i\in S}\Phi(-q_i)\,\E\,e^X$. Since $\Phi \leq 1$, we have $e^X \leq M_S := \prod_{i\in S}\Phi(-q_i)^{-1}$, and Taylor's theorem for the exponential gives $\bigl|e^X - 1 - X - \frac12 X^2 - \frac16 X^3\bigr| \leq \frac{M_S}{24}X^4$. Moreover, $|X| \leq C'\bigl(\|\varepsilon_S\| + \|\varepsilon_S\|^2\bigr)$ for a constant $C'$, so $\E X^4 = O(\rho^4)$ and
\begin{equation}
 \E\,e^X = 1 + \E X + \frac12\E X^2 + \frac16\E X^3 + O(\rho^4).
\end{equation}
The odd moments of $\varepsilon_S$ vanish, and its moments of order four are $O(\rho^4)$. Hence
\begin{align}
 \E X &= -\frac12\sum_{i\in S}\kappa_i K_{ii} + O(\rho^4), &
 \E X^2 &= \sum_{i,j\in S}\lambda_i\lambda_j K_{ij} + O(\rho^4), &
 \E X^3 &= O(\rho^4).
\end{align}
With $a := -\frac12\sum_{i\in S}\kappa_i K_{ii} + \frac12\sum_{i,j\in S}\lambda_i\lambda_j K_{ij} = O(\rho^2)$, this gives $\E\,e^X = 1 + a + O(\rho^4)$, and therefore $\log\E\,e^X = a + O(\rho^4)$. Splitting the double sum into its diagonal and off-diagonal parts and using $\lambda_i^2 - \kappa_i = \lambda_i q_i$ yields
\begin{equation}
 a = \frac{\rho^2}{2}\sum_{i\in S}\lambda_i q_i\|h_i\|^2 + \rho^2\sum_{i<j\in S}\lambda_i\lambda_j\,h_i^\top h_j,
\end{equation}
which proves the expansion of $\log L(S)$.

For the surrogate, \cref{eq:surrogate} with $\rho\,h_i$ in place of $h_i$ reads
\begin{equation}
 \log\widetilde L(S) = \sum_{i\in S}\log\Phi(-q_i) - \frac12\log\det(I_d + \rho^2\Sigma_S) + \frac{\rho^2}{2}\, b_S^\top(I_d + \rho^2\Sigma_S)^{-1} b_S .
\end{equation}
Since $\log\det(I_d + \rho^2\Sigma_S) = \rho^2\tr\Sigma_S + O(\rho^4)$ and $(I_d + \rho^2\Sigma_S)^{-1} = I_d + O(\rho^2)$,
\begin{equation}
 \log\widetilde L(S) = \sum_{i\in S}\log\Phi(-q_i) - \frac{\rho^2}{2}\sum_{i\in S}\kappa_i\|h_i\|^2 + \frac{\rho^2}{2}\Bigl\|\sum_{i\in S}\lambda_i h_i\Bigr\|^2 + O(\rho^4),
\end{equation}
and expanding the square gives the same expansion as for $\log L(S)$.
\end{proof}

\clearpage
\section{Selection Behavior Across Budgets}
\label{s2:section}

\Cref{prop:weak} gives intuition for how PKTP moves from Linear Probe to RepExp as $k$ grows when the correlations are weak. This section makes the two ends of this behavior precise for a fixed correlation scale $\rho$. PKTP starts from the empty set and repeatedly adds the candidate $j$ with the smallest $\log\widetilde L(S\cup\{j\})-\log\widetilde L(S)$, and we call this the \emph{surrogate greedy rule}. At one end, the \emph{probe-style rule} keeps the $k$ candidates with the largest $q_i$, as Linear Probe does. At the other end, RepExp is implicitly a greedy rule as well. After a random first pick, RepExp with regularizer $\tau$ adds at each step the candidate that most increases $\log\det(I_d+\tau^{-1}G_S)$, where $G_S:=\sum_{i\in S}h_ih_i^\top$ is the Gram matrix of the kept representations (\cref{app:repexp-limit}).

We keep the notation of \cref{cor:selection}, with $\lambda(q):=\phi(q)/\Phi(-q)$, $\kappa(q):=\lambda(q)(\lambda(q)-q)$, $\lambda_i=\lambda(q_i)$ and $\kappa_i=\kappa(q_i)$. Then $\Sigma_S=\sum_{i\in S}\kappa_ih_ih_i^\top$, $b_S=\sum_{i\in S}\lambda_ih_i$, and, as in \cref{eq:surrogate},
\begin{equation}
 \log\widetilde L(S)=\sum_{i\in S}\log\Phi(-q_i)-\frac12\log\det(I_d+\Sigma_S)+\frac12b_S^\top(I_d+\Sigma_S)^{-1}b_S.
 \label{s2:objective}
\end{equation}
We now give an overview of the results of \cref{s2:section} and the assumptions they rely on.

\subsection{Overview of \texorpdfstring{\cref{s2:section}}{Appendix C}}
\label{s2:overview}

\Cref{s2:assumptions} states the three assumptions that the results rely on. The first is that most candidates, the bulk, have qualities drawn independently from a fixed bounded distribution, while a few candidates, the spikes, have fixed qualities above every bulk quality. In other words, the pool contains a few spikes whose quality is higher than that of any other candidate, while the bulk consists of lower-quality candidates drawn from a common distribution. The second is that the representations are $h_i=\rho v_i$ with $\|v_i\|_2\leq1$ for a fixed scale $\rho$, and do not depend on the qualities. Thus $\rho$ caps how strongly two candidates can be correlated, since $\operatorname{Cov}(\mathbf{Z}_i,\mathbf{Z}_j\mid H)=h_i^\top h_j\leq\rho^2$, as in \cref{prop:weak}, and the similarity between candidates carries no information about their qualities. The third is needed only for large budgets. There we minimize the objective over a family of size-$k$ sets, and we require these sets to be diverse, meaning that their representations point in all directions rather than being concentrated in a few.

\Cref{s2:small-budget} shows that at small budgets the two rules agree when the gap between the $k$th and $(k+1)$th largest qualities, measured as $\log\Phi(-q_{(k+1)})-\log\Phi(-q_{(k)})$, exceeds a constant multiple of $k\rho^2$, which bounds the correlation terms that $k$ picks can accumulate. The surrogate greedy rule then returns exactly the top-$k$ set of the probe-style rule, and if this condition holds at every budget up to $k$, it also adds these candidates in the same order.

\Cref{s2:large-budget} shows that at large budgets PKTP and RepExp optimize nearly the same objective. To see this, we average the objective of PKTP over the bulk qualities, for the diverse sets of the third assumption. Apart from a constant and an error that shrinks as $k$ grows, minimizing this average is the same as keeping the spikes and maximizing the RepExp objective. This is a statement about the objectives, and it does not mean that PKTP and RepExp pick the same set.

\Cref{s2:compat} constructs an explicit model in which both results hold. It fixes one quality distribution and one correlation scale $\rho$. In this model, the surrogate greedy rule follows the probe-style rule for every $k \leq 3$, while the large-budget result holds as $k$ grows.

\subsection{Assumptions}
\label{s2:assumptions}

We use $k$ for the selection budget, with $1\leq k\leq N$ for a fixed candidate pool. The three assumptions below concern the quality scores and the representations of the candidates.
\begin{enumerate}[label=\textnormal{(A\arabic*)},leftmargin=*,itemsep=5pt]
 \item \textbf{Bulk and spikes.}
 There is a fixed set $O=\{o_1,\ldots,o_r\}$ of $r\geq1$ candidates, the spikes. The qualities of all other candidates, the bulk, are drawn independently from a distribution $\mathcal D_q$ supported in a fixed interval $[w_0,w]\subseteq[-M,M]$ for a finite $M$. The spike qualities are fixed, lie in $[-M,M]$, and satisfy
 \begin{equation}
   q_{o_1}>q_{o_2}>\cdots>q_{o_r}>w,
 \end{equation}
 so every spike has a higher quality than every bulk candidate.

 \item \textbf{Correlation scale.}
 The representations are $h_i=\rho v_i$ with $\|v_i\|_2\leq1$ and a fixed $0<\rho\leq1$. They are deterministic, so they do not depend on the bulk qualities. A small $\rho$ means weak correlations between candidates, as in \cref{prop:weak}. The small-budget result needs $\rho$ to be small compared to a quality gap, while the large-budget results allow any fixed $\rho>0$.

 \item \textbf{Diverse representations.}
 For large budgets, we let $k\to\infty$ along a sequence of candidate pools of sizes $N_k\geq k$, while the dimension $d$, the spikes and their qualities, the scale $\rho$, and the distribution $\mathcal D_q$ stay fixed. The large-budget results minimize the objective over a family $\mathcal A_k$ of size-$k$ subsets of $[N_k]$ with the following properties. It is deterministic and chosen independently of the bulk qualities. It is nonempty for every sufficiently large $k$, and every $S\in\mathcal A_k$ satisfies
 \begin{equation}
  c_0kI_d\preceq\sum_{i\in S}v_iv_i^\top\preceq C_0kI_d,
  \qquad
  \left\|\sum_{i\in S \setminus O}v_i\right\|_2\leq B_0,
  \label{s2:geometry}
 \end{equation}
 where $c_0,C_0,B_0>0$ do not depend on $k$ or $S$. The first condition says that the representations in $S$ cover every direction, with a total weight that grows linearly in $k$, and the second says that the representations of the bulk candidates nearly cancel out. Different sets in $\mathcal A_k$ may contain different spikes.
\end{enumerate}

\subsection{Small budget \texorpdfstring{$k$}{k}: a quality gap yields the probe-style rule}
\label{s2:small-budget}

This subsection shows that when the gap between the $k$th and $(k+1)$th largest qualities is large compared to $k\rho^2$, PKTP keeps exactly the top-$k$ candidates, as the probe-style rule does. We write
\begin{equation}
  \Lambda:=\max_{q\in[-M,M]}\lambda(q)<\infty
\end{equation}
for the largest value of $\lambda$ over the possible qualities, which bounds every $\lambda_i$. Order the realized qualities as $q_{(1)}\geq\cdots\geq q_{(N)}$, and let $S_k^{\mathrm{top}}$ be the set of the $k$ candidates with the largest qualities.  For $1\leq k<N$, define
\begin{equation}
 \Gamma_k:=\log\Phi(-q_{(k+1)})-\log\Phi(-q_{(k)}).
 \label{s2:quality-gap}
\end{equation}
Since $\frac{d}{dq}\log\Phi(-q)=-\lambda(q)<0$, we have $\Gamma_k>0$ exactly when $q_{(k)}>q_{(k+1)}$.
This gap measures the quality loss from replacing the weakest member of $S_k^{\mathrm{top}}$ by the strongest candidate outside it. The next lemma says that adding a candidate $j$ to a set of $t$ candidates changes the objective by its own failure term $\log\Phi(-q_j)$, up to an error that grows with $t$ and vanishes as $\rho\to0$.

\begin{lemma}
\label{s2:marginal}
Under (A1) and (A2),
if $|S|=t$ and $j\notin S$, then
\begin{equation}
 \left|\log\widetilde L(S\cup\{j\})-\log\widetilde L(S)-\log\Phi(-q_j)\right|\leq\frac{\rho^2}{2}+\Lambda^2\left(t+\frac12\right)\rho^2 + \frac{\Lambda^2t^2\rho^4}{2}.
 \label{s2:marginal-error}
\end{equation}
\end{lemma}
\begin{proof}
We first record $0<\kappa(q)<1$. Positivity is shown in \cref{app:proof-surrogate}. For the upper bound, let $Z\sim N(0,1)$ and recall that
\begin{equation}
 \E[Z\mid Z>q]=\lambda(q),\qquad
 \E[Z^2\mid Z>q]=1+q\lambda(q).
\end{equation}
Hence, the conditional variance is
\begin{equation}
 0<\operatorname{Var}(Z\mid Z>q)
   =1+q\lambda(q)-\lambda(q)^2=1-\kappa(q).
\end{equation}

Let $A=I_d+\Sigma_S$, $a_j=h_j^\top A^{-1}h_j$, and
$u_j=h_j^\top A^{-1}b_S$.  The matrix determinant lemma and the Sherman--Morrison formula give
\begin{align}
 \det(A+\kappa_jh_jh_j^\top)&=\det(A)(1+\kappa_ja_j),\\
 (A+\kappa_jh_jh_j^\top)^{-1}&=A^{-1}-\frac{\kappa_jA^{-1}h_jh_j^\top A^{-1}}{1+\kappa_ja_j}.
\end{align}
Expanding the quadratic term with $b_{S\cup\{j\}}=b_S+\lambda_jh_j$ consequently yields
\begin{align}
 \log\widetilde L(S\cup\{j\})-\log\widetilde L(S)
 ={}&\log\Phi(-q_j)-\frac12\log(1+\kappa_ja_j) +\frac{\lambda_ju_j+\lambda_j^2a_j/2-\kappa_ju_j^2/2}
               {1+\kappa_ja_j}.
 \label{s2:exact-marginal}
\end{align}
By (A2), $\|h_j\|_2\leq\rho$, and by (A1), $|q_i|\leq M$, so $\lambda_i\leq\Lambda$. Since $A\succeq I_d$, we have $0\leq a_j\leq\rho^2$ and
\begin{equation}
 |u_j|\leq\|h_j\|_2\|A^{-1}\|_{\op}\|b_S\|_2
          \leq t\Lambda\rho^2.
\end{equation}
Using $\kappa_j<1$, and $1+\kappa_ja_j\geq1$ in
\cref{s2:exact-marginal}, the absolute correction to $\log\Phi(-q_j)$ is at most
\begin{equation}
 \frac{\rho^2}{2}+\Lambda(t\Lambda\rho^2)
 +\frac{\Lambda^2\rho^2}{2}+\frac{t^2\Lambda^2\rho^4}{2},
\end{equation}
which is the bound in \cref{s2:marginal-error}.
\end{proof}

\begin{theorem}
\label{s2:probe}
Under (A1) and (A2), if $k\rho^2\leq1$ and $\Gamma_k>(1+3\Lambda^2)k\rho^2$,
then $S_k^{\mathrm{top}}$ is the unique size-$k$ minimizer of $\log\widetilde L$, and the surrogate greedy rule returns $S_k^{\mathrm{top}}$
after $k$ steps.
\end{theorem}
\begin{proof}
First, note that if $t\leq k-1$ and $k\rho^2\leq1$, then twice the bound of \cref{s2:marginal} satisfies
\begin{equation}
 \rho^2+\Lambda^2(2t+1)\rho^2+\Lambda^2t^2\rho^4
 \leq k\rho^2+2\Lambda^2k\rho^2+\Lambda^2k^2\rho^4
 \leq(1+3\Lambda^2)k\rho^2 < \Gamma_k.
\end{equation}
For any $i\in S_k^{\mathrm{top}}$, $j\notin S_k^{\mathrm{top}}$, and common current set $S$ excluding $i,j$, Lemma~\ref{s2:marginal}
gives, when $|S|=t\leq k-1$,
\begin{align}
 \log\widetilde L(S\cup\{i\})-\log\widetilde L(S\cup\{j\})
 &\leq \log\Phi(-q_i)-\log\Phi(-q_j)+(1+3\Lambda^2)k\rho^2 \\
 &\leq-\Gamma_k+(1+3\Lambda^2)k\rho^2<0.
 \label{s2:exchange}
\end{align}
If the current set of the surrogate greedy rule is contained in $S_k^{\mathrm{top}}$, at least one member of $S_k^{\mathrm{top}}$ is still
available until the $k$th selection.  By \cref{s2:exchange}, every outside
candidate is worse than an available inside candidate.  Starting with the empty set and
inducting therefore proves that the first $k$ selections form $S_k^{\mathrm{top}}$.

For global optimality, take a size-$k$ set $R\ne S_k^{\mathrm{top}}$.  There exist
$j\in R\setminus S_k^{\mathrm{top}}$ and $i\in S_k^{\mathrm{top}}\setminus R$.  Apply \cref{s2:exchange} with
$S=R\setminus\{j\}$.  Replacing $j$ by $i$ strictly decreases $\log\widetilde L$, so $R$ cannot be an
optimizer.  This proves uniqueness of the optimizing set, even if some qualities within $S_k^{\mathrm{top}}$
are tied.
\end{proof}

Applying \cref{s2:probe} at every budget up to $k$ shows that the surrogate greedy rule also adds these candidates in the probe-style order, as long as every gap up to position $k$ is large enough. The spikes give such gaps, which leads to the following corollary.

\begin{corollary}
\label{s2:spike-probe}
Under (A1) and (A2), suppose $1\leq r<N$ and define
\begin{equation}
 \Gamma_*:=\min\left(
  \{\log\Phi(-q_{o_{j+1}})-\log\Phi(-q_{o_j}):1\leq j<r\}
  \cup\{\log\Phi(-w)-\log\Phi(-q_{o_r})\}\right)>0.
 \label{s2:spike-gap}
\end{equation}
If $r\rho^2\leq1$ and
$(1+3\Lambda^2)r\rho^2<\Gamma_*$, then, for every realization of the bulk qualities,
the first $r$ selections of the surrogate greedy rule are $o_1,\ldots,o_r$ in the order selected by the probe-style rule.
\end{corollary}
\begin{proof}
All bulk qualities are at most $w$.  Thus the top $r$ candidates are the ordered spikes, and
each of the gaps $\Gamma_1,\ldots,\Gamma_r$ is at least $\Gamma_*$. Apply
Theorem~\ref{s2:probe} separately at each budget $k\leq r$.  The nested top-$k$ sets identify the full
greedy order.
\end{proof}

The spikes are what keep this result meaningful in large pools. With bulk qualities alone, the gap between the $k$th and $(k+1)$th largest qualities typically shrinks as the pool grows, so the condition of \cref{s2:probe} eventually fails. The gaps around the spikes do not depend on the bulk, so \cref{s2:spike-probe} holds for every pool size, while the bulk qualities keep a fixed variance.

\subsection{Large budget \texorpdfstring{$k$}{k}: the expected objective reduces to the RepExp objective}
\label{s2:large-budget}

We now keep $\rho>0$ fixed and let $k$ grow. Since $G_S=\rho^2\sum_{i\in S}v_iv_i^\top$, (A3) gives $G_S\succeq c_0\rho^2kI_d$. The Gram matrix of a diverse set therefore still grows linearly in $k$, and a small $\rho$ only slows this growth by the factor $\rho^2$. The constants below depend on $\rho$ and on the constants in (A1)--(A3), so the results of this section hold for each fixed $\rho$ as $k\to\infty$, but not necessarily when $\rho$ shrinks as $k$ grows. Define the constants
\begin{equation}
 \overline c:=\E_{q\sim\mathcal D_q} \log \Phi(-q),\qquad
 \overline\lambda:=\E_{q\sim\mathcal D_q}\lambda(q),\qquad
 \overline\kappa:=\E_{q\sim\mathcal D_q}\kappa(q),\qquad
 v_\lambda:=\operatorname{Var}_{q\sim\mathcal D_q}(\lambda(q)).
 \label{s2:bulk-constants}
\end{equation}
These expectations are finite by bounded support, and $\overline\kappa>0$ since $\kappa(q) > 0$.

\begin{theorem}
\label{s2:expected-log}
Under (A1)--(A3), uniformly over $S\in\mathcal A_k$,
\begin{equation}
 \E \log\widetilde L(S)
 =k\overline c
   +\sum_{o\in S\cap O}\bigl(\log\Phi(-q_o)-\overline c\bigr)
   +\frac{d\,v_\lambda}{2\overline\kappa}-\frac12\log\det(I_d+\overline\kappa G_S)
   +O(k^{-1/2}).
 \label{s2:log-limit}
\end{equation}
The expectation is over the bulk qualities, for the fixed set $S$.
\end{theorem}

In words, after averaging over the bulk qualities, the objective of a diverse set depends on $S$ only through the spikes it keeps and through $\log\det(I_d+\overline\kappa G_S)$, while all other terms are constants. The theorem concerns this average, not $\log\widetilde L(S)$ for a single draw of the bulk qualities.

\begin{proof}
Fix $S\in\mathcal A_k$, and write $J=S\cap O$ and $B=S\setminus O$.
The expectation of the scalar quality term is exactly
\begin{equation}
 \E\sum_{i\in S}\log\Phi(-q_i)
 =k\overline c+\sum_{o\in J}(\log\Phi(-q_o)-\overline c).
 \label{s2:quality-term}
\end{equation}
We treat the determinant and quadratic terms separately. Define
\begin{align}
 G_B&=\sum_{i\in B}h_ih_i^\top,&
 \Sigma_J&=\sum_{o\in J}\kappa(q_o)h_oh_o^\top,&
 b_J&=\sum_{o\in J}\lambda(q_o)h_o,\\
 A_0&=I_d+\overline\kappa G_B+\Sigma_J,&
 \Delta&=\sum_{i\in B}(\kappa(q_i)-\overline\kappa)h_ih_i^\top,&
 A&=I_d+\Sigma_S=A_0+\Delta.
\end{align}
There are at most $r$ spike contributions.  Removing them from $G_S$ changes its operator norm
by at most $r\rho^2$, so $G_B\succeq c_1kI_d$ for all sufficiently large $k$, with
$c_1>0$ depending on the fixed $\rho$.  Compactness of the quality interval gives a positive
lower bound on every $\kappa(q_i)$.  Hence $A$, $A_0$, and every matrix on their connecting
line segment have smallest eigenvalue at least $c_2k$.

\paragraph{The determinant term.}
Independence and centering eliminate the cross terms in the squared Frobenius norm, yielding
\begin{equation}
 \E\Delta=0,
 \qquad
 \E\|\Delta\|_F^2
 =\operatorname{Var}_{q\sim\mathcal D_q}(\kappa(q))\sum_{i\in B}\|h_i\|_2^4=O(k).
 \label{s2:delta-moment}
\end{equation}
For $A_t=A_0+t\Delta$, differentiation gives
\begin{equation}
 \frac{d}{dt}\log\det A_t=\tr(A_t^{-1}\Delta),\qquad
 \frac{d^2}{dt^2}\log\det A_t
 =-\tr(A_t^{-1}\Delta A_t^{-1}\Delta).
\end{equation}
The absolute second derivative is at most
$\|A_t^{-1}\|_{\op}^2\|\Delta\|_F^2\leq C k^{-2}\|\Delta\|_F^2$.
Taylor's formula and \cref{s2:delta-moment} therefore imply
\begin{equation}
 \E\log\det A=\log\det A_0+O(k^{-1}).
\end{equation}
Also, $A_0-(I_d+\overline\kappa G_S)
=\sum_{o\in J}(\kappa(q_o)-\overline\kappa)h_oh_o^\top$ has bounded norm.  The same
line-segment derivative identity, with inverse norm $O(k^{-1})$, gives
\begin{equation}
 \log\det A_0=\log\det(I_d+\overline\kappa G_S)+O(k^{-1}).
\end{equation}

\paragraph{The quadratic term.}
Decompose
\begin{equation}
 b_S=\overline b+\zeta,
 \qquad
 \overline b=b_J+\overline\lambda\sum_{i\in B}h_i,
 \qquad
 \zeta=\sum_{i\in B}(\lambda(q_i)-\overline\lambda)h_i.
\end{equation}
By balance and bounded spike contributions (Assumptions (A1) and (A3)), $\|\overline b\|=O(1)$.  Independence gives
$\E\zeta=0$ and $\E\zeta\zeta^\top=v_\lambda G_B$.  Consequently,
\begin{align}
 \E[b_S^\top A_0^{-1}b_S]
 &=\overline b^\top A_0^{-1}\overline b+v_\lambda\tr(A_0^{-1}G_B) =\frac{d\,v_\lambda}{\overline\kappa}+O(k^{-1}).
\end{align}
For the last equality, the first term is $O(k^{-1})$ and
\begin{equation}
 A_0^{-1}G_B
 =\frac{1}{\overline\kappa}
   \bigl(I_d-A_0^{-1}(I_d+\Sigma_J)\bigr).
\end{equation}

It remains to replace $A_0^{-1}$ by the correlated random inverse $A^{-1}$.
The resolvent identity yields
$A^{-1}-A_0^{-1}=-A^{-1}\Delta A_0^{-1}$.  Moreover,
$\E\|b_S\|_2^4=O(k^2)$.  To justify this last estimate, write each coordinate of $\zeta$
as a sum of independent centered bounded variables $X_i$.  Expanding its fourth power gives
\begin{equation}
 \E\left(\sum_iX_i\right)^4
 =\sum_i\E X_i^4+6\sum_{i<j}\E X_i^2\E X_j^2=O(k^2).
\end{equation}
There are only $d$ coordinates, and
$(\sum_{a=1}^d z_a^2)^2\leq d\sum_{a=1}^d z_a^4$, so the vector estimate follows;
adding $\overline b=O(1)$ does not change its order.  Cauchy--Schwarz now gives

\begin{equation}
 \E\left|b_S^\top(A^{-1}-A_0^{-1})b_S\right|
 \leq\frac{C}{k^2}\E[\|b_S\|_2^2\|\Delta\|_F] \leq\frac{C}{k^2}\sqrt{\E\|b_S\|_2^4\,\E\|\Delta\|_F^2}
 =O(k^{-1/2}).
 \label{s2:quadratic-random-inverse}
\end{equation}
Combining \cref{s2:quality-term}--\cref{s2:quadratic-random-inverse} with
\cref{s2:objective} proves the theorem.  Every bound used only the common constants in
(A1)--(A3), so the error is uniform over $\mathcal{A}_k$.
\end{proof}

The penalty term does not vanish, but on average it contributes the same constant $d\,v_\lambda/(2\overline\kappa)$ to every diverse set. The spike term in \cref{s2:log-limit}, in contrast, depends on which spikes $S$ keeps, so it cannot be dropped while keeping a spike is still a choice.

\paragraph{When the large-budget optimizer includes the spikes.}
\Cref{s2:expected-log} holds whichever spikes a set contains. To show that the best sets keep all the spikes, we need one more condition on $\mathcal A_k$, namely that a missing spike can always be swapped in for a bulk candidate.

\begin{definition}[Closure under insertion of a missing spike]
\label{s2:swap-closure}
A family $\mathcal A_k$ of size-$k$ subsets has the required closure property
if, whenever $S\in\mathcal A_k$ and $o\in O\setminus S$, there is a bulk
candidate $i\in S\setminus O$ such that
\begin{equation}
 S'=(S\setminus\{i\})\cup\{o\}\in\mathcal A_k.
\end{equation}
\end{definition}

This condition only requires one permissible replacement for each missing
spike. It is not automatic for a family defined only by the diversity conditions, and \cref{s2:compat} gives a concrete family for which it holds.

\begin{lemma}[The averaged objective keeps the spikes]
\label{s2:spike-inclusion}
Suppose $r\geq1$, every spike has a fixed quality $q_o>w$, and
$\mathcal A_k$ satisfies Definition~\ref{s2:swap-closure}.
Under the assumptions of Theorem~\ref{s2:expected-log}, every minimizer of
$\mathbb E \log\widetilde L(S)$ over $S\in\mathcal A_k$ contains all of $O$ for all
sufficiently large $k$.
\end{lemma}

\begin{proof}
We first bound the effect of one replacement on a log determinant.  Suppose
$S'=(S\setminus\{i\})\cup\{o\}$, where $i$ is a bulk candidate, and
both $S,S'\in\mathcal A_k$.  Fix $\nu>0$ and set
\begin{equation}
 A(t)=I_d+\nu\bigl((1-t)G_S+tG_{S'}\bigr),\qquad 0\leq t\leq1.
\end{equation}
The uniform covariance lower bound gives $A(t)\succeq(1+\nu ck)I_d$ for
some fixed $c>0$.  Also $G_{S'}-G_S=h_oh_o^\top-h_ih_i^\top$, so
\begin{align}
 \left|\frac{d}{dt}\log\det A(t)\right|
 &=\nu\left|h_o^\top A(t)^{-1}h_o-h_i^\top A(t)^{-1}h_i\right| \leq\frac{\nu(\|h_o\|_2^2+\|h_i\|_2^2)}{1+\nu ck}
 \leq\frac{2\nu\rho^2}{1+\nu ck}.
 \label{s2:swap-logdet-derivative}
\end{align}
Integrating over $t\in[0,1]$ shows, uniformly in the permissible replacement,
\begin{equation}
 \left|\log\det(I_d+\nu G_{S'})-\log\det(I_d+\nu G_S)\right|=O(k^{-1}).
 \label{s2:swap-logdet}
\end{equation}

Recall that
$\overline c=\mathbb E_{q\sim\mathcal D_q}\log\Phi(-q)$.  Since $q\leq w<q_o$ almost surely under $q\sim\mathcal D_q$
and $q\mapsto\log\Phi(-q)$ is strictly decreasing,
\begin{equation}
 \eta:=\min_{o\in O}\bigl(\overline c-\log\Phi(-q_o)\bigr)>0.
 \label{s2:spike-log-gap}
\end{equation}
The set $O$ is finite, so this minimum is a fixed positive number.
If $S$ misses spike $o$, choose the replacement supplied by the closure
condition.  The expansion in Theorem~\ref{s2:expected-log} and
\cref{s2:swap-logdet}, with $\nu=\overline\kappa$, give
\begin{align}
 \mathbb E \log\widetilde L(S')-\mathbb E \log\widetilde L(S)
 &=\log\Phi(-q_o)-\overline c
 -\frac12\left[\log\det(I_d+\overline\kappa G_{S'})
                 -\log\det(I_d+\overline\kappa G_S)\right]
 +O(k^{-1/2}) \\
 &\leq-\eta+O(k^{-1/2})<0
 \label{s2:spike-log-improvement}
\end{align}
for all sufficiently large $k$.  The threshold is uniform over the family.
Therefore no minimizer can omit a spike.
\end{proof}

Once all spikes are kept, the spike term is the same for every remaining set, so minimizing the average comes down to maximizing the log-determinant. The next corollary makes this precise.

\begin{corollary}
\label{s2:determinant-regret}
Let $\mathcal A_k^O=\{S\in\mathcal A_k:O\subseteq S\}$, and for $\nu>0$ let
\begin{equation}
 D_\nu(S)=\log\det(I_d+\nu G_S),\qquad
 S_k^{\det}(\nu)\in\operatorname*{arg\,max}_{S\in\mathcal A_k^O}D_\nu(S).
\end{equation}
Under the assumptions of \cref{s2:spike-inclusion}, let
$\widehat S_k^{\log}$ minimize $\mathbb E \log\widetilde L(S)$ over $\mathcal A_k$.
Then, for all sufficiently large $k$,
\begin{equation}
 0\leq D_{\overline\kappa}(S_k^{\det}(\overline\kappa))
       -D_{\overline\kappa}(\widehat S_k^{\log})=O(k^{-1/2}).
 \label{s2:log-determinant-regret}
\end{equation}

\end{corollary}

\begin{proof}
\Cref{s2:spike-inclusion} places each optimizer in
$\mathcal A_k^O$ for all sufficiently large $k$.
On this family, Theorem~\ref{s2:expected-log} has the form
\begin{equation}
 \mathbb E \log\widetilde L(S)=C_k-\tfrac12D_{\overline\kappa}(S)+e_k(S),
 \qquad \sup_{S\in\mathcal A_k^O}|e_k(S)|\leq C k^{-1/2},
\end{equation}
where $C_k$ is independent of $S$.  Optimality of
$\widehat S_k^{\log}$, compared with $S_k^{\det}(\overline\kappa)$, implies
\begin{equation}
 \tfrac12\bigl[D_{\overline\kappa}(S_k^{\det}(\overline\kappa))
 -D_{\overline\kappa}(\widehat S_k^{\log})\bigr]
 \leq e_k(S_k^{\det}(\overline\kappa))-e_k(\widehat S_k^{\log})
 \leq2C k^{-1/2}.
\end{equation}
This proves \cref{s2:log-determinant-regret}.
\end{proof}

Since $D_{\overline\kappa}$ is the RepExp objective with $\tau=1/\overline\kappa$ (\cref{app:repexp-limit}), the minimizer of the averaged objective nearly maximizes the RepExp objective among the diverse sets that keep all the spikes.

\subsection{An explicit model with both regimes}
\label{s2:compat}
This subsection constructs an explicit model in which the assumptions of \cref{s2:small-budget} and \cref{s2:large-budget} hold at the same time, for one fixed quality distribution and one fixed $\rho$. Let $N$ denote the total candidate-pool size, and consider a
nested sequence of pools as $N\to\infty$. Each pool contains the same
three spikes $O=\{o_1,o_2,o_3\}$, identified with the indices
$\{1,2,3\}$, and a bulk set
$\mathcal B_N=[N]\setminus O$ of size $N-3$. The pools may be coupled
by taking one infinite sequence of bulk candidates and letting
$\mathcal B_N$ contain its first $N-3$ members. Throughout the
construction, the dimension $d$ is fixed. For the large-budget regime,
fix any exponent $\alpha\in(0,1)$ and define the integer-valued budget
schedule
\begin{equation}
    k(N):=\left\lfloor N^\alpha\right\rfloor .
\end{equation}
Then $k(N)\to\infty$ and $k(N)/N\to0$ as $N\to\infty$, while
$3\leq k(N)\leq N-3$ for all sufficiently large $N$.
In the small-budget part of this construction,
$k\in\{1,2,3\}$ denotes a fixed selection budget, whereas $k(N)$ is
the budget used in the large-budget regime. We first choose the quality distribution, then
fix the common scale $\rho$, and finally construct the growing pools
and their families $\mathcal A_k$.

For the quality model in (A1), draw the bulk qualities independently from
$\mathcal D_q=\operatorname{Unif}[-\gamma,\gamma]$, once for the nested sequence
of pools, and let the three spike
qualities be
\begin{equation}
    q_{o_1}=4\gamma,\qquad q_{o_2}=3\gamma,\qquad q_{o_3}=2\gamma.
\end{equation}
Here $\gamma>0$ is a fixed constant.
Thus we may take $w_0=-\gamma$, $w=\gamma$, and $M=4\gamma$. Every spike has quality
strictly above the entire bulk quality distribution, and
\begin{equation}
    \operatorname{Var}_{q\sim\mathcal D_q}(q)=\frac{\gamma^2}{3}>0.
\end{equation}
In particular, the bulk qualities retain a fixed, nonzero variance
as $N$ grows.

Next, consider the small-budget quality gaps. Since
$q\mapsto\log\Phi(-q)$ is strictly decreasing, the quantity
\begin{align}
    \Gamma_*
    =\min\bigl\{
        \log\Phi(-3\gamma)-\log\Phi(-4\gamma),\,
        \log\Phi(-2\gamma)-\log\Phi(-3\gamma),\,
        \log\Phi(-\gamma)-\log\Phi(-2\gamma)
    \bigr\}
\end{align}
is strictly positive. These gaps do not depend on the number
of bulk candidates, since no bulk candidate has a quality above $\gamma$. More explicitly, since
$\frac{d}{dq}\log\Phi(-q)=-\lambda(q)$ and $\lambda'(q)=\kappa(q)>0$, each difference
in this minimum is the integral of $\lambda$ over an interval of
length $\gamma$ contained in $[\gamma,4\gamma]$. Consequently,
\begin{equation}
    \Gamma_*\geq \gamma\lambda(\gamma)>0.
\end{equation}

Let $\Lambda=\max_{q\in[-4\gamma,4\gamma]}\lambda(q)$, and choose the fixed
scale $\rho>0$ so that
\begin{equation}
    0<\rho^2<
    \min\left\{
        \frac13,\,
        \frac{\Gamma_*}{3(1+3\Lambda^2)}
    \right\}.
\end{equation}
Take deterministic vectors $v_i$ with $\|v_i\|_2\leq1$, and set
$h_i=\rho v_i$, as required by (A2). This choice gives
\begin{equation}
    3\rho^2\leq1,
    \qquad
    3(1+3\Lambda^2)\rho^2<\Gamma_*.
\end{equation}
For every $N>3$ and every realization of the bulk qualities,
\Cref{s2:spike-probe} therefore applies, so the first three
greedy selections are $o_1,o_2,o_3$ in that order. For each fixed
$k\in\{1,2,3\}$, \cref{s2:probe} also shows that the set
$\{o_1,\ldots,o_k\}$ is the unique size-$k$ minimizer of $\log\widetilde L$. Thus the small-budget conclusion holds
uniformly as the pool size $N$ grows.

For the large-budget part, we first describe a general way to obtain families with the closure property of Definition~\ref{s2:swap-closure}. Let
$\mathcal C_k$ be any nonempty deterministic family of bulk-only size-$k$
sets such that, for fixed constants $B_{\mathrm{core}},c_{\mathrm{core}},C_{\mathrm{core}}>0$
independent of $k$,
\begin{align}
 \left\|\sum_{i\in U}v_i\right\|_2&\leq B_{\mathrm{core}},
 &c_{\mathrm{core}}k I_d\preceq\sum_{i\in U}v_iv_i^\top&\preceq C_{\mathrm{core}}k I_d,
 \qquad U\in\mathcal C_k. \label{s2:core-family}
\end{align}
As before, $h_i=\rho v_i$ with $\|v_i\|_2\leq1$, and $\rho>0$ is fixed.
Define
\begin{equation}
 \mathcal A_k=
 \left\{(U\setminus R)\cup J:
 U\in\mathcal C_k,\ J\subseteq O,\ R\subseteq U,\ |R|=|J|\right\}.
 \label{s2:expanded-family}
\end{equation}
If $S=(U\setminus R)\cup J$ misses a spike $o$, replacing any
$i\in U\setminus R$ by $o$ gives the same representation with
$R\cup\{i\}$ and $J\cup\{o\}$.  Thus this family is closed in the sense
of Definition~\ref{s2:swap-closure}, for $k\geq r$.
Moreover, since at most $r$ vectors are removed or inserted,
\begin{align}
 \left\|\sum_{i\in S\setminus O}h_i\right\|_2
 &\leq\rho(B_{\mathrm{core}}+r), \label{s2:expanded-balance}\\
 \rho^2(c_{\mathrm{core}}k-r)I_d\preceq G_S
 &\preceq\rho^2(C_{\mathrm{core}}k+r)I_d. \label{s2:expanded-gram}
\end{align}
For the matrix bounds, use $0\preceq v_iv_i^\top\preceq I_d$ for each
removed or inserted vector.  For all sufficiently large $k$, these imply
(A3), for example with $B_0=B_{\mathrm{core}}+r$,
$c_0=c_{\mathrm{core}}/2$, and $C_0=C_{\mathrm{core}}+1$.  

We now build such a core family from antipodal pairs.
Let the bulk directions belong to
$\{\pm e_1,\ldots,\pm e_d\}$, where $e_1,\ldots,e_d$ are the
standard basis vectors. An antipodal pair consists of two distinct
candidates with directions $e_j$ and $-e_j$. Its contributions
to the vector sum and Gram matrix are, respectively,
\begin{equation}
    e_j+(-e_j)=0,
    \qquad
    e_je_j^\top+(-e_j)(-e_j)^\top=2e_je_j^\top.
\end{equation}
Thus pairing cancels the vector sum while preserving covariance
in direction $e_j$.

Write $b_\ell:=3+\ell$ for the $\ell$th candidate in the infinite
bulk sequence, and assign the directions of $b_1,b_2,\ldots$ in the
repeating order
\begin{equation}
    e_1,-e_1,e_2,-e_2,\ldots,e_d,-e_d,
    e_1,-e_1,\ldots .
\end{equation}
Thus $\mathcal B_N=\{b_1,\ldots,b_{N-3}\}$. The directions are
deterministic and independent of the random quality realizations, and
the spike directions may be arbitrary deterministic vectors of norm
at most one. For all sufficiently large $N$, define the bulk-only
core and its singleton family by
\begin{equation}
    U_N:=\{b_1,\ldots,b_{k(N)}\},
    \qquad
    \mathcal C_N:=\{U_N\}.
\end{equation}
Then $|U_N|=k(N)$. Let $n_{N,j}$ denote the number of complete
antipodal pairs on axis $j$ among the first
$2\lfloor k(N)/2\rfloor$ members of $U_N$. If $k(N)$ is even, $U_N$
consists entirely of these pairs, and the cyclic construction gives
\begin{equation}
    2\sum_{j=1}^d n_{N,j}=k(N),
    \qquad
    n_{N,j}\geq\frac{k(N)}{4d}
    \quad\text{for every }j
\end{equation}
for all sufficiently large $N$. The lower bound means that at least a
fraction $1/(2d)$ of the selected vectors lies on each axis. In this
case,
\begin{equation}
    \sum_{i\in U_N}v_i=0,
    \qquad
    \frac{k(N)}{2d}I_d
    \preceq
    \sum_{i\in U_N}v_iv_i^\top
    =2\operatorname{diag}(n_{N,1},\ldots,n_{N,d})
    \preceq k(N)I_d.
\end{equation}
If $k(N)$ is odd, its first $k(N)-1$ members form the complete pairs
counted by $n_{N,j}$, and these counts satisfy
\begin{equation}
    2\sum_{j=1}^d n_{N,j}=k(N)-1,
    \qquad
    n_{N,j}\geq\frac{k(N)-1}{4d}
\end{equation}
for all sufficiently large $N$. The remaining member has direction
$e_{j_N}$ for an axis $j_N\in\{1,\ldots,d\}$ determined by the cyclic ordering.
Thus
\begin{equation}
    \left\|\sum_{i\in U_N}v_i\right\|_2=1,
    \qquad
    \frac{k(N)-1}{2d}I_d
    \preceq
    \sum_{i\in U_N}v_iv_i^\top
    \preceq k(N)I_d.
\end{equation}
Consequently, for both parities and all sufficiently large $N$, the
core-family bounds hold with the fixed constants
\begin{equation}
    B_{\mathrm{core}}=1,
    \qquad
    c_{\mathrm{core}}=\frac{1}{4d},
    \qquad
    C_{\mathrm{core}}=1.
\end{equation}

Enlarging this core family as in \cref{s2:expanded-family}, with $r=3$, gives a family $\mathcal A_N^{\mathrm{con}}$ of size-$k(N)$ subsets of $[N]$ that is closed under insertion of a missing spike. By \cref{s2:expanded-balance,s2:expanded-gram}, it satisfies (A3) for all sufficiently large $N$, with $B_0=4$, $c_0=1/(8d)$ and $C_0=2$. When applying the general
large-budget results to this construction, their budget parameter
$k$ is instantiated as $k(N)$. To recover the generic budget-indexed
notation, set
\begin{equation}
    N_k:=\left\lceil k^{1/\alpha}\right\rceil,
    \qquad
    \mathcal A_k:=\mathcal A_{N_k}^{\mathrm{con}}.
\end{equation}
For $0<\alpha<1$, this choice satisfies
$\lfloor N_k^\alpha\rfloor=k$ and $N_k\geq k+3$ for all sufficiently
large $k$.

Consequently, \cref{s2:spike-inclusion} applies, so for all sufficiently large $N$ every minimizer of $\mathbb E \log\widetilde L(S)$ over $\mathcal A_N^{\mathrm{con}}$ contains all three spikes, and \cref{s2:determinant-regret} holds.

In summary, one fixed quality distribution and one fixed $\rho$ give both regimes. For every $k\leq3$, the correlation terms, of order $(1+3\Lambda^2)k\rho^2$, stay below the fixed gap $\Gamma_*$, so the surrogate greedy rule follows the probe-style rule for every pool size. As $k(N)\to\infty$, the Gram matrix of every set in $\mathcal A_N^{\mathrm{con}}$ grows like $k(N)\rho^2$, and the results of \cref{s2:large-budget} apply, while the kept fraction $k(N)/N$ of the pool vanishes.

\clearpage
\section{Benchmark-Stratified Results from \texorpdfstring{\cref{sec:results}}{Section 4.2}}
\label{app:per-benchmark}

This section reports the numbers behind \cref{fig:gains}. \Cref{tab:main} gives, for every model, the \passk of each pruning policy pooled over all test problems, and \crefrange{tab:app-qwen3-0.6b}{tab:app-skywork-7b} break these results down by benchmark, one table per model. All tables use the setting of \cref{sec:setup}, with a pool of $N = 128$ rollouts, a cutoff of $T = 4000$ tokens and eight-fold cross-validation repeated $20$ times. Each entry is the average \passk over the repetitions, in percent, and the value in parentheses is its standard error over those repetitions.

\begin{table}[H]
\centering
\setlength{\tabcolsep}{3pt}
\caption{Pass@$k$ (\%), pooled over problems.}
\label{tab:main}
\begin{minipage}[t]{0.49\linewidth}\centering
{\small\textbf{Qwen3-0.6B}}\\[2pt]
\resizebox{\linewidth}{!}{%
\begin{tabular}{@{}l*{5}{r}@{}}
\toprule
 & @1 & @2 & @4 & @8 & @16 \\
\midrule
Random (Pass@$k$) & 24.0\,{\scriptsize (0.00)} & 34.4\,{\scriptsize (0.00)} & 45.4\,{\scriptsize (0.00)} & 56.5\,{\scriptsize (0.00)} & 67.6\,{\scriptsize (0.00)} \\
\midrule
Linear Probe & 24.9\,{\scriptsize (0.27)} & 33.8\,{\scriptsize (0.25)} & 44.3\,{\scriptsize (0.21)} & 56.2\,{\scriptsize (0.30)} & 67.1\,{\scriptsize (0.27)} \\
RepExp & 23.9\,{\scriptsize (0.12)} & 33.5\,{\scriptsize (0.10)} & 45.0\,{\scriptsize (0.08)} & 55.9\,{\scriptsize (0.08)} & 67.0\,{\scriptsize (0.06)} \\
PKTP & 25.6\,{\scriptsize (0.29)} & 36.3\,{\scriptsize (0.22)} & 47.2\,{\scriptsize (0.29)} & 58.4\,{\scriptsize (0.31)} & 69.3\,{\scriptsize (0.28)} \\
\bottomrule
\end{tabular}}
\end{minipage}
\hfill
\begin{minipage}[t]{0.49\linewidth}\centering
{\small\textbf{Qwen3-1.7B}}\\[2pt]
\resizebox{\linewidth}{!}{%
\begin{tabular}{@{}l*{5}{r}@{}}
\toprule
 & @1 & @2 & @4 & @8 & @16 \\
\midrule
Random (Pass@$k$) & 39.5\,{\scriptsize (0.00)} & 51.5\,{\scriptsize (0.00)} & 62.6\,{\scriptsize (0.00)} & 72.5\,{\scriptsize (0.00)} & 81.0\,{\scriptsize (0.00)} \\
\midrule
Linear Probe & 40.7\,{\scriptsize (0.28)} & 52.0\,{\scriptsize (0.24)} & 63.1\,{\scriptsize (0.27)} & 72.5\,{\scriptsize (0.18)} & 79.9\,{\scriptsize (0.20)} \\
RepExp & 39.6\,{\scriptsize (0.08)} & 51.6\,{\scriptsize (0.08)} & 63.9\,{\scriptsize (0.07)} & 74.6\,{\scriptsize (0.04)} & 82.2\,{\scriptsize (0.03)} \\
PKTP & 39.3\,{\scriptsize (0.35)} & 52.7\,{\scriptsize (0.29)} & 64.4\,{\scriptsize (0.24)} & 73.8\,{\scriptsize (0.23)} & 81.7\,{\scriptsize (0.21)} \\
\bottomrule
\end{tabular}}
\end{minipage}
\par\vspace{6pt}
\begin{minipage}[t]{0.49\linewidth}\centering
{\small\textbf{OpenMath-Nemotron-1.5B}}\\[2pt]
\resizebox{\linewidth}{!}{%
\begin{tabular}{@{}l*{5}{r}@{}}
\toprule
 & @1 & @2 & @4 & @8 & @16 \\
\midrule
Random (Pass@$k$) & 47.1\,{\scriptsize (0.00)} & 59.3\,{\scriptsize (0.00)} & 69.3\,{\scriptsize (0.00)} & 77.3\,{\scriptsize (0.00)} & 83.9\,{\scriptsize (0.00)} \\
\midrule
Linear Probe & 48.2\,{\scriptsize (0.26)} & 59.8\,{\scriptsize (0.23)} & 69.4\,{\scriptsize (0.26)} & 77.6\,{\scriptsize (0.18)} & 84.0\,{\scriptsize (0.16)} \\
RepExp & 47.3\,{\scriptsize (0.06)} & 57.9\,{\scriptsize (0.07)} & 68.9\,{\scriptsize (0.05)} & 76.8\,{\scriptsize (0.05)} & 84.7\,{\scriptsize (0.03)} \\
PKTP & 48.1\,{\scriptsize (0.32)} & 59.4\,{\scriptsize (0.28)} & 69.4\,{\scriptsize (0.20)} & 78.3\,{\scriptsize (0.19)} & 85.2\,{\scriptsize (0.17)} \\
\bottomrule
\end{tabular}}
\end{minipage}
\hfill
\begin{minipage}[t]{0.49\linewidth}\centering
{\small\textbf{DeepSeek-R1-Distill-Qwen-1.5B}}\\[2pt]
\resizebox{\linewidth}{!}{%
\begin{tabular}{@{}l*{5}{r}@{}}
\toprule
 & @1 & @2 & @4 & @8 & @16 \\
\midrule
Random (Pass@$k$) & 27.4\,{\scriptsize (0.00)} & 38.4\,{\scriptsize (0.00)} & 49.5\,{\scriptsize (0.00)} & 60.2\,{\scriptsize (0.00)} & 70.3\,{\scriptsize (0.00)} \\
\midrule
Linear Probe & 30.1\,{\scriptsize (0.22)} & 39.8\,{\scriptsize (0.21)} & 49.9\,{\scriptsize (0.20)} & 59.4\,{\scriptsize (0.22)} & 69.7\,{\scriptsize (0.24)} \\
RepExp & 27.3\,{\scriptsize (0.08)} & 38.5\,{\scriptsize (0.07)} & 49.8\,{\scriptsize (0.04)} & 60.6\,{\scriptsize (0.06)} & 72.6\,{\scriptsize (0.03)} \\
PKTP & 30.3\,{\scriptsize (0.28)} & 41.4\,{\scriptsize (0.18)} & 52.9\,{\scriptsize (0.28)} & 62.5\,{\scriptsize (0.21)} & 71.7\,{\scriptsize (0.23)} \\
\bottomrule
\end{tabular}}
\end{minipage}
\par\vspace{6pt}
\begin{minipage}[t]{0.49\linewidth}\centering
{\small\textbf{Qwen3-8B}}\\[2pt]
\resizebox{\linewidth}{!}{%
\begin{tabular}{@{}l*{5}{r}@{}}
\toprule
 & @1 & @2 & @4 & @8 & @16 \\
\midrule
Random (Pass@$k$) & 46.9\,{\scriptsize (0.00)} & 58.0\,{\scriptsize (0.00)} & 67.2\,{\scriptsize (0.00)} & 75.2\,{\scriptsize (0.00)} & 82.3\,{\scriptsize (0.00)} \\
\midrule
Linear Probe & 49.2\,{\scriptsize (0.35)} & 59.2\,{\scriptsize (0.28)} & 67.4\,{\scriptsize (0.23)} & 74.4\,{\scriptsize (0.20)} & 81.2\,{\scriptsize (0.19)} \\
RepExp & 46.9\,{\scriptsize (0.09)} & 57.4\,{\scriptsize (0.06)} & 68.1\,{\scriptsize (0.07)} & 75.9\,{\scriptsize (0.04)} & 83.0\,{\scriptsize (0.04)} \\
PKTP & 47.3\,{\scriptsize (0.23)} & 58.7\,{\scriptsize (0.33)} & 68.6\,{\scriptsize (0.29)} & 76.0\,{\scriptsize (0.22)} & 83.0\,{\scriptsize (0.26)} \\
\bottomrule
\end{tabular}}
\end{minipage}
\hfill
\begin{minipage}[t]{0.49\linewidth}\centering
{\small\textbf{OpenMath-Nemotron-7B}}\\[2pt]
\resizebox{\linewidth}{!}{%
\begin{tabular}{@{}l*{5}{r}@{}}
\toprule
 & @1 & @2 & @4 & @8 & @16 \\
\midrule
Random (Pass@$k$) & 44.0\,{\scriptsize (0.00)} & 55.5\,{\scriptsize (0.00)} & 66.0\,{\scriptsize (0.00)} & 75.2\,{\scriptsize (0.00)} & 83.1\,{\scriptsize (0.00)} \\
\midrule
Linear Probe & 46.3\,{\scriptsize (0.21)} & 57.5\,{\scriptsize (0.22)} & 67.5\,{\scriptsize (0.23)} & 76.3\,{\scriptsize (0.16)} & 83.6\,{\scriptsize (0.17)} \\
RepExp & 44.0\,{\scriptsize (0.11)} & 56.8\,{\scriptsize (0.07)} & 66.4\,{\scriptsize (0.07)} & 76.3\,{\scriptsize (0.04)} & 84.7\,{\scriptsize (0.05)} \\
PKTP & 47.4\,{\scriptsize (0.29)} & 59.1\,{\scriptsize (0.35)} & 68.2\,{\scriptsize (0.29)} & 76.7\,{\scriptsize (0.24)} & 84.3\,{\scriptsize (0.21)} \\
\bottomrule
\end{tabular}}
\end{minipage}
\par\vspace{6pt}
\begin{minipage}[t]{0.49\linewidth}\centering
{\small\textbf{DeepSeek-R1-Distill-Qwen-7B}}\\[2pt]
\resizebox{\linewidth}{!}{%
\begin{tabular}{@{}l*{5}{r}@{}}
\toprule
 & @1 & @2 & @4 & @8 & @16 \\
\midrule
Random (Pass@$k$) & 37.5\,{\scriptsize (0.00)} & 48.3\,{\scriptsize (0.00)} & 58.5\,{\scriptsize (0.00)} & 68.2\,{\scriptsize (0.00)} & 77.5\,{\scriptsize (0.00)} \\
\midrule
Linear Probe & 37.4\,{\scriptsize (0.17)} & 46.7\,{\scriptsize (0.27)} & 56.4\,{\scriptsize (0.25)} & 66.0\,{\scriptsize (0.22)} & 76.6\,{\scriptsize (0.32)} \\
RepExp & 37.4\,{\scriptsize (0.08)} & 47.2\,{\scriptsize (0.08)} & 58.6\,{\scriptsize (0.07)} & 70.0\,{\scriptsize (0.04)} & 79.7\,{\scriptsize (0.04)} \\
PKTP & 38.6\,{\scriptsize (0.27)} & 49.5\,{\scriptsize (0.24)} & 59.7\,{\scriptsize (0.25)} & 70.0\,{\scriptsize (0.31)} & 79.3\,{\scriptsize (0.28)} \\
\bottomrule
\end{tabular}}
\end{minipage}
\hfill
\begin{minipage}[t]{0.49\linewidth}\centering
{\small\textbf{Skywork-OR1-Math-7B}}\\[2pt]
\resizebox{\linewidth}{!}{%
\begin{tabular}{@{}l*{5}{r}@{}}
\toprule
 & @1 & @2 & @4 & @8 & @16 \\
\midrule
Random (Pass@$k$) & 42.5\,{\scriptsize (0.00)} & 53.1\,{\scriptsize (0.00)} & 62.7\,{\scriptsize (0.00)} & 71.5\,{\scriptsize (0.00)} & 79.8\,{\scriptsize (0.00)} \\
\midrule
Linear Probe & 44.4\,{\scriptsize (0.24)} & 54.2\,{\scriptsize (0.22)} & 62.9\,{\scriptsize (0.18)} & 71.3\,{\scriptsize (0.29)} & 79.4\,{\scriptsize (0.22)} \\
RepExp & 42.5\,{\scriptsize (0.09)} & 54.5\,{\scriptsize (0.09)} & 64.5\,{\scriptsize (0.06)} & 72.1\,{\scriptsize (0.04)} & 79.4\,{\scriptsize (0.05)} \\
PKTP & 44.7\,{\scriptsize (0.34)} & 55.7\,{\scriptsize (0.22)} & 65.6\,{\scriptsize (0.26)} & 73.2\,{\scriptsize (0.29)} & 80.3\,{\scriptsize (0.22)} \\
\bottomrule
\end{tabular}}
\end{minipage}
\end{table}

\begin{table}[H]
\centering
\setlength{\tabcolsep}{3pt}
\caption{Per-benchmark pass@$k$ (\%), Qwen3-0.6B.}
\label{tab:app-qwen3-0.6b}
\begin{minipage}[t]{0.49\linewidth}\centering
{\small\textbf{AIME}}\\[2pt]
\resizebox{\linewidth}{!}{%
\begin{tabular}{@{}l*{5}{r}@{}}
\toprule
 & @1 & @2 & @4 & @8 & @16 \\
\midrule
Random (Pass@$k$) & 23.2\,{\scriptsize (0.00)} & 33.8\,{\scriptsize (0.00)} & 45.5\,{\scriptsize (0.00)} & 58.3\,{\scriptsize (0.00)} & 71.2\,{\scriptsize (0.00)} \\
\midrule
Linear Probe & 26.6\,{\scriptsize (0.70)} & 35.6\,{\scriptsize (0.61)} & 43.6\,{\scriptsize (0.84)} & 56.5\,{\scriptsize (0.89)} & 72.5\,{\scriptsize (0.79)} \\
RepExp & 23.3\,{\scriptsize (0.28)} & 29.5\,{\scriptsize (0.23)} & 41.4\,{\scriptsize (0.20)} & 53.7\,{\scriptsize (0.22)} & 66.5\,{\scriptsize (0.15)} \\
PKTP & 25.2\,{\scriptsize (1.07)} & 36.7\,{\scriptsize (1.01)} & 50.2\,{\scriptsize (1.11)} & 64.2\,{\scriptsize (0.79)} & 75.5\,{\scriptsize (0.78)} \\
\bottomrule
\end{tabular}}
\end{minipage}
\hfill
\begin{minipage}[t]{0.49\linewidth}\centering
{\small\textbf{HMMT}}\\[2pt]
\resizebox{\linewidth}{!}{%
\begin{tabular}{@{}l*{5}{r}@{}}
\toprule
 & @1 & @2 & @4 & @8 & @16 \\
\midrule
Random (Pass@$k$) & 20.8\,{\scriptsize (0.00)} & 29.9\,{\scriptsize (0.00)} & 39.9\,{\scriptsize (0.00)} & 50.8\,{\scriptsize (0.00)} & 62.1\,{\scriptsize (0.00)} \\
\midrule
Linear Probe & 27.1\,{\scriptsize (0.68)} & 33.3\,{\scriptsize (0.84)} & 43.0\,{\scriptsize (0.69)} & 49.9\,{\scriptsize (0.60)} & 60.5\,{\scriptsize (0.60)} \\
RepExp & 20.3\,{\scriptsize (0.24)} & 30.9\,{\scriptsize (0.25)} & 39.1\,{\scriptsize (0.20)} & 47.7\,{\scriptsize (0.20)} & 60.0\,{\scriptsize (0.14)} \\
PKTP & 26.3\,{\scriptsize (0.95)} & 39.0\,{\scriptsize (0.75)} & 48.5\,{\scriptsize (0.83)} & 55.0\,{\scriptsize (0.52)} & 63.4\,{\scriptsize (0.56)} \\
\bottomrule
\end{tabular}}
\end{minipage}
\par\vspace{6pt}
\begin{minipage}[t]{0.49\linewidth}\centering
{\small\textbf{BeyondAIME}}\\[2pt]
\resizebox{\linewidth}{!}{%
\begin{tabular}{@{}l*{5}{r}@{}}
\toprule
 & @1 & @2 & @4 & @8 & @16 \\
\midrule
Random (Pass@$k$) & 11.7\,{\scriptsize (0.00)} & 19.5\,{\scriptsize (0.00)} & 29.6\,{\scriptsize (0.00)} & 41.2\,{\scriptsize (0.00)} & 54.2\,{\scriptsize (0.00)} \\
\midrule
Linear Probe & 15.9\,{\scriptsize (0.76)} & 22.6\,{\scriptsize (0.76)} & 34.1\,{\scriptsize (0.86)} & 44.3\,{\scriptsize (1.09)} & 52.0\,{\scriptsize (0.89)} \\
RepExp & 11.7\,{\scriptsize (0.28)} & 19.2\,{\scriptsize (0.22)} & 30.8\,{\scriptsize (0.24)} & 44.8\,{\scriptsize (0.22)} & 58.8\,{\scriptsize (0.14)} \\
PKTP & 12.8\,{\scriptsize (0.72)} & 19.2\,{\scriptsize (0.76)} & 28.0\,{\scriptsize (0.76)} & 41.2\,{\scriptsize (1.02)} & 54.8\,{\scriptsize (0.68)} \\
\bottomrule
\end{tabular}}
\end{minipage}
\hfill
\begin{minipage}[t]{0.49\linewidth}\centering
{\small\textbf{OlymMATH (Easy)}}\\[2pt]
\resizebox{\linewidth}{!}{%
\begin{tabular}{@{}l*{5}{r}@{}}
\toprule
 & @1 & @2 & @4 & @8 & @16 \\
\midrule
Random (Pass@$k$) & 13.3\,{\scriptsize (0.00)} & 21.4\,{\scriptsize (0.00)} & 31.7\,{\scriptsize (0.00)} & 44.0\,{\scriptsize (0.00)} & 57.8\,{\scriptsize (0.00)} \\
\midrule
Linear Probe & 13.1\,{\scriptsize (0.55)} & 19.6\,{\scriptsize (0.63)} & 29.0\,{\scriptsize (0.68)} & 41.2\,{\scriptsize (0.74)} & 53.0\,{\scriptsize (0.87)} \\
RepExp & 13.4\,{\scriptsize (0.18)} & 21.2\,{\scriptsize (0.18)} & 31.6\,{\scriptsize (0.18)} & 42.3\,{\scriptsize (0.21)} & 58.5\,{\scriptsize (0.15)} \\
PKTP & 13.3\,{\scriptsize (0.51)} & 19.5\,{\scriptsize (0.71)} & 28.5\,{\scriptsize (0.77)} & 43.1\,{\scriptsize (1.11)} & 59.7\,{\scriptsize (0.80)} \\
\bottomrule
\end{tabular}}
\end{minipage}
\par\vspace{6pt}
\begin{minipage}[t]{0.49\linewidth}\centering
{\small\textbf{Omni-MATH}}\\[2pt]
\resizebox{\linewidth}{!}{%
\begin{tabular}{@{}l*{5}{r}@{}}
\toprule
 & @1 & @2 & @4 & @8 & @16 \\
\midrule
Random (Pass@$k$) & 28.3\,{\scriptsize (0.00)} & 39.4\,{\scriptsize (0.00)} & 50.2\,{\scriptsize (0.00)} & 59.7\,{\scriptsize (0.00)} & 68.8\,{\scriptsize (0.00)} \\
\midrule
Linear Probe & 27.1\,{\scriptsize (0.63)} & 36.7\,{\scriptsize (0.77)} & 48.4\,{\scriptsize (0.41)} & 60.9\,{\scriptsize (0.61)} & 70.1\,{\scriptsize (0.51)} \\
RepExp & 28.0\,{\scriptsize (0.20)} & 39.7\,{\scriptsize (0.17)} & 52.6\,{\scriptsize (0.12)} & 59.8\,{\scriptsize (0.11)} & 67.1\,{\scriptsize (0.10)} \\
PKTP & 28.7\,{\scriptsize (0.72)} & 40.2\,{\scriptsize (0.72)} & 52.5\,{\scriptsize (0.75)} & 62.5\,{\scriptsize (0.41)} & 71.9\,{\scriptsize (0.58)} \\
\bottomrule
\end{tabular}}
\end{minipage}
\hfill
\begin{minipage}[t]{0.49\linewidth}\centering
{\small\textbf{RIMO-N}}\\[2pt]
\resizebox{\linewidth}{!}{%
\begin{tabular}{@{}l*{5}{r}@{}}
\toprule
 & @1 & @2 & @4 & @8 & @16 \\
\midrule
Random (Pass@$k$) & 38.9\,{\scriptsize (0.00)} & 53.0\,{\scriptsize (0.00)} & 65.5\,{\scriptsize (0.00)} & 76.0\,{\scriptsize (0.00)} & 84.3\,{\scriptsize (0.00)} \\
\midrule
Linear Probe & 36.5\,{\scriptsize (0.60)} & 49.9\,{\scriptsize (0.66)} & 61.6\,{\scriptsize (0.55)} & 75.7\,{\scriptsize (0.72)} & 86.1\,{\scriptsize (0.50)} \\
RepExp & 38.9\,{\scriptsize (0.28)} & 50.9\,{\scriptsize (0.20)} & 63.3\,{\scriptsize (0.21)} & 77.2\,{\scriptsize (0.11)} & 84.2\,{\scriptsize (0.08)} \\
PKTP & 41.0\,{\scriptsize (0.86)} & 56.5\,{\scriptsize (0.69)} & 67.6\,{\scriptsize (0.49)} & 76.5\,{\scriptsize (0.41)} & 83.2\,{\scriptsize (0.41)} \\
\bottomrule
\end{tabular}}
\end{minipage}
\end{table}

\begin{table}[H]
\centering
\setlength{\tabcolsep}{3pt}
\caption{Per-benchmark pass@$k$ (\%), Qwen3-1.7B.}
\label{tab:app-qwen3-1.7b}
\begin{minipage}[t]{0.49\linewidth}\centering
{\small\textbf{AIME}}\\[2pt]
\resizebox{\linewidth}{!}{%
\begin{tabular}{@{}l*{5}{r}@{}}
\toprule
 & @1 & @2 & @4 & @8 & @16 \\
\midrule
Random (Pass@$k$) & 47.0\,{\scriptsize (0.00)} & 59.6\,{\scriptsize (0.00)} & 70.4\,{\scriptsize (0.00)} & 78.9\,{\scriptsize (0.00)} & 85.4\,{\scriptsize (0.00)} \\
\midrule
Linear Probe & 51.7\,{\scriptsize (0.60)} & 63.3\,{\scriptsize (0.65)} & 73.2\,{\scriptsize (0.56)} & 81.0\,{\scriptsize (0.57)} & 87.3\,{\scriptsize (0.31)} \\
RepExp & 47.4\,{\scriptsize (0.21)} & 60.7\,{\scriptsize (0.23)} & 72.2\,{\scriptsize (0.19)} & 85.6\,{\scriptsize (0.09)} & 90.6\,{\scriptsize (0.07)} \\
PKTP & 44.0\,{\scriptsize (0.96)} & 58.3\,{\scriptsize (0.62)} & 71.5\,{\scriptsize (0.57)} & 80.3\,{\scriptsize (0.44)} & 87.2\,{\scriptsize (0.41)} \\
\bottomrule
\end{tabular}}
\end{minipage}
\hfill
\begin{minipage}[t]{0.49\linewidth}\centering
{\small\textbf{HMMT}}\\[2pt]
\resizebox{\linewidth}{!}{%
\begin{tabular}{@{}l*{5}{r}@{}}
\toprule
 & @1 & @2 & @4 & @8 & @16 \\
\midrule
Random (Pass@$k$) & 41.9\,{\scriptsize (0.00)} & 52.4\,{\scriptsize (0.00)} & 62.7\,{\scriptsize (0.00)} & 72.7\,{\scriptsize (0.00)} & 82.2\,{\scriptsize (0.00)} \\
\midrule
Linear Probe & 45.8\,{\scriptsize (0.58)} & 56.0\,{\scriptsize (0.63)} & 62.1\,{\scriptsize (0.48)} & 69.9\,{\scriptsize (0.51)} & 78.1\,{\scriptsize (0.63)} \\
RepExp & 41.7\,{\scriptsize (0.21)} & 51.7\,{\scriptsize (0.22)} & 64.2\,{\scriptsize (0.18)} & 76.7\,{\scriptsize (0.18)} & 85.2\,{\scriptsize (0.09)} \\
PKTP & 42.7\,{\scriptsize (0.83)} & 54.8\,{\scriptsize (0.91)} & 66.6\,{\scriptsize (0.67)} & 74.7\,{\scriptsize (0.51)} & 82.1\,{\scriptsize (0.65)} \\
\bottomrule
\end{tabular}}
\end{minipage}
\par\vspace{6pt}
\begin{minipage}[t]{0.49\linewidth}\centering
{\small\textbf{BeyondAIME}}\\[2pt]
\resizebox{\linewidth}{!}{%
\begin{tabular}{@{}l*{5}{r}@{}}
\toprule
 & @1 & @2 & @4 & @8 & @16 \\
\midrule
Random (Pass@$k$) & 35.0\,{\scriptsize (0.00)} & 45.7\,{\scriptsize (0.00)} & 55.8\,{\scriptsize (0.00)} & 65.8\,{\scriptsize (0.00)} & 75.7\,{\scriptsize (0.00)} \\
\midrule
Linear Probe & 34.2\,{\scriptsize (0.67)} & 44.4\,{\scriptsize (0.81)} & 55.6\,{\scriptsize (0.88)} & 65.9\,{\scriptsize (0.60)} & 76.2\,{\scriptsize (0.60)} \\
RepExp & 34.9\,{\scriptsize (0.23)} & 41.6\,{\scriptsize (0.24)} & 55.2\,{\scriptsize (0.14)} & 66.7\,{\scriptsize (0.14)} & 77.3\,{\scriptsize (0.09)} \\
PKTP & 32.0\,{\scriptsize (0.79)} & 43.7\,{\scriptsize (0.70)} & 54.3\,{\scriptsize (0.74)} & 65.2\,{\scriptsize (0.61)} & 76.0\,{\scriptsize (0.75)} \\
\bottomrule
\end{tabular}}
\end{minipage}
\hfill
\begin{minipage}[t]{0.49\linewidth}\centering
{\small\textbf{OlymMATH (Easy)}}\\[2pt]
\resizebox{\linewidth}{!}{%
\begin{tabular}{@{}l*{5}{r}@{}}
\toprule
 & @1 & @2 & @4 & @8 & @16 \\
\midrule
Random (Pass@$k$) & 42.4\,{\scriptsize (0.00)} & 57.5\,{\scriptsize (0.00)} & 71.3\,{\scriptsize (0.00)} & 81.8\,{\scriptsize (0.00)} & 89.6\,{\scriptsize (0.00)} \\
\midrule
Linear Probe & 41.8\,{\scriptsize (0.54)} & 56.1\,{\scriptsize (0.70)} & 72.3\,{\scriptsize (0.65)} & 84.2\,{\scriptsize (0.51)} & 89.7\,{\scriptsize (0.33)} \\
RepExp & 42.6\,{\scriptsize (0.19)} & 56.6\,{\scriptsize (0.18)} & 71.3\,{\scriptsize (0.14)} & 79.7\,{\scriptsize (0.10)} & 89.0\,{\scriptsize (0.09)} \\
PKTP & 42.9\,{\scriptsize (0.71)} & 59.2\,{\scriptsize (0.84)} & 73.4\,{\scriptsize (0.56)} & 84.5\,{\scriptsize (0.46)} & 91.4\,{\scriptsize (0.37)} \\
\bottomrule
\end{tabular}}
\end{minipage}
\par\vspace{6pt}
\begin{minipage}[t]{0.49\linewidth}\centering
{\small\textbf{Omni-MATH}}\\[2pt]
\resizebox{\linewidth}{!}{%
\begin{tabular}{@{}l*{5}{r}@{}}
\toprule
 & @1 & @2 & @4 & @8 & @16 \\
\midrule
Random (Pass@$k$) & 31.9\,{\scriptsize (0.00)} & 42.0\,{\scriptsize (0.00)} & 51.8\,{\scriptsize (0.00)} & 61.2\,{\scriptsize (0.00)} & 70.3\,{\scriptsize (0.00)} \\
\midrule
Linear Probe & 27.8\,{\scriptsize (0.44)} & 38.5\,{\scriptsize (0.56)} & 49.7\,{\scriptsize (0.53)} & 59.5\,{\scriptsize (0.25)} & 66.4\,{\scriptsize (0.50)} \\
RepExp & 31.9\,{\scriptsize (0.17)} & 44.2\,{\scriptsize (0.16)} & 55.1\,{\scriptsize (0.11)} & 64.4\,{\scriptsize (0.09)} & 70.1\,{\scriptsize (0.08)} \\
PKTP & 31.3\,{\scriptsize (0.70)} & 42.1\,{\scriptsize (0.41)} & 51.2\,{\scriptsize (0.38)} & 61.2\,{\scriptsize (0.47)} & 70.2\,{\scriptsize (0.26)} \\
\bottomrule
\end{tabular}}
\end{minipage}
\hfill
\begin{minipage}[t]{0.49\linewidth}\centering
{\small\textbf{RIMO-N}}\\[2pt]
\resizebox{\linewidth}{!}{%
\begin{tabular}{@{}l*{5}{r}@{}}
\toprule
 & @1 & @2 & @4 & @8 & @16 \\
\midrule
Random (Pass@$k$) & 41.5\,{\scriptsize (0.00)} & 54.2\,{\scriptsize (0.00)} & 65.8\,{\scriptsize (0.00)} & 76.3\,{\scriptsize (0.00)} & 84.9\,{\scriptsize (0.00)} \\
\midrule
Linear Probe & 48.8\,{\scriptsize (0.70)} & 58.9\,{\scriptsize (0.52)} & 68.6\,{\scriptsize (0.64)} & 76.3\,{\scriptsize (0.65)} & 84.6\,{\scriptsize (0.40)} \\
RepExp & 41.6\,{\scriptsize (0.24)} & 55.8\,{\scriptsize (0.20)} & 66.4\,{\scriptsize (0.19)} & 76.9\,{\scriptsize (0.16)} & 85.0\,{\scriptsize (0.10)} \\
PKTP & 45.7\,{\scriptsize (0.76)} & 60.8\,{\scriptsize (0.59)} & 72.3\,{\scriptsize (0.65)} & 79.2\,{\scriptsize (0.48)} & 85.4\,{\scriptsize (0.55)} \\
\bottomrule
\end{tabular}}
\end{minipage}
\end{table}

\begin{table}[H]
\centering
\setlength{\tabcolsep}{3pt}
\caption{Per-benchmark pass@$k$ (\%), OpenMath-Nemotron-1.5B.}
\label{tab:app-nemotron-1.5b}
\begin{minipage}[t]{0.49\linewidth}\centering
{\small\textbf{AIME}}\\[2pt]
\resizebox{\linewidth}{!}{%
\begin{tabular}{@{}l*{5}{r}@{}}
\toprule
 & @1 & @2 & @4 & @8 & @16 \\
\midrule
Random (Pass@$k$) & 59.1\,{\scriptsize (0.00)} & 72.1\,{\scriptsize (0.00)} & 80.8\,{\scriptsize (0.00)} & 86.2\,{\scriptsize (0.00)} & 90.1\,{\scriptsize (0.00)} \\
\midrule
Linear Probe & 63.5\,{\scriptsize (0.68)} & 75.4\,{\scriptsize (0.43)} & 83.3\,{\scriptsize (0.45)} & 89.1\,{\scriptsize (0.37)} & 92.1\,{\scriptsize (0.25)} \\
RepExp & 59.1\,{\scriptsize (0.17)} & 70.7\,{\scriptsize (0.15)} & 77.6\,{\scriptsize (0.11)} & 84.7\,{\scriptsize (0.10)} & 91.6\,{\scriptsize (0.06)} \\
PKTP & 62.1\,{\scriptsize (0.61)} & 74.1\,{\scriptsize (0.75)} & 82.3\,{\scriptsize (0.41)} & 85.9\,{\scriptsize (0.33)} & 89.9\,{\scriptsize (0.36)} \\
\bottomrule
\end{tabular}}
\end{minipage}
\hfill
\begin{minipage}[t]{0.49\linewidth}\centering
{\small\textbf{HMMT}}\\[2pt]
\resizebox{\linewidth}{!}{%
\begin{tabular}{@{}l*{5}{r}@{}}
\toprule
 & @1 & @2 & @4 & @8 & @16 \\
\midrule
Random (Pass@$k$) & 41.3\,{\scriptsize (0.00)} & 51.2\,{\scriptsize (0.00)} & 60.5\,{\scriptsize (0.00)} & 69.1\,{\scriptsize (0.00)} & 76.9\,{\scriptsize (0.00)} \\
\midrule
Linear Probe & 45.5\,{\scriptsize (0.60)} & 55.2\,{\scriptsize (0.69)} & 63.9\,{\scriptsize (0.69)} & 71.6\,{\scriptsize (0.59)} & 77.4\,{\scriptsize (0.35)} \\
RepExp & 41.5\,{\scriptsize (0.26)} & 51.5\,{\scriptsize (0.24)} & 63.7\,{\scriptsize (0.16)} & 72.2\,{\scriptsize (0.15)} & 80.0\,{\scriptsize (0.10)} \\
PKTP & 49.3\,{\scriptsize (0.64)} & 56.4\,{\scriptsize (0.71)} & 65.6\,{\scriptsize (0.72)} & 74.2\,{\scriptsize (0.45)} & 83.0\,{\scriptsize (0.40)} \\
\bottomrule
\end{tabular}}
\end{minipage}
\par\vspace{6pt}
\begin{minipage}[t]{0.49\linewidth}\centering
{\small\textbf{BeyondAIME}}\\[2pt]
\resizebox{\linewidth}{!}{%
\begin{tabular}{@{}l*{5}{r}@{}}
\toprule
 & @1 & @2 & @4 & @8 & @16 \\
\midrule
Random (Pass@$k$) & 43.8\,{\scriptsize (0.00)} & 54.7\,{\scriptsize (0.00)} & 63.8\,{\scriptsize (0.00)} & 71.5\,{\scriptsize (0.00)} & 78.1\,{\scriptsize (0.00)} \\
\midrule
Linear Probe & 46.0\,{\scriptsize (0.71)} & 56.8\,{\scriptsize (0.67)} & 65.7\,{\scriptsize (0.67)} & 74.1\,{\scriptsize (0.46)} & 81.1\,{\scriptsize (0.42)} \\
RepExp & 43.8\,{\scriptsize (0.16)} & 52.9\,{\scriptsize (0.14)} & 63.4\,{\scriptsize (0.12)} & 72.3\,{\scriptsize (0.13)} & 79.6\,{\scriptsize (0.12)} \\
PKTP & 43.8\,{\scriptsize (0.51)} & 53.9\,{\scriptsize (0.63)} & 62.7\,{\scriptsize (0.55)} & 72.2\,{\scriptsize (0.45)} & 80.9\,{\scriptsize (0.56)} \\
\bottomrule
\end{tabular}}
\end{minipage}
\hfill
\begin{minipage}[t]{0.49\linewidth}\centering
{\small\textbf{OlymMATH (Easy)}}\\[2pt]
\resizebox{\linewidth}{!}{%
\begin{tabular}{@{}l*{5}{r}@{}}
\toprule
 & @1 & @2 & @4 & @8 & @16 \\
\midrule
Random (Pass@$k$) & 54.7\,{\scriptsize (0.00)} & 69.9\,{\scriptsize (0.00)} & 81.4\,{\scriptsize (0.00)} & 88.7\,{\scriptsize (0.00)} & 92.9\,{\scriptsize (0.00)} \\
\midrule
Linear Probe & 58.8\,{\scriptsize (0.61)} & 69.2\,{\scriptsize (0.59)} & 79.4\,{\scriptsize (0.50)} & 88.4\,{\scriptsize (0.48)} & 93.9\,{\scriptsize (0.20)} \\
RepExp & 54.9\,{\scriptsize (0.16)} & 66.1\,{\scriptsize (0.16)} & 75.4\,{\scriptsize (0.14)} & 83.2\,{\scriptsize (0.13)} & 89.9\,{\scriptsize (0.06)} \\
PKTP & 59.4\,{\scriptsize (0.53)} & 72.2\,{\scriptsize (0.61)} & 80.8\,{\scriptsize (0.42)} & 87.6\,{\scriptsize (0.32)} & 91.5\,{\scriptsize (0.24)} \\
\bottomrule
\end{tabular}}
\end{minipage}
\par\vspace{6pt}
\begin{minipage}[t]{0.49\linewidth}\centering
{\small\textbf{Omni-MATH}}\\[2pt]
\resizebox{\linewidth}{!}{%
\begin{tabular}{@{}l*{5}{r}@{}}
\toprule
 & @1 & @2 & @4 & @8 & @16 \\
\midrule
Random (Pass@$k$) & 41.2\,{\scriptsize (0.00)} & 52.5\,{\scriptsize (0.00)} & 62.8\,{\scriptsize (0.00)} & 72.1\,{\scriptsize (0.00)} & 80.6\,{\scriptsize (0.00)} \\
\midrule
Linear Probe & 37.4\,{\scriptsize (0.38)} & 50.3\,{\scriptsize (0.39)} & 61.2\,{\scriptsize (0.47)} & 70.1\,{\scriptsize (0.32)} & 78.6\,{\scriptsize (0.35)} \\
RepExp & 41.2\,{\scriptsize (0.15)} & 51.8\,{\scriptsize (0.15)} & 65.0\,{\scriptsize (0.13)} & 72.1\,{\scriptsize (0.08)} & 81.9\,{\scriptsize (0.04)} \\
PKTP & 38.0\,{\scriptsize (0.59)} & 50.5\,{\scriptsize (0.49)} & 62.9\,{\scriptsize (0.66)} & 74.6\,{\scriptsize (0.45)} & 82.2\,{\scriptsize (0.30)} \\
\bottomrule
\end{tabular}}
\end{minipage}
\hfill
\begin{minipage}[t]{0.49\linewidth}\centering
{\small\textbf{RIMO-N}}\\[2pt]
\resizebox{\linewidth}{!}{%
\begin{tabular}{@{}l*{5}{r}@{}}
\toprule
 & @1 & @2 & @4 & @8 & @16 \\
\midrule
Random (Pass@$k$) & 44.7\,{\scriptsize (0.00)} & 56.7\,{\scriptsize (0.00)} & 67.2\,{\scriptsize (0.00)} & 75.9\,{\scriptsize (0.00)} & 83.6\,{\scriptsize (0.00)} \\
\midrule
Linear Probe & 43.8\,{\scriptsize (0.80)} & 56.4\,{\scriptsize (0.82)} & 66.0\,{\scriptsize (0.55)} & 74.4\,{\scriptsize (0.51)} & 81.7\,{\scriptsize (0.59)} \\
RepExp & 45.0\,{\scriptsize (0.21)} & 55.9\,{\scriptsize (0.22)} & 69.0\,{\scriptsize (0.19)} & 77.5\,{\scriptsize (0.13)} & 85.7\,{\scriptsize (0.09)} \\
PKTP & 41.5\,{\scriptsize (0.80)} & 52.5\,{\scriptsize (0.79)} & 63.1\,{\scriptsize (0.78)} & 74.9\,{\scriptsize (0.52)} & 84.2\,{\scriptsize (0.41)} \\
\bottomrule
\end{tabular}}
\end{minipage}
\end{table}

\begin{table}[H]
\centering
\setlength{\tabcolsep}{3pt}
\caption{Per-benchmark pass@$k$ (\%), DeepSeek-R1-Distill-Qwen-1.5B.}
\label{tab:app-r1d-1.5b}
\begin{minipage}[t]{0.49\linewidth}\centering
{\small\textbf{AIME}}\\[2pt]
\resizebox{\linewidth}{!}{%
\begin{tabular}{@{}l*{5}{r}@{}}
\toprule
 & @1 & @2 & @4 & @8 & @16 \\
\midrule
Random (Pass@$k$) & 31.8\,{\scriptsize (0.00)} & 42.5\,{\scriptsize (0.00)} & 53.1\,{\scriptsize (0.00)} & 63.4\,{\scriptsize (0.00)} & 73.2\,{\scriptsize (0.00)} \\
\midrule
Linear Probe & 39.3\,{\scriptsize (0.55)} & 47.2\,{\scriptsize (0.41)} & 57.0\,{\scriptsize (0.57)} & 65.4\,{\scriptsize (0.59)} & 75.3\,{\scriptsize (0.50)} \\
RepExp & 31.9\,{\scriptsize (0.21)} & 38.9\,{\scriptsize (0.13)} & 50.2\,{\scriptsize (0.15)} & 63.1\,{\scriptsize (0.18)} & 73.8\,{\scriptsize (0.11)} \\
PKTP & 37.2\,{\scriptsize (0.67)} & 48.9\,{\scriptsize (0.84)} & 59.3\,{\scriptsize (0.73)} & 67.8\,{\scriptsize (0.63)} & 76.1\,{\scriptsize (0.42)} \\
\bottomrule
\end{tabular}}
\end{minipage}
\hfill
\begin{minipage}[t]{0.49\linewidth}\centering
{\small\textbf{HMMT}}\\[2pt]
\resizebox{\linewidth}{!}{%
\begin{tabular}{@{}l*{5}{r}@{}}
\toprule
 & @1 & @2 & @4 & @8 & @16 \\
\midrule
Random (Pass@$k$) & 27.0\,{\scriptsize (0.00)} & 37.1\,{\scriptsize (0.00)} & 47.5\,{\scriptsize (0.00)} & 57.3\,{\scriptsize (0.00)} & 66.9\,{\scriptsize (0.00)} \\
\midrule
Linear Probe & 34.6\,{\scriptsize (0.92)} & 45.2\,{\scriptsize (0.73)} & 54.2\,{\scriptsize (0.62)} & 60.1\,{\scriptsize (0.50)} & 68.3\,{\scriptsize (0.65)} \\
RepExp & 26.8\,{\scriptsize (0.34)} & 36.1\,{\scriptsize (0.28)} & 47.7\,{\scriptsize (0.15)} & 54.0\,{\scriptsize (0.11)} & 67.0\,{\scriptsize (0.14)} \\
PKTP & 34.1\,{\scriptsize (0.73)} & 43.3\,{\scriptsize (0.72)} & 49.5\,{\scriptsize (0.69)} & 56.0\,{\scriptsize (0.65)} & 64.1\,{\scriptsize (0.72)} \\
\bottomrule
\end{tabular}}
\end{minipage}
\par\vspace{6pt}
\begin{minipage}[t]{0.49\linewidth}\centering
{\small\textbf{BeyondAIME}}\\[2pt]
\resizebox{\linewidth}{!}{%
\begin{tabular}{@{}l*{5}{r}@{}}
\toprule
 & @1 & @2 & @4 & @8 & @16 \\
\midrule
Random (Pass@$k$) & 16.2\,{\scriptsize (0.00)} & 24.6\,{\scriptsize (0.00)} & 34.5\,{\scriptsize (0.00)} & 45.9\,{\scriptsize (0.00)} & 58.4\,{\scriptsize (0.00)} \\
\midrule
Linear Probe & 17.0\,{\scriptsize (0.34)} & 23.2\,{\scriptsize (0.57)} & 29.4\,{\scriptsize (0.74)} & 38.3\,{\scriptsize (0.69)} & 51.3\,{\scriptsize (0.67)} \\
RepExp & 16.0\,{\scriptsize (0.16)} & 22.5\,{\scriptsize (0.13)} & 31.8\,{\scriptsize (0.15)} & 45.9\,{\scriptsize (0.21)} & 62.1\,{\scriptsize (0.12)} \\
PKTP & 19.2\,{\scriptsize (0.55)} & 25.6\,{\scriptsize (0.61)} & 36.3\,{\scriptsize (0.64)} & 46.5\,{\scriptsize (0.54)} & 58.2\,{\scriptsize (0.89)} \\
\bottomrule
\end{tabular}}
\end{minipage}
\hfill
\begin{minipage}[t]{0.49\linewidth}\centering
{\small\textbf{OlymMATH (Easy)}}\\[2pt]
\resizebox{\linewidth}{!}{%
\begin{tabular}{@{}l*{5}{r}@{}}
\toprule
 & @1 & @2 & @4 & @8 & @16 \\
\midrule
Random (Pass@$k$) & 17.2\,{\scriptsize (0.00)} & 27.8\,{\scriptsize (0.00)} & 40.8\,{\scriptsize (0.00)} & 54.4\,{\scriptsize (0.00)} & 67.3\,{\scriptsize (0.00)} \\
\midrule
Linear Probe & 16.1\,{\scriptsize (0.45)} & 27.5\,{\scriptsize (0.68)} & 41.0\,{\scriptsize (0.58)} & 50.9\,{\scriptsize (0.64)} & 62.8\,{\scriptsize (0.48)} \\
RepExp & 16.9\,{\scriptsize (0.22)} & 31.9\,{\scriptsize (0.21)} & 42.4\,{\scriptsize (0.19)} & 53.6\,{\scriptsize (0.15)} & 68.7\,{\scriptsize (0.12)} \\
PKTP & 18.5\,{\scriptsize (0.63)} & 29.0\,{\scriptsize (0.99)} & 41.9\,{\scriptsize (0.91)} & 55.7\,{\scriptsize (0.56)} & 69.9\,{\scriptsize (0.38)} \\
\bottomrule
\end{tabular}}
\end{minipage}
\par\vspace{6pt}
\begin{minipage}[t]{0.49\linewidth}\centering
{\small\textbf{Omni-MATH}}\\[2pt]
\resizebox{\linewidth}{!}{%
\begin{tabular}{@{}l*{5}{r}@{}}
\toprule
 & @1 & @2 & @4 & @8 & @16 \\
\midrule
Random (Pass@$k$) & 29.3\,{\scriptsize (0.00)} & 40.9\,{\scriptsize (0.00)} & 51.5\,{\scriptsize (0.00)} & 60.7\,{\scriptsize (0.00)} & 69.7\,{\scriptsize (0.00)} \\
\midrule
Linear Probe & 32.5\,{\scriptsize (0.42)} & 42.3\,{\scriptsize (0.46)} & 52.0\,{\scriptsize (0.50)} & 62.3\,{\scriptsize (0.51)} & 72.3\,{\scriptsize (0.42)} \\
RepExp & 29.2\,{\scriptsize (0.14)} & 42.5\,{\scriptsize (0.16)} & 54.7\,{\scriptsize (0.12)} & 63.3\,{\scriptsize (0.10)} & 74.5\,{\scriptsize (0.07)} \\
PKTP & 31.4\,{\scriptsize (0.57)} & 44.9\,{\scriptsize (0.65)} & 57.8\,{\scriptsize (0.53)} & 65.2\,{\scriptsize (0.30)} & 71.1\,{\scriptsize (0.40)} \\
\bottomrule
\end{tabular}}
\end{minipage}
\hfill
\begin{minipage}[t]{0.49\linewidth}\centering
{\small\textbf{RIMO-N}}\\[2pt]
\resizebox{\linewidth}{!}{%
\begin{tabular}{@{}l*{5}{r}@{}}
\toprule
 & @1 & @2 & @4 & @8 & @16 \\
\midrule
Random (Pass@$k$) & 39.8\,{\scriptsize (0.00)} & 53.1\,{\scriptsize (0.00)} & 65.2\,{\scriptsize (0.00)} & 75.3\,{\scriptsize (0.00)} & 82.7\,{\scriptsize (0.00)} \\
\midrule
Linear Probe & 40.4\,{\scriptsize (0.48)} & 52.0\,{\scriptsize (0.68)} & 63.0\,{\scriptsize (0.74)} & 74.2\,{\scriptsize (0.56)} & 83.4\,{\scriptsize (0.64)} \\
RepExp & 40.0\,{\scriptsize (0.19)} & 52.9\,{\scriptsize (0.16)} & 64.8\,{\scriptsize (0.12)} & 77.1\,{\scriptsize (0.13)} & 84.2\,{\scriptsize (0.10)} \\
PKTP & 40.6\,{\scriptsize (0.82)} & 53.7\,{\scriptsize (0.67)} & 66.6\,{\scriptsize (0.59)} & 77.5\,{\scriptsize (0.51)} & 85.5\,{\scriptsize (0.39)} \\
\bottomrule
\end{tabular}}
\end{minipage}
\end{table}

\begin{table}[H]
\centering
\setlength{\tabcolsep}{3pt}
\caption{Per-benchmark pass@$k$ (\%), Qwen3-8B.}
\label{tab:app-qwen3-8b}
\begin{minipage}[t]{0.49\linewidth}\centering
{\small\textbf{AIME}}\\[2pt]
\resizebox{\linewidth}{!}{%
\begin{tabular}{@{}l*{5}{r}@{}}
\toprule
 & @1 & @2 & @4 & @8 & @16 \\
\midrule
Random (Pass@$k$) & 70.2\,{\scriptsize (0.00)} & 80.6\,{\scriptsize (0.00)} & 86.3\,{\scriptsize (0.00)} & 90.1\,{\scriptsize (0.00)} & 93.1\,{\scriptsize (0.00)} \\
\midrule
Linear Probe & 69.3\,{\scriptsize (0.64)} & 77.7\,{\scriptsize (0.50)} & 83.6\,{\scriptsize (0.34)} & 87.3\,{\scriptsize (0.32)} & 90.8\,{\scriptsize (0.38)} \\
RepExp & 70.0\,{\scriptsize (0.28)} & 80.9\,{\scriptsize (0.12)} & 86.6\,{\scriptsize (0.10)} & 91.0\,{\scriptsize (0.06)} & 92.8\,{\scriptsize (0.07)} \\
PKTP & 64.1\,{\scriptsize (0.69)} & 77.9\,{\scriptsize (0.55)} & 84.2\,{\scriptsize (0.43)} & 88.6\,{\scriptsize (0.43)} & 93.6\,{\scriptsize (0.41)} \\
\bottomrule
\end{tabular}}
\end{minipage}
\hfill
\begin{minipage}[t]{0.49\linewidth}\centering
{\small\textbf{HMMT}}\\[2pt]
\resizebox{\linewidth}{!}{%
\begin{tabular}{@{}l*{5}{r}@{}}
\toprule
 & @1 & @2 & @4 & @8 & @16 \\
\midrule
Random (Pass@$k$) & 46.2\,{\scriptsize (0.00)} & 56.3\,{\scriptsize (0.00)} & 65.3\,{\scriptsize (0.00)} & 73.1\,{\scriptsize (0.00)} & 80.3\,{\scriptsize (0.00)} \\
\midrule
Linear Probe & 54.3\,{\scriptsize (0.57)} & 65.1\,{\scriptsize (0.61)} & 70.1\,{\scriptsize (0.57)} & 75.2\,{\scriptsize (0.33)} & 78.7\,{\scriptsize (0.33)} \\
RepExp & 46.3\,{\scriptsize (0.23)} & 56.5\,{\scriptsize (0.20)} & 65.5\,{\scriptsize (0.17)} & 74.9\,{\scriptsize (0.15)} & 84.3\,{\scriptsize (0.08)} \\
PKTP & 51.7\,{\scriptsize (0.68)} & 61.8\,{\scriptsize (0.49)} & 70.6\,{\scriptsize (0.54)} & 76.6\,{\scriptsize (0.59)} & 82.9\,{\scriptsize (0.50)} \\
\bottomrule
\end{tabular}}
\end{minipage}
\par\vspace{6pt}
\begin{minipage}[t]{0.49\linewidth}\centering
{\small\textbf{AMO-Bench}}\\[2pt]
\resizebox{\linewidth}{!}{%
\begin{tabular}{@{}l*{5}{r}@{}}
\toprule
 & @1 & @2 & @4 & @8 & @16 \\
\midrule
Random (Pass@$k$) & 22.9\,{\scriptsize (0.00)} & 32.4\,{\scriptsize (0.00)} & 43.2\,{\scriptsize (0.00)} & 55.5\,{\scriptsize (0.00)} & 69.2\,{\scriptsize (0.00)} \\
\midrule
Linear Probe & 25.7\,{\scriptsize (0.68)} & 34.1\,{\scriptsize (0.83)} & 47.2\,{\scriptsize (1.40)} & 57.8\,{\scriptsize (0.96)} & 71.0\,{\scriptsize (1.01)} \\
RepExp & 22.4\,{\scriptsize (0.30)} & 32.3\,{\scriptsize (0.35)} & 40.5\,{\scriptsize (0.33)} & 53.0\,{\scriptsize (0.29)} & 65.1\,{\scriptsize (0.26)} \\
PKTP & 27.2\,{\scriptsize (1.03)} & 37.9\,{\scriptsize (0.79)} & 48.4\,{\scriptsize (0.95)} & 60.5\,{\scriptsize (1.36)} & 73.8\,{\scriptsize (1.10)} \\
\bottomrule
\end{tabular}}
\end{minipage}
\hfill
\begin{minipage}[t]{0.49\linewidth}\centering
{\small\textbf{OlymMATH (Hard)}}\\[2pt]
\resizebox{\linewidth}{!}{%
\begin{tabular}{@{}l*{5}{r}@{}}
\toprule
 & @1 & @2 & @4 & @8 & @16 \\
\midrule
Random (Pass@$k$) & 30.2\,{\scriptsize (0.00)} & 44.0\,{\scriptsize (0.00)} & 57.8\,{\scriptsize (0.00)} & 70.0\,{\scriptsize (0.00)} & 80.4\,{\scriptsize (0.00)} \\
\midrule
Linear Probe & 31.9\,{\scriptsize (0.67)} & 44.1\,{\scriptsize (0.80)} & 57.9\,{\scriptsize (0.60)} & 70.3\,{\scriptsize (0.53)} & 82.2\,{\scriptsize (0.67)} \\
RepExp & 30.3\,{\scriptsize (0.23)} & 46.2\,{\scriptsize (0.19)} & 61.7\,{\scriptsize (0.16)} & 73.8\,{\scriptsize (0.11)} & 83.2\,{\scriptsize (0.10)} \\
PKTP & 33.5\,{\scriptsize (0.80)} & 45.9\,{\scriptsize (0.73)} & 59.3\,{\scriptsize (0.78)} & 69.7\,{\scriptsize (0.76)} & 79.3\,{\scriptsize (0.56)} \\
\bottomrule
\end{tabular}}
\end{minipage}
\par\vspace{6pt}
\begin{minipage}[t]{0.49\linewidth}\centering
{\small\textbf{Omni-MATH}}\\[2pt]
\resizebox{\linewidth}{!}{%
\begin{tabular}{@{}l*{5}{r}@{}}
\toprule
 & @1 & @2 & @4 & @8 & @16 \\
\midrule
Random (Pass@$k$) & 46.0\,{\scriptsize (0.00)} & 56.9\,{\scriptsize (0.00)} & 65.8\,{\scriptsize (0.00)} & 73.1\,{\scriptsize (0.00)} & 79.7\,{\scriptsize (0.00)} \\
\midrule
Linear Probe & 45.3\,{\scriptsize (0.72)} & 54.4\,{\scriptsize (0.51)} & 61.7\,{\scriptsize (0.43)} & 68.5\,{\scriptsize (0.43)} & 75.7\,{\scriptsize (0.43)} \\
RepExp & 46.0\,{\scriptsize (0.19)} & 55.6\,{\scriptsize (0.15)} & 68.6\,{\scriptsize (0.12)} & 74.5\,{\scriptsize (0.11)} & 79.9\,{\scriptsize (0.08)} \\
PKTP & 45.5\,{\scriptsize (0.52)} & 56.5\,{\scriptsize (0.58)} & 64.5\,{\scriptsize (0.44)} & 71.3\,{\scriptsize (0.33)} & 77.7\,{\scriptsize (0.47)} \\
\bottomrule
\end{tabular}}
\end{minipage}
\hfill
\begin{minipage}[t]{0.49\linewidth}\centering
{\small\textbf{RIMO-N}}\\[2pt]
\resizebox{\linewidth}{!}{%
\begin{tabular}{@{}l*{5}{r}@{}}
\toprule
 & @1 & @2 & @4 & @8 & @16 \\
\midrule
Random (Pass@$k$) & 55.3\,{\scriptsize (0.00)} & 65.9\,{\scriptsize (0.00)} & 74.0\,{\scriptsize (0.00)} & 80.3\,{\scriptsize (0.00)} & 85.7\,{\scriptsize (0.00)} \\
\midrule
Linear Probe & 60.9\,{\scriptsize (0.68)} & 70.9\,{\scriptsize (0.59)} & 77.8\,{\scriptsize (0.67)} & 82.7\,{\scriptsize (0.50)} & 86.5\,{\scriptsize (0.29)} \\
RepExp & 55.3\,{\scriptsize (0.15)} & 61.8\,{\scriptsize (0.16)} & 70.7\,{\scriptsize (0.14)} & 76.5\,{\scriptsize (0.08)} & 84.8\,{\scriptsize (0.08)} \\
PKTP & 53.9\,{\scriptsize (1.00)} & 64.0\,{\scriptsize (0.76)} & 77.3\,{\scriptsize (0.74)} & 85.0\,{\scriptsize (0.63)} & 89.8\,{\scriptsize (0.32)} \\
\bottomrule
\end{tabular}}
\end{minipage}
\end{table}

\begin{table}[H]
\centering
\setlength{\tabcolsep}{3pt}
\caption{Per-benchmark pass@$k$ (\%), OpenMath-Nemotron-7B.}
\label{tab:app-nemotron-7b}
\begin{minipage}[t]{0.49\linewidth}\centering
{\small\textbf{AIME}}\\[2pt]
\resizebox{\linewidth}{!}{%
\begin{tabular}{@{}l*{5}{r}@{}}
\toprule
 & @1 & @2 & @4 & @8 & @16 \\
\midrule
Random (Pass@$k$) & 68.7\,{\scriptsize (0.00)} & 80.0\,{\scriptsize (0.00)} & 87.0\,{\scriptsize (0.00)} & 91.7\,{\scriptsize (0.00)} & 95.2\,{\scriptsize (0.00)} \\
\midrule
Linear Probe & 71.5\,{\scriptsize (0.67)} & 81.8\,{\scriptsize (0.47)} & 86.8\,{\scriptsize (0.35)} & 90.5\,{\scriptsize (0.35)} & 93.8\,{\scriptsize (0.35)} \\
RepExp & 69.1\,{\scriptsize (0.25)} & 81.1\,{\scriptsize (0.16)} & 89.6\,{\scriptsize (0.12)} & 93.5\,{\scriptsize (0.09)} & 98.1\,{\scriptsize (0.05)} \\
PKTP & 72.1\,{\scriptsize (0.71)} & 82.9\,{\scriptsize (0.29)} & 87.6\,{\scriptsize (0.31)} & 90.6\,{\scriptsize (0.45)} & 94.3\,{\scriptsize (0.42)} \\
\bottomrule
\end{tabular}}
\end{minipage}
\hfill
\begin{minipage}[t]{0.49\linewidth}\centering
{\small\textbf{HMMT}}\\[2pt]
\resizebox{\linewidth}{!}{%
\begin{tabular}{@{}l*{5}{r}@{}}
\toprule
 & @1 & @2 & @4 & @8 & @16 \\
\midrule
Random (Pass@$k$) & 41.4\,{\scriptsize (0.00)} & 52.3\,{\scriptsize (0.00)} & 62.9\,{\scriptsize (0.00)} & 73.1\,{\scriptsize (0.00)} & 82.4\,{\scriptsize (0.00)} \\
\midrule
Linear Probe & 43.0\,{\scriptsize (0.50)} & 52.6\,{\scriptsize (0.70)} & 63.4\,{\scriptsize (0.46)} & 72.4\,{\scriptsize (0.40)} & 80.7\,{\scriptsize (0.63)} \\
RepExp & 41.4\,{\scriptsize (0.28)} & 54.1\,{\scriptsize (0.25)} & 62.5\,{\scriptsize (0.18)} & 73.6\,{\scriptsize (0.11)} & 83.6\,{\scriptsize (0.09)} \\
PKTP & 47.0\,{\scriptsize (0.68)} & 58.7\,{\scriptsize (0.88)} & 67.9\,{\scriptsize (0.55)} & 77.8\,{\scriptsize (0.44)} & 83.9\,{\scriptsize (0.51)} \\
\bottomrule
\end{tabular}}
\end{minipage}
\par\vspace{6pt}
\begin{minipage}[t]{0.49\linewidth}\centering
{\small\textbf{AMO-Bench}}\\[2pt]
\resizebox{\linewidth}{!}{%
\begin{tabular}{@{}l*{5}{r}@{}}
\toprule
 & @1 & @2 & @4 & @8 & @16 \\
\midrule
Random (Pass@$k$) & 16.9\,{\scriptsize (0.00)} & 27.0\,{\scriptsize (0.00)} & 38.9\,{\scriptsize (0.00)} & 51.2\,{\scriptsize (0.00)} & 63.6\,{\scriptsize (0.00)} \\
\midrule
Linear Probe & 18.8\,{\scriptsize (1.35)} & 29.2\,{\scriptsize (1.13)} & 41.8\,{\scriptsize (1.44)} & 54.0\,{\scriptsize (0.82)} & 66.3\,{\scriptsize (0.80)} \\
RepExp & 17.1\,{\scriptsize (0.35)} & 34.0\,{\scriptsize (0.33)} & 43.1\,{\scriptsize (0.28)} & 55.7\,{\scriptsize (0.27)} & 71.0\,{\scriptsize (0.19)} \\
PKTP & 20.5\,{\scriptsize (0.78)} & 27.0\,{\scriptsize (0.76)} & 35.5\,{\scriptsize (1.06)} & 47.7\,{\scriptsize (1.11)} & 67.7\,{\scriptsize (1.14)} \\
\bottomrule
\end{tabular}}
\end{minipage}
\hfill
\begin{minipage}[t]{0.49\linewidth}\centering
{\small\textbf{OlymMATH (Hard)}}\\[2pt]
\resizebox{\linewidth}{!}{%
\begin{tabular}{@{}l*{5}{r}@{}}
\toprule
 & @1 & @2 & @4 & @8 & @16 \\
\midrule
Random (Pass@$k$) & 24.5\,{\scriptsize (0.00)} & 36.9\,{\scriptsize (0.00)} & 50.4\,{\scriptsize (0.00)} & 63.1\,{\scriptsize (0.00)} & 73.7\,{\scriptsize (0.00)} \\
\midrule
Linear Probe & 28.7\,{\scriptsize (0.87)} & 42.4\,{\scriptsize (0.68)} & 55.9\,{\scriptsize (0.68)} & 68.2\,{\scriptsize (0.52)} & 77.2\,{\scriptsize (0.51)} \\
RepExp & 24.4\,{\scriptsize (0.23)} & 38.6\,{\scriptsize (0.21)} & 53.3\,{\scriptsize (0.16)} & 67.1\,{\scriptsize (0.13)} & 75.4\,{\scriptsize (0.10)} \\
PKTP & 29.7\,{\scriptsize (0.75)} & 43.6\,{\scriptsize (0.90)} & 55.3\,{\scriptsize (0.65)} & 67.5\,{\scriptsize (0.45)} & 77.4\,{\scriptsize (0.50)} \\
\bottomrule
\end{tabular}}
\end{minipage}
\par\vspace{6pt}
\begin{minipage}[t]{0.49\linewidth}\centering
{\small\textbf{Omni-MATH}}\\[2pt]
\resizebox{\linewidth}{!}{%
\begin{tabular}{@{}l*{5}{r}@{}}
\toprule
 & @1 & @2 & @4 & @8 & @16 \\
\midrule
Random (Pass@$k$) & 47.4\,{\scriptsize (0.00)} & 59.1\,{\scriptsize (0.00)} & 69.6\,{\scriptsize (0.00)} & 78.7\,{\scriptsize (0.00)} & 86.1\,{\scriptsize (0.00)} \\
\midrule
Linear Probe & 47.1\,{\scriptsize (0.36)} & 58.9\,{\scriptsize (0.48)} & 69.3\,{\scriptsize (0.41)} & 77.9\,{\scriptsize (0.39)} & 85.6\,{\scriptsize (0.29)} \\
RepExp & 47.2\,{\scriptsize (0.17)} & 57.9\,{\scriptsize (0.11)} & 65.8\,{\scriptsize (0.10)} & 77.4\,{\scriptsize (0.07)} & 87.4\,{\scriptsize (0.06)} \\
PKTP & 48.1\,{\scriptsize (0.44)} & 60.9\,{\scriptsize (0.57)} & 71.9\,{\scriptsize (0.61)} & 80.1\,{\scriptsize (0.43)} & 87.1\,{\scriptsize (0.22)} \\
\bottomrule
\end{tabular}}
\end{minipage}
\hfill
\begin{minipage}[t]{0.49\linewidth}\centering
{\small\textbf{RIMO-N}}\\[2pt]
\resizebox{\linewidth}{!}{%
\begin{tabular}{@{}l*{5}{r}@{}}
\toprule
 & @1 & @2 & @4 & @8 & @16 \\
\midrule
Random (Pass@$k$) & 49.5\,{\scriptsize (0.00)} & 61.1\,{\scriptsize (0.00)} & 70.6\,{\scriptsize (0.00)} & 78.3\,{\scriptsize (0.00)} & 84.9\,{\scriptsize (0.00)} \\
\midrule
Linear Probe & 55.5\,{\scriptsize (0.82)} & 65.2\,{\scriptsize (0.75)} & 73.3\,{\scriptsize (0.50)} & 82.0\,{\scriptsize (0.55)} & 87.2\,{\scriptsize (0.35)} \\
RepExp & 49.6\,{\scriptsize (0.17)} & 64.1\,{\scriptsize (0.18)} & 74.0\,{\scriptsize (0.11)} & 78.9\,{\scriptsize (0.11)} & 83.9\,{\scriptsize (0.10)} \\
PKTP & 54.4\,{\scriptsize (0.81)} & 64.5\,{\scriptsize (0.65)} & 71.1\,{\scriptsize (0.43)} & 78.2\,{\scriptsize (0.54)} & 84.5\,{\scriptsize (0.42)} \\
\bottomrule
\end{tabular}}
\end{minipage}
\end{table}

\begin{table}[H]
\centering
\setlength{\tabcolsep}{3pt}
\caption{Per-benchmark pass@$k$ (\%), DeepSeek-R1-Distill-Qwen-7B.}
\label{tab:app-r1d-7b}
\begin{minipage}[t]{0.49\linewidth}\centering
{\small\textbf{AIME}}\\[2pt]
\resizebox{\linewidth}{!}{%
\begin{tabular}{@{}l*{5}{r}@{}}
\toprule
 & @1 & @2 & @4 & @8 & @16 \\
\midrule
Random (Pass@$k$) & 52.7\,{\scriptsize (0.00)} & 66.5\,{\scriptsize (0.00)} & 76.9\,{\scriptsize (0.00)} & 83.5\,{\scriptsize (0.00)} & 88.2\,{\scriptsize (0.00)} \\
\midrule
Linear Probe & 51.6\,{\scriptsize (0.58)} & 65.4\,{\scriptsize (0.47)} & 75.1\,{\scriptsize (0.44)} & 80.9\,{\scriptsize (0.35)} & 87.4\,{\scriptsize (0.60)} \\
RepExp & 52.3\,{\scriptsize (0.25)} & 65.0\,{\scriptsize (0.25)} & 77.1\,{\scriptsize (0.16)} & 87.2\,{\scriptsize (0.09)} & 90.5\,{\scriptsize (0.04)} \\
PKTP & 55.2\,{\scriptsize (0.87)} & 68.6\,{\scriptsize (0.57)} & 78.1\,{\scriptsize (0.40)} & 86.1\,{\scriptsize (0.45)} & 90.2\,{\scriptsize (0.33)} \\
\bottomrule
\end{tabular}}
\end{minipage}
\hfill
\begin{minipage}[t]{0.49\linewidth}\centering
{\small\textbf{HMMT}}\\[2pt]
\resizebox{\linewidth}{!}{%
\begin{tabular}{@{}l*{5}{r}@{}}
\toprule
 & @1 & @2 & @4 & @8 & @16 \\
\midrule
Random (Pass@$k$) & 38.1\,{\scriptsize (0.00)} & 47.0\,{\scriptsize (0.00)} & 55.7\,{\scriptsize (0.00)} & 65.4\,{\scriptsize (0.00)} & 75.4\,{\scriptsize (0.00)} \\
\midrule
Linear Probe & 41.5\,{\scriptsize (0.40)} & 45.4\,{\scriptsize (0.38)} & 51.7\,{\scriptsize (0.68)} & 62.3\,{\scriptsize (0.59)} & 77.1\,{\scriptsize (0.71)} \\
RepExp & 38.0\,{\scriptsize (0.16)} & 46.1\,{\scriptsize (0.20)} & 55.0\,{\scriptsize (0.22)} & 68.7\,{\scriptsize (0.19)} & 80.2\,{\scriptsize (0.10)} \\
PKTP & 41.0\,{\scriptsize (0.73)} & 48.6\,{\scriptsize (0.76)} & 57.9\,{\scriptsize (0.82)} & 68.3\,{\scriptsize (0.63)} & 77.5\,{\scriptsize (0.65)} \\
\bottomrule
\end{tabular}}
\end{minipage}
\par\vspace{6pt}
\begin{minipage}[t]{0.49\linewidth}\centering
{\small\textbf{AMO-Bench}}\\[2pt]
\resizebox{\linewidth}{!}{%
\begin{tabular}{@{}l*{5}{r}@{}}
\toprule
 & @1 & @2 & @4 & @8 & @16 \\
\midrule
Random (Pass@$k$) & 7.0\,{\scriptsize (0.00)} & 12.5\,{\scriptsize (0.00)} & 21.1\,{\scriptsize (0.00)} & 33.3\,{\scriptsize (0.00)} & 48.8\,{\scriptsize (0.00)} \\
\midrule
Linear Probe & 6.3\,{\scriptsize (0.59)} & 12.4\,{\scriptsize (0.79)} & 23.7\,{\scriptsize (1.11)} & 42.0\,{\scriptsize (1.52)} & 57.6\,{\scriptsize (1.26)} \\
RepExp & 7.4\,{\scriptsize (0.36)} & 12.2\,{\scriptsize (0.39)} & 26.2\,{\scriptsize (0.50)} & 41.7\,{\scriptsize (0.31)} & 60.4\,{\scriptsize (0.24)} \\
PKTP & 10.7\,{\scriptsize (1.32)} & 20.0\,{\scriptsize (1.82)} & 33.9\,{\scriptsize (1.40)} & 49.3\,{\scriptsize (1.45)} & 62.4\,{\scriptsize (1.27)} \\
\bottomrule
\end{tabular}}
\end{minipage}
\hfill
\begin{minipage}[t]{0.49\linewidth}\centering
{\small\textbf{OlymMATH (Hard)}}\\[2pt]
\resizebox{\linewidth}{!}{%
\begin{tabular}{@{}l*{5}{r}@{}}
\toprule
 & @1 & @2 & @4 & @8 & @16 \\
\midrule
Random (Pass@$k$) & 11.5\,{\scriptsize (0.00)} & 19.2\,{\scriptsize (0.00)} & 29.9\,{\scriptsize (0.00)} & 43.3\,{\scriptsize (0.00)} & 57.9\,{\scriptsize (0.00)} \\
\midrule
Linear Probe & 13.7\,{\scriptsize (0.53)} & 21.2\,{\scriptsize (0.60)} & 28.0\,{\scriptsize (0.42)} & 37.7\,{\scriptsize (0.65)} & 52.7\,{\scriptsize (0.73)} \\
RepExp & 11.6\,{\scriptsize (0.18)} & 19.2\,{\scriptsize (0.20)} & 27.8\,{\scriptsize (0.17)} & 39.5\,{\scriptsize (0.16)} & 60.6\,{\scriptsize (0.10)} \\
PKTP & 10.6\,{\scriptsize (0.46)} & 20.3\,{\scriptsize (0.56)} & 29.4\,{\scriptsize (0.54)} & 43.4\,{\scriptsize (0.73)} & 58.0\,{\scriptsize (0.59)} \\
\bottomrule
\end{tabular}}
\end{minipage}
\par\vspace{6pt}
\begin{minipage}[t]{0.49\linewidth}\centering
{\small\textbf{Omni-MATH}}\\[2pt]
\resizebox{\linewidth}{!}{%
\begin{tabular}{@{}l*{5}{r}@{}}
\toprule
 & @1 & @2 & @4 & @8 & @16 \\
\midrule
Random (Pass@$k$) & 41.6\,{\scriptsize (0.00)} & 53.1\,{\scriptsize (0.00)} & 63.3\,{\scriptsize (0.00)} & 72.8\,{\scriptsize (0.00)} & 81.5\,{\scriptsize (0.00)} \\
\midrule
Linear Probe & 40.6\,{\scriptsize (0.35)} & 50.2\,{\scriptsize (0.60)} & 60.5\,{\scriptsize (0.48)} & 70.6\,{\scriptsize (0.41)} & 80.8\,{\scriptsize (0.34)} \\
RepExp & 41.4\,{\scriptsize (0.14)} & 51.5\,{\scriptsize (0.11)} & 64.0\,{\scriptsize (0.12)} & 74.4\,{\scriptsize (0.08)} & 80.4\,{\scriptsize (0.04)} \\
PKTP & 42.8\,{\scriptsize (0.64)} & 54.0\,{\scriptsize (0.54)} & 64.4\,{\scriptsize (0.54)} & 74.0\,{\scriptsize (0.53)} & 83.5\,{\scriptsize (0.36)} \\
\bottomrule
\end{tabular}}
\end{minipage}
\hfill
\begin{minipage}[t]{0.49\linewidth}\centering
{\small\textbf{RIMO-N}}\\[2pt]
\resizebox{\linewidth}{!}{%
\begin{tabular}{@{}l*{5}{r}@{}}
\toprule
 & @1 & @2 & @4 & @8 & @16 \\
\midrule
Random (Pass@$k$) & 50.2\,{\scriptsize (0.00)} & 63.3\,{\scriptsize (0.00)} & 74.4\,{\scriptsize (0.00)} & 83.3\,{\scriptsize (0.00)} & 89.9\,{\scriptsize (0.00)} \\
\midrule
Linear Probe & 47.2\,{\scriptsize (0.89)} & 59.3\,{\scriptsize (0.80)} & 73.3\,{\scriptsize (0.64)} & 82.1\,{\scriptsize (0.51)} & 87.8\,{\scriptsize (0.40)} \\
RepExp & 50.1\,{\scriptsize (0.30)} & 61.8\,{\scriptsize (0.22)} & 74.6\,{\scriptsize (0.17)} & 85.8\,{\scriptsize (0.11)} & 92.3\,{\scriptsize (0.03)} \\
PKTP & 49.4\,{\scriptsize (0.71)} & 61.7\,{\scriptsize (0.87)} & 73.0\,{\scriptsize (0.43)} & 81.1\,{\scriptsize (0.43)} & 89.2\,{\scriptsize (0.35)} \\
\bottomrule
\end{tabular}}
\end{minipage}
\end{table}

\begin{table}[H]
\centering
\setlength{\tabcolsep}{3pt}
\caption{Per-benchmark pass@$k$ (\%), Skywork-OR1-Math-7B.}
\label{tab:app-skywork-7b}
\begin{minipage}[t]{0.49\linewidth}\centering
{\small\textbf{AIME}}\\[2pt]
\resizebox{\linewidth}{!}{%
\begin{tabular}{@{}l*{5}{r}@{}}
\toprule
 & @1 & @2 & @4 & @8 & @16 \\
\midrule
Random (Pass@$k$) & 59.7\,{\scriptsize (0.00)} & 71.2\,{\scriptsize (0.00)} & 79.1\,{\scriptsize (0.00)} & 84.7\,{\scriptsize (0.00)} & 89.2\,{\scriptsize (0.00)} \\
\midrule
Linear Probe & 62.3\,{\scriptsize (0.73)} & 70.7\,{\scriptsize (0.67)} & 80.2\,{\scriptsize (0.44)} & 85.3\,{\scriptsize (0.37)} & 89.8\,{\scriptsize (0.29)} \\
RepExp & 59.7\,{\scriptsize (0.23)} & 73.2\,{\scriptsize (0.16)} & 82.0\,{\scriptsize (0.12)} & 85.4\,{\scriptsize (0.08)} & 92.8\,{\scriptsize (0.08)} \\
PKTP & 63.6\,{\scriptsize (0.89)} & 75.4\,{\scriptsize (0.49)} & 80.6\,{\scriptsize (0.46)} & 84.6\,{\scriptsize (0.27)} & 88.4\,{\scriptsize (0.45)} \\
\bottomrule
\end{tabular}}
\end{minipage}
\hfill
\begin{minipage}[t]{0.49\linewidth}\centering
{\small\textbf{HMMT}}\\[2pt]
\resizebox{\linewidth}{!}{%
\begin{tabular}{@{}l*{5}{r}@{}}
\toprule
 & @1 & @2 & @4 & @8 & @16 \\
\midrule
Random (Pass@$k$) & 35.8\,{\scriptsize (0.00)} & 43.8\,{\scriptsize (0.00)} & 52.3\,{\scriptsize (0.00)} & 61.7\,{\scriptsize (0.00)} & 71.7\,{\scriptsize (0.00)} \\
\midrule
Linear Probe & 40.3\,{\scriptsize (0.47)} & 48.2\,{\scriptsize (0.57)} & 54.5\,{\scriptsize (0.63)} & 62.5\,{\scriptsize (0.48)} & 70.8\,{\scriptsize (0.45)} \\
RepExp & 35.8\,{\scriptsize (0.18)} & 41.9\,{\scriptsize (0.20)} & 50.5\,{\scriptsize (0.16)} & 64.0\,{\scriptsize (0.13)} & 70.6\,{\scriptsize (0.10)} \\
PKTP & 37.1\,{\scriptsize (0.63)} & 46.2\,{\scriptsize (0.64)} & 56.1\,{\scriptsize (0.79)} & 64.8\,{\scriptsize (0.78)} & 71.5\,{\scriptsize (0.45)} \\
\bottomrule
\end{tabular}}
\end{minipage}
\par\vspace{6pt}
\begin{minipage}[t]{0.49\linewidth}\centering
{\small\textbf{AMO-Bench}}\\[2pt]
\resizebox{\linewidth}{!}{%
\begin{tabular}{@{}l*{5}{r}@{}}
\toprule
 & @1 & @2 & @4 & @8 & @16 \\
\midrule
Random (Pass@$k$) & 11.9\,{\scriptsize (0.00)} & 18.7\,{\scriptsize (0.00)} & 26.5\,{\scriptsize (0.00)} & 35.0\,{\scriptsize (0.00)} & 45.0\,{\scriptsize (0.00)} \\
\midrule
Linear Probe & 11.7\,{\scriptsize (0.84)} & 17.6\,{\scriptsize (0.80)} & 23.7\,{\scriptsize (0.74)} & 29.3\,{\scriptsize (0.76)} & 41.3\,{\scriptsize (0.92)} \\
RepExp & 12.1\,{\scriptsize (0.40)} & 21.7\,{\scriptsize (0.24)} & 30.9\,{\scriptsize (0.23)} & 38.4\,{\scriptsize (0.35)} & 53.1\,{\scriptsize (0.24)} \\
PKTP & 12.8\,{\scriptsize (0.97)} & 22.8\,{\scriptsize (1.30)} & 32.2\,{\scriptsize (1.15)} & 43.0\,{\scriptsize (1.44)} & 61.1\,{\scriptsize (1.65)} \\
\bottomrule
\end{tabular}}
\end{minipage}
\hfill
\begin{minipage}[t]{0.49\linewidth}\centering
{\small\textbf{OlymMATH (Hard)}}\\[2pt]
\resizebox{\linewidth}{!}{%
\begin{tabular}{@{}l*{5}{r}@{}}
\toprule
 & @1 & @2 & @4 & @8 & @16 \\
\midrule
Random (Pass@$k$) & 19.7\,{\scriptsize (0.00)} & 31.0\,{\scriptsize (0.00)} & 44.5\,{\scriptsize (0.00)} & 58.6\,{\scriptsize (0.00)} & 72.2\,{\scriptsize (0.00)} \\
\midrule
Linear Probe & 19.0\,{\scriptsize (0.52)} & 28.5\,{\scriptsize (0.67)} & 39.8\,{\scriptsize (0.72)} & 55.6\,{\scriptsize (0.92)} & 72.1\,{\scriptsize (0.89)} \\
RepExp & 19.9\,{\scriptsize (0.18)} & 30.6\,{\scriptsize (0.23)} & 44.2\,{\scriptsize (0.20)} & 52.3\,{\scriptsize (0.14)} & 64.3\,{\scriptsize (0.20)} \\
PKTP & 20.4\,{\scriptsize (0.61)} & 31.7\,{\scriptsize (0.96)} & 45.6\,{\scriptsize (0.93)} & 58.0\,{\scriptsize (0.64)} & 72.2\,{\scriptsize (0.62)} \\
\bottomrule
\end{tabular}}
\end{minipage}
\par\vspace{6pt}
\begin{minipage}[t]{0.49\linewidth}\centering
{\small\textbf{Omni-MATH}}\\[2pt]
\resizebox{\linewidth}{!}{%
\begin{tabular}{@{}l*{5}{r}@{}}
\toprule
 & @1 & @2 & @4 & @8 & @16 \\
\midrule
Random (Pass@$k$) & 50.2\,{\scriptsize (0.00)} & 61.4\,{\scriptsize (0.00)} & 70.7\,{\scriptsize (0.00)} & 78.8\,{\scriptsize (0.00)} & 85.6\,{\scriptsize (0.00)} \\
\midrule
Linear Probe & 53.5\,{\scriptsize (0.60)} & 63.6\,{\scriptsize (0.40)} & 71.8\,{\scriptsize (0.42)} & 79.9\,{\scriptsize (0.41)} & 85.1\,{\scriptsize (0.32)} \\
RepExp & 50.1\,{\scriptsize (0.21)} & 63.5\,{\scriptsize (0.12)} & 73.4\,{\scriptsize (0.12)} & 81.5\,{\scriptsize (0.08)} & 84.8\,{\scriptsize (0.08)} \\
PKTP & 52.1\,{\scriptsize (0.63)} & 63.3\,{\scriptsize (0.41)} & 74.5\,{\scriptsize (0.38)} & 81.1\,{\scriptsize (0.42)} & 84.5\,{\scriptsize (0.26)} \\
\bottomrule
\end{tabular}}
\end{minipage}
\hfill
\begin{minipage}[t]{0.49\linewidth}\centering
{\small\textbf{RIMO-N}}\\[2pt]
\resizebox{\linewidth}{!}{%
\begin{tabular}{@{}l*{5}{r}@{}}
\toprule
 & @1 & @2 & @4 & @8 & @16 \\
\midrule
Random (Pass@$k$) & 53.1\,{\scriptsize (0.00)} & 65.1\,{\scriptsize (0.00)} & 74.6\,{\scriptsize (0.00)} & 82.4\,{\scriptsize (0.00)} & 89.2\,{\scriptsize (0.00)} \\
\midrule
Linear Probe & 53.1\,{\scriptsize (0.61)} & 67.2\,{\scriptsize (0.58)} & 76.7\,{\scriptsize (0.51)} & 82.8\,{\scriptsize (0.45)} & 88.9\,{\scriptsize (0.46)} \\
RepExp & 53.1\,{\scriptsize (0.25)} & 69.5\,{\scriptsize (0.16)} & 78.4\,{\scriptsize (0.16)} & 83.4\,{\scriptsize (0.11)} & 90.3\,{\scriptsize (0.11)} \\
PKTP & 57.0\,{\scriptsize (1.01)} & 68.9\,{\scriptsize (0.82)} & 77.5\,{\scriptsize (0.60)} & 83.5\,{\scriptsize (0.59)} & 89.0\,{\scriptsize (0.46)} \\
\bottomrule
\end{tabular}}
\end{minipage}
\end{table}

\clearpage
\section{Results on different pool sizes}
\label{app:pool-sizes}

This section repeats \cref{fig:gains} for the other pool sizes, $N \in \{32, 48, 64, 80, 96, 112\}$, where the pool of each problem is drawn from its $128$ generated rollouts and the cutoff is $T = 4000$ as in the main setting.

\begin{figure}[H]
\setlength{\abovecaptionskip}{2pt}
\centering
\includegraphics[width=\textwidth]{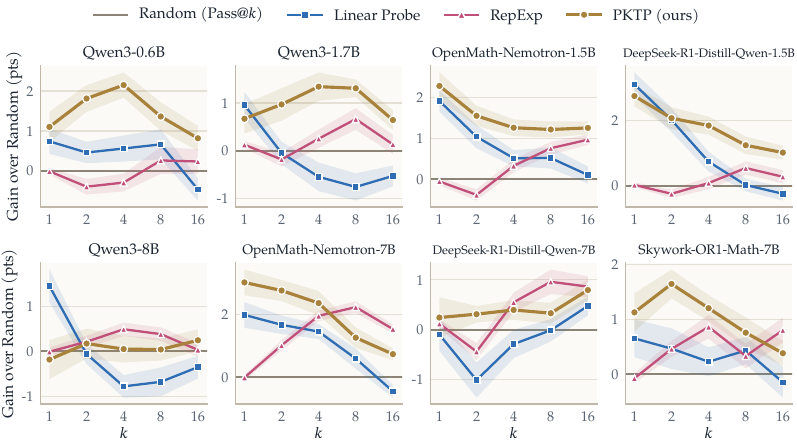}
\caption{\textbf{Gain in \passk over random pruning with $N = 32$.} As in \cref{fig:gains}, with a pool of $N = 32$ rollouts drawn from the $128$ generated for each problem and $T = 4000$.}
\label{fig:gains32}
\end{figure}

\begin{figure}[H]
\setlength{\abovecaptionskip}{2pt}
\centering
\includegraphics[width=\textwidth]{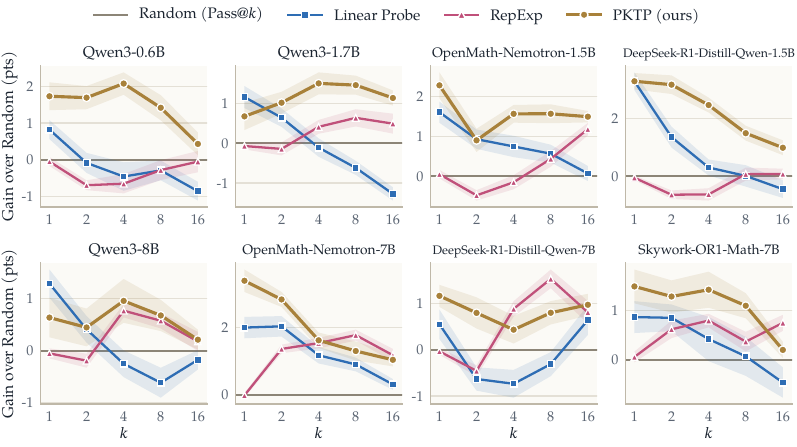}
\caption{\textbf{Gain in \passk over random pruning with $N = 48$.} As in \cref{fig:gains}, with a pool of $N = 48$ rollouts drawn from the $128$ generated for each problem and $T = 4000$.}
\label{fig:gains48}
\end{figure}

\begin{figure}[H]
\setlength{\abovecaptionskip}{2pt}
\centering
\includegraphics[width=\textwidth]{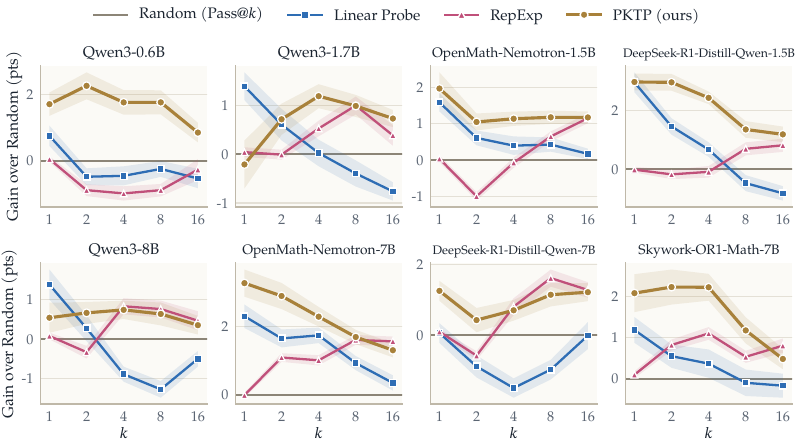}
\caption{\textbf{Gain in \passk over random pruning with $N = 64$.} As in \cref{fig:gains}, with a pool of $N = 64$ rollouts drawn from the $128$ generated for each problem and $T = 4000$.}
\label{fig:gains64}
\end{figure}

\begin{figure}[H]
\setlength{\abovecaptionskip}{2pt}
\centering
\includegraphics[width=\textwidth]{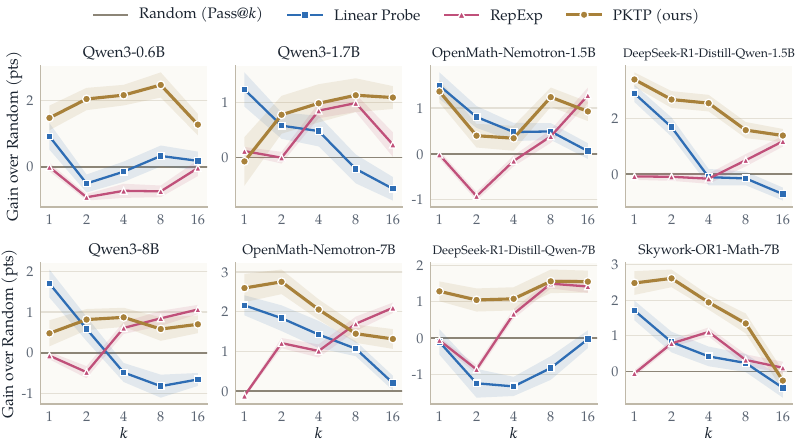}
\caption{\textbf{Gain in \passk over random pruning with $N = 80$.} As in \cref{fig:gains}, with a pool of $N = 80$ rollouts drawn from the $128$ generated for each problem and $T = 4000$.}
\label{fig:gains80}
\end{figure}

\begin{figure}[H]
\setlength{\abovecaptionskip}{2pt}
\centering
\includegraphics[width=\textwidth]{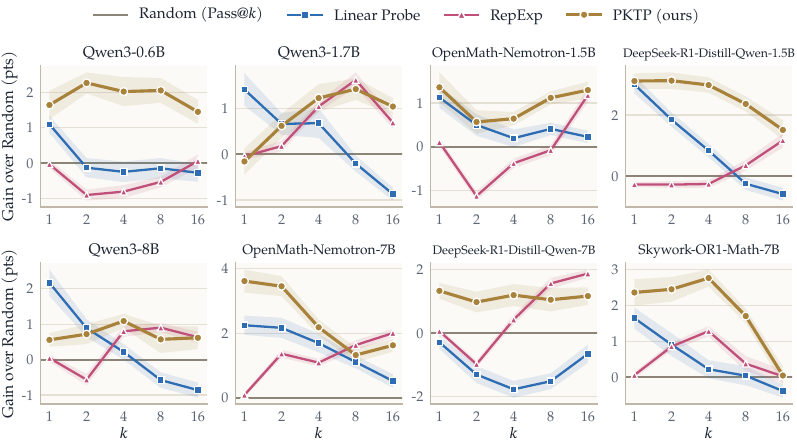}
\caption{\textbf{Gain in \passk over random pruning with $N = 96$.} As in \cref{fig:gains}, with a pool of $N = 96$ rollouts drawn from the $128$ generated for each problem and $T = 4000$.}
\label{fig:gains96}
\end{figure}

\begin{figure}[H]
\setlength{\abovecaptionskip}{2pt}
\centering
\includegraphics[width=\textwidth]{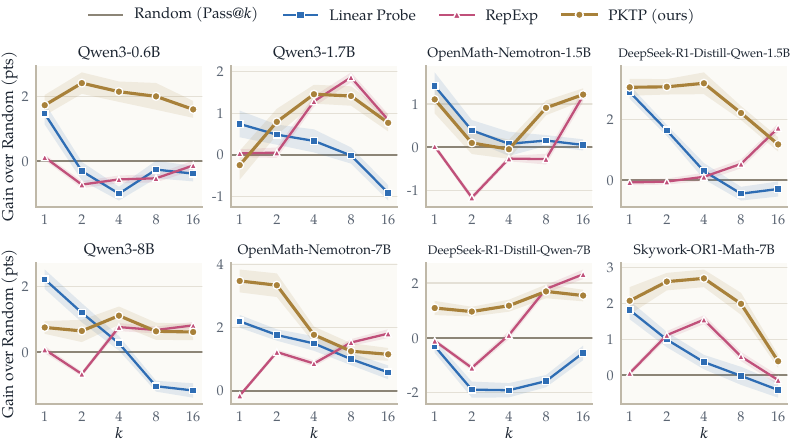}
\caption{\textbf{Gain in \passk over random pruning with $N = 112$.} As in \cref{fig:gains}, with a pool of $N = 112$ rollouts drawn from the $128$ generated for each problem and $T = 4000$.}
\label{fig:gains112}
\end{figure}

\clearpage
\section{Representation details}
\label{app:representation}

For each rollout, let $x_t \in \R^d$ be the hidden state of the last transformer block of the model at the $t$-th generated token, where $d$ is the hidden size of the model, which is $1024$ for Qwen3-0.6B, $2048$ for Qwen3-1.7B, $4096$ for Qwen3-8B, $1536$ for OpenMath-Nemotron-1.5B and DeepSeek-R1-Distill-Qwen-1.5B, and $3584$ for OpenMath-Nemotron-7B, DeepSeek-R1-Distill-Qwen-7B and Skywork-OR1-Math-7B. The EMA follows $e_t = (1 - \alpha)\, e_{t-1} + \alpha\, x_t$ with $\alpha = 10^{-3}$, and $h_i$ is the normalized value of $e_t$ after the first $T$ generated tokens, or after all of them for a rollout that ends before the cutoff.

\section{Extended related work}
\label{app:related}

\paragraph{Pruning rollouts at test time.}
Test-time scaling generates many rollouts per problem and aggregates them \citep{wang2022self, brown2024large, snell2024scaling, wu2024inference}, and a growing line of work saves compute by stopping unpromising rollouts early. One family uses the model's own confidence, such as token-level confidence or its self-estimate of success, to decide which rollouts to continue \citep{manvi2024adaptive, fu2025deep}. Another scores partial rollouts with a reward model \citep{sun2024fast} or with a probe on the hidden states \citep{zhang2025reasoning, afzal2025knowing, step2026}. \citet{yang2026gambit} concurrently use such a probe for a beam search over partial rollouts, pruning weak rollouts and branching new rollouts from the best prefixes. All of these score each rollout on its own and keep the highest-scoring ones, which is the pure exploitation extreme of \cref{sec:problem}. At the other extreme, RepExp \citep{tuyls2025repexp} selects rollouts whose representations are least explained by those already chosen, without using their quality. PKTP sits between the two, since it models how the outcomes of the rollouts depend on each other and chooses the set with the highest predicted \passk.

\paragraph{Pass@$k$ as an objective.}
Pass@$k$ is a standard evaluation metric for reasoning and code generation, and several works optimize it during training. \citet{walder2025pass} derive low-variance policy-gradient estimators for \passk, \citet{chen2025passk} use \passk as the reward to balance exploration and exploitation, and \citet{chow2024inference} fine-tune a model for best-of-$N$ inference. \citet{chen2026dqo} reward a group of responses for both quality and diversity, measuring diversity by the determinant of their similarity matrix. These are post-training methods, while PKTP leaves the model unchanged and optimizes \passk at inference time by choosing which partial rollouts to continue. Concurrently, \citet{wu2026ptts} also leave the reasoning model unchanged, but use \passk as the reward to train a planner that writes a different outline for each rollout before it is generated. They thus change how the rollouts are sampled, while PKTP takes $N$ independently sampled partial rollouts as given and chooses which $k$ of them to continue.

\paragraph{Related objectives: maximizing at least one success.}
The same objective appears in information retrieval as the expected 1-call@$k$, the probability that at least one of the top $k$ documents is relevant. \citet{chen2006less} optimize it greedily, and \citet{sanner2011diverse} show that greedy optimization under a latent subtopic model of relevance leads to diverse rankings, in the same way that our latent factor leads to diverse kept sets. In Bayesian optimization, the batch probability of improvement is the probability that at least one point of a batch improves on the incumbent, and it is monotone and submodular, so it can be maximized greedily \citep{wilson2018maximizing, nemhauser1978analysis}. Other methods balance quality and diversity with a hand-set trade-off, such as maximal marginal relevance \citep{carbonell1998use} and determinantal point processes \citep{kulesza2012determinantal}. In contrast to these works, PKTP applies to partial reasoning rollouts, learns the joint distribution of their outcomes from their representations, and optimizes a closed-form surrogate that keeps the submodularity of the exact objective.

\end{document}